\documentclass[english]{article}
\usepackage{geometry}

\usepackage[T1]{fontenc}
\usepackage[latin9]{inputenc}
\usepackage{bm}
\usepackage{amsmath,mathtools}
\usepackage{amssymb}
\usepackage[unicode=true,
 bookmarks=false,
 breaklinks=false,pdfborder={0 0 1},colorlinks=false]
 {hyperref}
\hypersetup{
 colorlinks,citecolor=blue,filecolor=blue,linkcolor=blue,urlcolor=blue}

\usepackage{enumitem}

\makeatletter
\usepackage{amsthm}
\usepackage{comment}
\usepackage{natbib}
\usepackage{booktabs}

\usepackage{graphicx}
\usepackage{cleveref}
\usepackage[linesnumbered,ruled,vlined]{algorithm2e}

\SetCommentSty{mycommfont}
\usepackage{algorithmic}

\usepackage{float}
\usepackage{multirow}
 
\usepackage{dsfont}
\usepackage{tcolorbox}

\usepackage{color}
\definecolor{yxc}{RGB}{255,0,0}
\definecolor{yjc}{RGB}{125,0,0}
\definecolor{ytw}{RGB}{255,69,0}
\definecolor{gen}{RGB}{0,0,200}

\allowdisplaybreaks

\DeclareMathOperator{\ind}{\mathds{1}}  

\newcommand{\defn}{\coloneqq}

\newcommand{\real}{\mathbb{R}}

\newcommand{\mymid}{\,|\,}

\newcommand{\Pdata}{p_{\mathsf{data}}}

\newcommand{\score}{\mathsf{score}}
\newcommand{\Jacobi}{\mathsf{Jacobi}}

\definecolor{yanxi}{RGB}{0,200,100}

\theoremstyle{plain}

\newtheorem{theo}{Theorem}[section]

\newtheorem{lem}{Lemma}[section]
\newtheorem{prop}{Proposition}[section]
\newtheorem{cor}{Corollary}[section]

\theoremstyle{definition} 

\newtheorem{nota}{Notation}[section]
\newtheorem{de}{Definition}[section]
\newtheorem{exa}{Example}[section]
\newtheorem{as}{Assumption}[section]
\newtheorem{alg}{Algorithm}[section]

\newcommand{\btheo}{\begin{theo}}
\newcommand{\bde}{\begin{de}}
\newcommand{\ble}{\begin{lem}}
\newcommand{\bpr}{\begin{prop}}
\newcommand{\bno}{\begin{nota}}
\newcommand{\bex}{\begin{exa}}
\newcommand{\bcor}{\begin{cor}}
\newcommand{\spro}{\begin{proof}}
\newcommand{\bas}{\begin{as}}
\newcommand{\balg}{\begin{alg}}

\newcommand{\etheo}{\end{theo}}
\newcommand{\ede}{\end{de}}
\newcommand{\ele}{\end{lem}}
\newcommand{\epr}{\end{prop}}
\newcommand{\eno}{\end{nota}}
\newcommand{\eex}{\end{exa}}
\newcommand{\ecor}{\end{cor}}
\newcommand{\fpro}{\end{proof}}
\newcommand{\eas}{\end{as}}
\newcommand{\ealg}{\end{alg}}

\theoremstyle{plain}

\newtheorem{theos}{Theorem}
\newtheorem{props}{Proposition}
\newtheorem{lems}{Lemma}
\newtheorem{cors}{Corollary}

\theoremstyle{definition}
\newtheorem{exas}{Example}
\newtheorem{algs}{Algorithm}
\newtheorem{asss}{Assumption}
\newtheorem{defns}{Definition}

\newcommand{\btheos}{\begin{theos}}
\newcommand{\etheos}{\end{theos}}
\newcommand{\bprops}{\begin{props}}
\newcommand{\eprops}{\end{props}}
\newcommand{\bdes}{\begin{defns}}
\newcommand{\edes}{\end{defns}}
\newcommand{\blems}{\begin{lems}}
\newcommand{\elems}{\end{lems}}
\newcommand{\bcors}{\begin{cors}}
\newcommand{\ecors}{\end{cors}}
\newcommand{\bexs}{\begin{exas}}
\newcommand{\eexs}{\end{exas}}
\newcommand{\balgs}{\begin{algs}}
\newcommand{\ealgs}{\end{algs}}
\newcommand{\bass}{\begin{asss}}
\newcommand{\eass}{\end{asss}}

\title{Accelerating Convergence of \\ Score-Based Diffusion Models, Provably}

\author{ \qquad~~~Gen Li\footnote{The first two authors contributed equally.} \thanks{Department of Statistics, The Chinese University
of Hong Kong, Hong Kong.}  \qquad \\
 \qquad~~ CUHK  \qquad 
\and \qquad\qquad~~ Yu Huang\footnotemark[1] \thanks{Department of Statistics and Data Science, Wharton School, University
of Pennsylvania, Philadelphia, PA 19104, USA.} \qquad\qquad \\
\qquad\qquad~~~ UPenn \qquad\qquad
\and \qquad Timofey Efimov\thanks{Department of Electrical and Computer Engineering, Carnegie Mellon University, Pittsburgh, PA 15213, USA.}  \qquad \\
\qquad CMU \qquad 
\and
~~~~~Yuting Wei\footnotemark[3] \qquad\qquad \\
 ~~~~~UPenn \qquad\qquad
\and   \qquad Yuejie Chi\footnotemark[4] \qquad\qquad \\
\quad~~~  CMU \qquad\qquad \\ 
\and ~~\qquad Yuxin Chen\footnotemark[3] ~\qquad \\
\qquad ~~UPenn \qquad
 }

\date{\today}

\makeatother

\begin{document}

\theoremstyle{plain} \newtheorem{lemma}{\textbf{Lemma}}\newtheorem{proposition}{\textbf{Proposition}}\newtheorem{theorem}{\textbf{Theorem}}

\theoremstyle{assumption}\newtheorem{assumption}{\textbf{Assumption}}
\theoremstyle{remark}\newtheorem{remark}{\textbf{Remark}}
\theoremstyle{definition}\newtheorem{definition}{\textbf{Definition}}

\maketitle

\begin{abstract}	
	Score-based diffusion models, while achieving remarkable empirical performance, often suffer from low sampling speed, due to extensive function evaluations needed during the sampling phase.  Despite a flurry of recent activities towards speeding up diffusion generative modeling in practice, 
	 theoretical underpinnings for acceleration techniques remain severely limited. 
	 In this paper, we design novel training-free algorithms to accelerate popular deterministic (i.e., DDIM) and stochastic (i.e., DDPM) samplers.  
	 Our accelerated deterministic sampler converges at a rate $O(\frac{1}{{T}^2})$ with $T$ the number of steps, improving upon the $O(\frac{1}{T})$ rate for the DDIM sampler; 
	 and our accelerated stochastic sampler converges at a rate $O(\frac{1}{T})$, outperforming the rate $O(\frac{1}{\sqrt{T}})$ for the DDPM sampler. 
	 The design of our algorithms leverages insights from higher-order approximation, and shares similar intuitions as popular high-order ODE solvers like the DPM-Solver-2. 
	Our theory accommodates $\ell_2$-accurate score estimates, and does not require log-concavity or smoothness on the target distribution. 
\end{abstract}




\tableofcontents

\section{Introduction}
\label{sec:intro}

Initially introduced by \citet{sohl2015deep} and subsequently gaining momentum through the works \citet{ho2020denoising,song2020score}, 
diffusion models have risen to the forefront of generative modeling. 
Remarkably, score-based diffusion models have demonstrated superior performance across various domains like computer vision,  natural language processing, medical imaging, and bioinformatics 
\citep{croitoru2023diffusion,yang2023diffusion,kazerouni2023diffusion,guo2023diffusion}, 
outperforming earlier generative methods such as GANs \citep{goodfellow2020generative} and VAEs \citep{kingma2013auto} on multiple fronts \citep{dhariwal2021diffusion}.

\subsection{Score-based diffusion models}

On a high level, 
diffusion-based generative modeling begins by considering a forward Markov diffusion process that progressively diffuses a data distribution into noise:  
\begin{equation}
	X_0 \overset{\text{add noise}}{\longrightarrow} X_1 \overset{\text{add noise}}{\longrightarrow} 
X_2 \overset{\text{add noise}}{\longrightarrow} \cdots \overset{\text{add noise}}{\longrightarrow} X_T,
	\label{eq:forward-process-informal}
\end{equation}
where $X_0 \sim p_{\mathsf{data}}$ is drawn from the target data distribution   in $\mathbb{R}^d$, and $X_T$ resembles pure noise (e.g., with a distribution close to $\mathcal{N}(0,I_d)$). 
The pivotal step then lies in learning to construct a reverse Markov process 
\begin{equation}
	Y_0 \overset{\text{use scores}}{\longleftarrow} Y_1 \overset{\text{use scores}}{\longleftarrow} 
Y_2 \overset{\text{use scores}}{\longleftarrow} \cdots \overset{\text{use scores}}{\longleftarrow} Y_T ,
\label{eq:reverse-MC}
\end{equation}
which starts from purse noise $Y_T\sim \mathcal{N}(0,I_d)$ and maintains distributional proximity throughout in the sense that $Y_t \overset{\mathrm{d}}{\approx}  X_t$ ($t\leq T$). 
To accomplish this goal,  $Y_{t-1}$ in  each step is typically obtained from $Y_t$ with the aid of (Stein) score functions 
--- namely, $\nabla_X \log p_{X_t}(X)$, with $p_{X_t}$ denoting the distribution of $X_t$ ---  
where the score functions are pre-trained by means of score matching techniques (e.g., \citet{hyvarinen2005estimation,ho2020denoising,hyvarinen2007some,vincent2011connection,song2019generative,pang2020efficient}).

The mainstream approaches for constructing the reverse-time process \eqref{eq:reverse-MC} can roughly be  divided into two categories, as described below. 
\begin{itemize}
	\item {\em Stochastic (or SDE-based) samplers.} 
	A widely adopted strategy  
involves exploiting both the score function and some injected random noise 
when generating each $Y_{t-1}$; 
that is,  
	 $Y_{t-1}$ is taken to be a function of $Y_t$ and some independent noise $Z_t$.  
	A prominent example of this kind is the Denoising Diffusion Probabilistic Model (DDPM) 
	 \citep{ho2020denoising}, to be detailed in Section~\ref{sec:background}.  
	Notably, this approach has intimate connections with certain stochastic differential equations (SDEs), 
	which can be elucidated via celebrated SDE results concerning the existence of reverse-time diffusion processes \citep{anderson1982reverse,haussmann1986time}.

	\item {\em Deterministic (or ODE-based) samplers.} 
In contrast, another approach is purely deterministic (except for the generation of $Y_T$), 
constructing $Y_{t-1}$ as a function of the previously computed steps (e.g., $Y_t$) without injecting any additional noise. 
This approach was  introduced by \citet{song2020score}, 
		as inspired by the existence of ordinary differential equations (ODEs) --- termed {\em probability flow ODEs} --- exhibiting the same marginal distributions as the above-mentioned reverse-time diffusion process. 
A notable example in this category is often referred to as the Denoising Diffusion Implicit Model (DDIM) \citep{song2020denoising}.

\end{itemize}
\noindent 
In practice, it is often observed that DDIM converges more rapidly than DDPM, although the final data instances produced by DDPM (given sufficient runtime) 
enjoy better diversity compared to the output of DDIM.

\subsection{Non-asymptotic convergence theory and acceleration}

Despite the astounding empirical success, 
theoretical analysis for diffusion-based generative modeling is still in its early stages of development.  
Treating the score matching step as a blackbox and exploiting only (crude) information about the score estimation error, 
a recent strand of works have explored the convergence rate of the data generating process (i.e., the reverse Markov process) in a non-asymptotic fashion, 
in an attempt to uncover how fast sampling can be performed (e.g., \citet{lee2022convergence,lee2023convergence,chen2022sampling,chen2023improved,chen2023restoration,chen2023probability,li2023towards,benton2023error,benton2023linear,liang2024non}).  
In what follows, let us give a brief overview of the state-of-the-art results in this direction. 
Here and throughout, the iteration complexity of a sampler refers to the number of steps $T$ needed to attain $\varepsilon$ accuracy 
in the sense that $\mathsf{TV}(p_{X_1},p_{Y_1}) \leq \varepsilon$, 
where $\mathsf{TV}(\cdot,\cdot)$ represents the total-variation (TV) distance between two distributions, and $p_{X_1}$ (resp.~$p_{Y_1}$) stands for the distribution of $X_1$ (resp.~$Y_1$). 
\begin{itemize}
	\item {\em Convergence rate of stochastic samplers.} 
		Assuming Lipschitz continuity (or smoothness) of the score functions across all steps, 
	\citet{chen2022sampling} proved that the iteration complexity 
	of the DDPM sampler is proportional to $1/ \varepsilon^2$. 
	The Lipschitz assumption is then relaxed by \citet{chen2023improved,benton2023linear,li2023towards}, 
	revealing that the scaling $1/ \varepsilon^2$ is achievable for a fairly general family of data distributions.

	\item {\em Convergence rate of deterministic samplers.} 
	As alluded to previously, deterministic samplers often exhibit faster convergence in both practice and theory. 
	For instance, \citet{chen2023restoration} provided  the first polynomial convergence guarantees for the probability flow ODE sampler under exact scores, 
	whereas \citet{li2023towards} demonstrated that its iteration complexity scales proportionally to $1/\varepsilon$ allowing score estimation error. 
	Additionally, it is noteworthy that an iteration complexity proportional to $1/\varepsilon$ has also been established 
	by \citet{chen2023probability} for a variant of the probability flow ODE sampler, 
	although the sampler studied therein incorporates a stochastic corrector step in each iteration. 

\end{itemize}

\paragraph{Acceleration?} 
While the theoretical studies outlined above have offered non-asymptotic convergence guarantees for both the stochastic and deterministic samplers, 
one might naturally wonder whether there is potential for achieving faster rates.  
In practice, the evaluation of Stein scores in each step often entails computing the output of a large neural network, 
thereby calling for new solutions to reduce the number of score evaluations without compromising sampling fidelity.  
Indeed, this has inspired a large strand of recent works focused on speeding up diffusion generative modeling. 
Towards this end, one prominent approach is {\em distillation}, 
which attempts to distill a pre-trained diffusion model into another model (e.g., progressive distillation, consistency model) that can be executed in significantly fewer steps \citep{luhman2021knowledge,salimans2021progressive,meng2023distillation,song2023consistency}. 
However, while distillation-based techniques have achieved outstanding empirical performance, 
they often necessitate additional training processes, imposing high computational burdens beyond score matching. 
In contrast, an alternative route towards acceleration is ``training-free,''  
which directly invokes the pre-trained diffusion model (particularly the pre-trained score functions) for sampling without requiring additional training processes.  
Examples of training-free accelerated samplers include DPM-Solver \citep{lu2022dpm},  DPM-Solver++ \citep{lu2022dpmv2}, DEIS \citep{zhang2022fast}, 
UniPC \citep{zhao2023unipc}, the SA-Solver \citep{xue2023sa}, among others, 
which leverage faster solvers for ODE and SDE using only the pre-trained score functions. 
Nevertheless, non-asymptotic convergence analyses for these methods remain largely absent, 
making it challenging to rigorize the degrees of acceleration compared to the non-accelerated results \citep{lee2023convergence,chen2022sampling,chen2023improved,li2023towards,benton2023linear}. 
All of this leads to the following question that we aim to explore in this work: 
{
\setlist{rightmargin=\leftmargin}
\begin{itemize}[topsep=5pt, itemsep=0pt]
	\item[] {\em Can we design a training-free deterministic (resp.~stochastic) sampler that converges provably faster than the DDIM (resp.~DDPM)? }
\end{itemize}
}

\subsection{Our contributions}
In this paper, we answer the above question in the affirmative. Our main contributions can be summarized as follows.  
\begin{itemize}
	\item  In the deterministic setting, we demonstrate how to speed up the ODE-based sampler (i.e., the DDIM-type sampler).  
		The proposed sampler, which exploits some sort of momentum term to adjust the update rule, 
		leverages insights from higher-order ODE approximation in discrete time and shares similar intuitions with the fast ODE-based sampler DPM-Solver-2~\citep{lu2022dpm}.
	We establish non-asymptotic convergence guarantees for the accelerated DDIM-type sampler, showing that its iteration complexity  scales proportionally to 
		$1/\sqrt{\varepsilon}$ (up to log factor). This substantially improves upon the prior convergence theory for the original DDIM sampler~\citep{li2023towards} 
		(which has an iteration complexity proportional to $1/\varepsilon$). 

	\item In the stochastic setting, we propose a novel sampling procedure to accelarate the SDE-based sampler (i.e., the DDPM-type sampler). 
		For this new sampler, we  establish an iteration complexity bound proportional to $1/\varepsilon$ (modulo some log factor), 
		thus unveiling the superiority of the proposed sampler compared to the original DDPM sampler (recall that the original DDPM sampler has an iteration complexity proportional to $1/\varepsilon^2$~\citep{li2023towards,chen2023improved,chen2022sampling}). 
\end{itemize}
In addition, 
two aspects of our theory are worth emphasizing: 
(i) our theory accommodates $\ell_2$-accurate score estimates, rather than requiring $\ell_{\infty}$ score estimation accuracy; 
(ii) our theory covers a fairly general family of target data distributions, without imposing stringent assumptions like log-concavity and smoothness on the target  distributions.

\subsection{Other related works}

We now briefly discuss additional related works in the prior art. 

\paragraph{Convergence of score-based generative models (SGMs).}   For stochastic samplers of SGMs, the convergence guarantees were initially provided by early works including but not limited to~\citet{de2021diffusion,liu2022let, pidstrigach2022score,block2020generative,de2022convergence,wibisono2022convergence,gao2023wasserstein}, which often faced issues of either being not quantitative or suffering from the curse of dimensionality. 
More recent research has advanced this field by relaxing the
assumptions on the score function and achieving polynomial convergence rates~\citep{lee2022convergence,lee2023convergence,chen2022sampling,chen2023improved,chen2023probability,li2023towards,benton2023linear,liang2024non,tang2024score}.  Furthermore, theoretical insights into probability flow-based ODE samplers, though less abundant, have been explored in recent works~\citep{chen2023restoration,li2023towards,chen2023probability,benton2023error,gao2024convergence}. Additionally, \citet{tang2024contractive} provided a continuous-time sampling error guarantee  for a novel class of contraction diffusion models. \citet{gao2024convergence} studies the convergence properties for general probability flow ODEs w.r.t.~Wasserstein distances. Most recently, \citet{chen2024convergence} makes a step towards the convergence analysis of discrete state space diffusion model.  Note that this body of research primarily aims to quantify the proximity between distributions generated by SGMs and the ground truth distributions, assuming availability of an accurate score estimation oracle.   Interestingly, a very recent research by \citet{li2024good} reveals that even SGMs with empirically optimized score functions might underperform due to strong memorization effects.  Moreover, some works delve into other aspects of the theoretical understanding of diffusion
models. 
Furthermore, \citet{wu2024theoretical} investigated how diffusion guidance combined with DDPM and DDIM samplers influences the conditional sampling quality.

\paragraph{Fast sampling in diffusion models.} A recent strand of works to achieve few-step sampling --- or even one-step sampling --- falls under the category of training-based samplers,  primarily focused  on knowledge distillation~\citep{meng2023distillation,salimans2021progressive,song2023consistency}. This method aims to distill a pre-trained diffusion model into another model that can be executed in significantly fewer steps. 
The recent work \cite{li2024consistency} provided a first attempt towards theoretically understanding  the sampling efficiency of consistency models. 
Another line of works aims to design training-free samplers~\citep{lu2022dpm,lu2022dpmv2,zhao2023unipc,zhang2022fast,liu2022pseudo,zhang2022gddim}, which addresses the efficiency issue by developing faster solvers for the  reverse-time SDE or ODE without requiring other information beyond the pre-trained SGMs. In addition, 
\citet{li2023towards,liang2024non} introduced accelerated samplers that require additional training pertaining to estimating Hessian information at each step. 
Furthermore, combining GAN with diffusion has shown to be an effective strategy  to speed up the sampling process~\citep{wang2022diffusion, xiao2021tackling}.  




\subsection{Notation} 

Before continuing, we find it helpful to introduce some notational conventions to be used throughout this paper. 
Capital letters are often used to represent random variables/vectors/processes, while lowercase letters denote deterministic variables. 
When considering two probability measures $P$ and $Q$, we define their total-variation (TV) distance as $\mathsf{TV}(P,Q)\coloneqq \frac{1}{2}\int |\mathrm{d}P - \mathrm{d}Q|$, 
and the Kullback-Leibler (KL) divergence as $\mathsf{KL}(P\,\|\,Q)\coloneqq \int  \big(\log \frac{\mathrm{d}P}{ \mathrm{d}Q} \big) \mathrm{d}P  $. 
We use $p_X(\cdot)$ and $p_{X\mymid Y}(\cdot \mymid \cdot)$ to denote the probability density function of a random vector $X$, and the conditional probability of $X$ given $Y$, respectively. 
For matrices, $\|A\|$ and $\|A\|_{\mathrm{F}}$ refer to the spectral norm and Frobenius norm of a matrix $A$, respectively. For vector-valued functions $f$, 
we use $J_f$ or $\frac{\partial f}{\partial x}$ to represent the Jacobian matrix of $f$.  Given two functions $f(d,T)$ and $g(d,T)$, we employ the notation  $f(d,T)\lesssim g(d,T)$ or $f(d,T)=O( g(d,T) )$ (resp.~$f(d,T)\gtrsim g(d,T)$) to indicate the existence of a universal constant $C_1 > 0$ such that for all $d$ and $T$, $f(d,T)\leq C_1g(d,T)$ (resp.~$f(d,T)\geq C_1g(d,T)$). The notation $f(d,T)\asymp g(d,T)$ indicates that both $f(d,T)\lesssim g(d,T)$ and $f(d,T)\gtrsim g(d,T)$ hold at once.



\section{Problem settings}
\label{sec:background}

In this section, we formulate the problem, and introduce a couple of key assumptions.

\subsection{Model and sampling process}

\paragraph{Forward process.}
Consider the forward Markov process \eqref{eq:forward-process-informal} in discrete time 
that 
starts from the target data distribution $X_0 \sim \Pdata$ in $\mathbb{R}^d$ and proceeds as follows:  
%
\begin{align}
	\label{eq:forward-process}
	X_t &= \sqrt{1-\beta_t}X_{t-1} + \sqrt{\beta_t}\,W_{t}, \qquad t= 1,\cdots, T, 
\end{align}
%
where the $W_t$'s are independently drawn from $\mathcal{N}(0,I_d)$. 
This process is said to be ``variance-preserving,'' 
in the sense that the covariance $\mathsf{Cov}(X_t)=I_d$ holds throughout if  $\mathsf{Cov}(X_0)=I_d$.  
Taking 
\begin{align} 
	\overline{\alpha}_t \coloneqq \prod_{k = 1}^t \alpha_k 
	\qquad \text{with }\alpha_t\coloneqq 1 - \beta_t
	\label{eq:defn-alpha-bar-t}
\end{align}
for every $1\leq t\leq T$, one can write
\begin{align}
\label{eqn:Xt-X0}
	X_t = \sqrt{\overline{\alpha}_t} X_{0} + \sqrt{1- \overline{\alpha}_t} \,\overline{W}_{t}
	\quad \text{for } \overline{W}_{t}\sim \mathcal{N}(0,I_d) .
\end{align}
Throughout the paper, we shall use $q_t(\cdot)$ or $p_{X_t}(\cdot)$ interchangeably to denote the probability density function (PDF) of $X_t$. 
While we shall concentrate on the discrete-time process in the current paper, 
we shall note that the forward process has also been commonly studied in the continuous-time limit through the following diffusion process
\begin{equation}
	\mathrm{d}X_{t}=-\frac{1}{2}\beta(t)X_{t}\mathrm{d}t+\sqrt{\beta(t)}\mathrm{d}W_{t} \quad(0 \leq t \leq T), \quad X_0 \sim p_{\text {data }}
	\label{eq:forward-diffusion}
\end{equation}
 %
for some function $\beta(t)$ related to the learning rate, 
 where $W_t$ is the standard Brownian motion.

\paragraph{Score functions and score estimates.} 
A key ingredient that plays a pivotal role in the sampling process is  
the (Stein) score function, defined as the log marginal density of the forward process. 
\begin{definition}[Score function]
	\label{defn:score-function}
	The score function, denoted by $s_t^{\star}: \mathbb{R}^d \rightarrow \mathbb{R}^d(1 \leq t \leq T)$, is defined as
\begin{align}
\label{eq:defn-score-true}
	s_t^{\star}(X)	&\coloneqq 
\nabla \log q_t(X)
	=-\frac{1}{1 - \overline{\alpha}_t}\int_{x_0} p_{X_0 \mymid X_{t}}(x_0 \mymid x)\big(x - \sqrt{\overline{\alpha}_{t}}x_0\big) \mathrm{d} x_0.
\end{align}
\end{definition}
\noindent 
Here, the last identity follows from standard properties about Gaussians; see, e.g., \citet{chen2022sampling}. 
In most applications, we have no access to perfect score functions; 
instead, what we have available are certain estimates for the score functions, 
to be denoted by $\{s_t(\cdot)\}_{1\leq t\leq T}$ throughout.

\paragraph{Data generation process.}  
The sampling process is performed via careful construction of the reverse process \eqref{eq:reverse-MC} to ensure distributional proximity. 
Working backward from $t=T,\dots,1$, we assume throughout that $Y_T\sim \mathcal{N}(0,I_d)$. 
\begin{itemize}[leftmargin=*, topsep=5pt, itemsep=0pt]
	\item {\em Deterministic sampler.} 
		A deterministic sampler typically chooses $Y_{t-1}$ for each $t$ to be a function of $\{Y_t,\dots, Y_T\}$. 
		For instance, the following construction 
		\begin{equation}
			Y_{t-1}=\frac{1}{\sqrt{\alpha_{t}}}\left(Y_{t}+\frac{1-\alpha_{t}}{2}s_{t}(Y_{t})\right), 
			\qquad t=T,\dots,1
			\label{eq:original-DDIM}
		\end{equation}
		%
		can be viewed as a DDIM-type sampler in discrete time. 
		Note that the DDIM sampler is intimately connected with the following ODE --- called the probability flow ODE or the diffusion ODE --- in the continuous-time limit: 
		\begin{equation}
\mathrm{d}\widetilde{Y}_{t}=-\frac{1}{2}\beta\big(t\big)\left(\widetilde{Y}_{t}+\nabla\log q_{t}\big(\widetilde{Y}_{t}\big)\right)\mathrm{d}t,
			\quad\widetilde{Y}_{T}\sim q_{T},
			\label{eq:diffusion-ODE}
		\end{equation}
		which enjoys matching marginal distributions as the forward diffusion process \eqref{eq:forward-diffusion} in the sense that $\widetilde{Y}_{t}\overset{\mathrm{d}}{=} X_t$ for all $0\leq t\leq T$
		\citep{song2020score}.

	\item {\em Stochastic sampler.}
		In contrast to the deterministic case, each $Y_{t-1}$ is a function of not only $\{Y_t,\dots, Y_T\}$ but also an additional independent noise $Z_t\sim \mathcal{N}(0,I_d)$.  
		One example is the following sampler:  
		\begin{equation}
			Y_{t-1}=\frac{1}{\sqrt{\alpha_{t}}}\Big(Y_{t}+(1-\alpha_{t})s_{t}(Y_{t})+\sqrt{1-\alpha_{t}}Z_{t}\Big),  \quad t=T,\dots,1
			\label{eq:original-DDPM}
		\end{equation}
		which is closely related to the DDPM sampler in discrete time. 
		The design of DDPM draws inspiration from a well-renowned result in the SDE literature \citep{anderson1982reverse,haussmann1986time}; namely, there exists a reverse-time SDE 
		\begin{align}
			\mathrm{d}\widehat{Y}_{t}&=-\frac{1}{2}\beta\big(t\big)\left(\widehat{Y}_{t}+2\nabla\log q_{t}\big(\widehat{Y}_{t}\big)\right)\mathrm{d}t 
			+\sqrt{\beta(t)}\mathrm{d}\widehat{Z}_{t}	
			\label{eq:diffusion-SDE}
		\end{align}
		with $\widehat{Y}_{T}\sim q_{T} \notag$ that exhibits the same marginals --- $\widehat{Y}_{t}\overset{\mathrm{d}}{=} X_t$ for all $t$ --- as the forward diffusion process \eqref{eq:forward-diffusion}. 
		Here, $\widehat{Z}_t$ indicates a backward standard Brownian motion. 
		
\end{itemize}

\subsection{Assumptions} 

Before moving on to our algorithms and theory, 
let us introduce several assumptions that shall be used multiple times in this paper. 
To begin with, we impose the following assumption on the target data distribution. 
\begin{assumption}\label{assump:assumption-data-bounded}
    Suppose that $X_0$ is a continuous random vector,  
    and obeys
\begin{equation}
\mathbb{P}\big(\|X_0\|_2 \leq R = T^{c_R} \mid X_0\sim \Pdata \big)=1
\label{eq:assumption-data-bounded}
\end{equation}
for some arbitrarily large constant $c_R>0$. 
\end{assumption}
\noindent 
In words, the size of $X_0$ is allowed to grow polynomially (with arbitrarily large constant degree) in the number of steps, 
which suffices to accommodate the vast majority of practical applications.

Next,  we specify the learning rates $\{\beta_t\}$ (or $\{\alpha_t\}$) employed in the forward process \eqref{eq:forward-process}. 
Throughout this paper, we select the same learning rate schedule as in~\citet{li2023towards}, namely, 
\begin{subequations}
	\begin{align}
	   \beta_1&=1-\alpha_1=\frac{1}{T^{c_0}}, \\
		\beta_t&=1-\alpha_t =\frac{c_1 \log T}{T} \min \left\{\beta_1\left(1+\frac{c_1 \log T}{T}\right)^t, 1\right\}, \quad t>1 \label{eqn:alpha-t}
	   \end{align}   
\end{subequations}
for some large enough numerical constants $c_0, c_1>0$. 
In short, there are two phases here: at first $\beta_t$ grows exponentially fast, and then stays unchanged after surpassing some threshold. 
This also resembles the learning rate choices recommended by \citet{benton2023linear}.

Moreover, let us also introduce two assumptions regarding the accuracy of the score estimates $\{s_t\}$, which are adopted in \citet{li2023towards}.  
Here and throughout, 
we denote by 
\begin{align}
	J_{s_t^{\star}}=\frac{\partial s_t^{\star}}{\partial x}
	\qquad \text{and} \qquad
	J_{s_t}=\frac{\partial s_t}{\partial x}
\end{align}
the Jacobian matrices of $s_t^{\star}(\cdot)$ and $s_t(\cdot)$, respectively. 
\begin{assumption}\label{assumption:score-estimate}
 Suppose that the mean squared estimation error of the score estimates $\left\{s_t\right\}_{1 \leq t \leq T}$ obeys
$$
\frac{1}{T} \sum_{t=1}^T \underset{X \sim q_t}{\mathbb{E}}\left[\left\|s_t(X)-s_t^{\star}(X)\right\|_2^2\right] \leq \varepsilon_{\mathsf {score }}^2 .
$$
\end{assumption}
\begin{assumption}\label{assumption:score-estimate-Jacobi}
	Suppose that $s_t(\cdot)$ is continuously differentiable for each $1\leq t\leq T$, and that the Jacobian matrices associated with the score estimates $\left\{s_t\right\}_{1 \leq t \leq T}$ satisfy
$$
\frac{1}{T} \sum_{t=1}^T \underset{X \sim q_t}{\mathbb{E}}\left[\left\|J_{s_t}(X)-J_{s_t^{\star}}(X)\right\|\right] \leq \varepsilon_{\mathsf {Jacobi }}. 
$$
\end{assumption}
\noindent 
In short, Assumption~\ref{assumption:score-estimate} 
is concerned with the $\ell_2$ score estimation error averaged across all steps, 
whereas Assumption~\ref{assumption:score-estimate-Jacobi} is about the average discrepancy in the associated Jacobian matrices. 
It is worth noting that Assumption~\ref{assumption:score-estimate-Jacobi} will 
only be imposed when analyzing the convergence of deterministic samplers, 
and is completely unnecessary for the stochastic counterpart.

\section{Algorithm and main theory}

In this section, we put forward two accelerated samplers --- an ODE-based algorithm and an SDE-based algorithm --- and present convergence theory to confirm the acceleration compared with prior DDIM and DDPM approaches.

\subsection{Accelerated ODE-based sampler}
\label{sec:accelerated_ODE}

The first algorithm we propose is an accelerated variant of the ODE-based deterministic sampler. 
Specifically,  starting from $Y_T \sim \mathcal{N}(0, I_d)$, the proposed discrete-time sampler adopts the following update rule:
\begin{subequations}
\label{eqn:ode-extragradient}
\begin{equation}\label{eqn:ode-extragradient-Y}
 Y_{t}^- = \Phi_t(Y_t),
	\qquad Y_{t-1} = \Psi_t(Y_t, Y_{t}^-) \quad~ \text{for } t= T,\cdots,1 
\end{equation}
where the mappings $\Phi_t(\cdot)$ and $\Psi_t(\cdot, \cdot)$ are chosen to be
\begin{align}
\Phi_{t}(x) =& \sqrt{\alpha_{t+1}}\bigg(x-\frac{1-\alpha_{t+1}}{2}s_{t}(x)\bigg), \label{eqn:ode-extragradient-Phi} \\
	\Psi_{t}(x, y) =& \frac{1}{\sqrt{\alpha_{t}}}\bigg(x+\frac{1-\alpha_{t}}{2}s_{t}(x) +\frac{(1-\alpha_{t})^2}{4(1-\alpha_{t+1})}\big(s_{t}(x) - \sqrt{\alpha_{t+1}}s_{t+1}(y)\big)\bigg), \label{eqn:ode-extragradient-Psi}
\end{align}
\end{subequations}
and we remind the reader that $s_{t}$ is the score estimate. 
In contrast to the original DDIM-type solver \eqref{eq:original-DDIM}, 
the proposed accelerated sampler enjoys two distinguishing features: 
\begin{itemize}
	\item In each iteration $t$, the proposed sampler computes a mid-point $Y_t^-=\Phi_t(Y_t)$ (cf.~\eqref{eqn:ode-extragradient-Phi}). 
		As it turns out, this mid-point is designed as a prediction of the probability flow ODE at time $t+1$ using $Y_t$. 
		
	\item In contrast to \eqref{eq:original-DDIM}, 
		the proposed update rule $Y_{t-1} = \Psi_t(Y_t, Y_{t}^-)$ (see \eqref{eqn:ode-extragradient-Psi}) 
		includes an additional term that is a properly scaled version of  
		$s_{t}(Y_t) - \sqrt{\alpha_{t+1}}s_{t+1}(Y_t^-)$. 
		In some sense, this term can be roughly viewed as 
		exploiting ``momentum'' in adjusting the original sampling rule.

\end{itemize}

\paragraph{Theoretical guarantees.} 
Let us proceed to present our convergence theory and its implications for the proposed deterministic sampler. 
%
\begin{theorem}\label{thm:main}
	Suppose that Assumptions~\ref{assump:assumption-data-bounded}, \ref{assumption:score-estimate} and \ref{assumption:score-estimate-Jacobi} hold. 
	Then the proposed  sampler \eqref{eqn:ode-extragradient} with the learning rate schedule \eqref{eqn:alpha-t} satisfies
 \begin{align}\label{eq:ratio-ODE}
	 \mathsf{TV}\big(q_1, p_1\big) &\leq C_1\frac{d^{6}\log^{6}T}{T^{2}}+C_1\sqrt{d\log^{3}T}\varepsilon_{\score}+C_1(d\log T)\varepsilon_{\Jacobi}
	\end{align}
	for some universal constants $C_1>0$, where we recall that $p_1$ (resp.~$q_1$) denotes the distribution of $Y_1$ (resp.~$X_1$).
\end{theorem}
We now take a moment to discuss the implications about this theorem. 
%
\begin{itemize}
	\item {\em Iteration complexity.} 
		When the target accuracy level $\varepsilon$ is small enough, the number of iterations needed to yield $\mathsf{TV}\big(q_1, p_1\big) \leq \varepsilon$ 
		is no larger than 
		\begin{equation}
			(\text{iteration complexity}) \qquad \frac{\mathrm{poly}(d)}{\sqrt{\varepsilon}},
		\end{equation}
		ignoring any logarithmic factor in $1/\varepsilon$. 
		Clearly, the dependency on $1/\varepsilon$ substantially improves upon the vanilla DDIM sampler, the latter of which has an iteration complexity proportional to $1/\varepsilon$ \citep{li2023towards}. 

	\item {\em Stability vis-a-vis score errors.} 
		The discrepancy between the distribution of $Y_1$  and the target distribution of $X_1$ 
		is proportional to the $\ell_2$ score estimation error $\varepsilon_{\score}$ defined in Assumption~\ref{assumption:score-estimate}, as well as the Jacobian error $\varepsilon_{\Jacobi}$ defined in Assumption~\ref{assumption:score-estimate-Jacobi}. 
		It is worth noting, however, that the same result might not hold if we remove Assumption~\ref{assumption:score-estimate-Jacobi}. 
		More specifically, when only score estimation accuracy is assumed, 
		the deterministic sampler is not guaranteed to achieve small TV error; see \citet{li2023towards} for an illustrative example.  
\end{itemize}


\paragraph{Interpretation via second-order ODE.} 
In order to help elucidate the rationale of the proposed sampler, 
we make note of an intimate connection between \eqref{eqn:ode-extragradient} and high-order ODE, 
the latter of which has facilitated the design of fast deterministic samplers (e.g., DPM-Solver \citep{lu2022dpm}).

In view of the relation \eqref{eqn:Xt-X0}, for any $0<\gamma<1$,  let us first  abuse the notation and introduce 
\begin{subequations}
	\begin{align}
		&X(\gamma) \overset{\mathrm{d}}{=} \sqrt{\gamma}X_{0}+\sqrt{1-\gamma}Z, \quad Z\sim \mathcal{N}(0,I_d)  \\
		&s^{\star}_{\gamma}(X) \coloneqq \nabla_X \log p_{X(\gamma)}(X). 
	\end{align}
	We further consider the following continuous-time analog $\overline{\alpha}(t)$ of the discrete learning rate $\overline{\alpha}_t$ (cf. \eqref{eq:defn-alpha-bar-t}):
	\begin{align}
		\frac{\mathrm{d}\overline{\alpha}(t)}{\mathrm{d}t}=-\beta(t) \overline{\alpha}(t),\quad \overline{\alpha}(T)=\overline{\alpha}_T. \label{eq: conti-alpha}
	\end{align}
\end{subequations}
Given that the probability flow ODE \eqref{eq:diffusion-ODE} yields identical marginal distributions as the forward process $X_t$ (cf.~\eqref{eq:forward-diffusion}) for every $t$, invoking \eqref{eq: conti-alpha}, 
we can easily see that $X(\overline{\alpha}(t)) \overset{\mathrm{d}}{=} X_t$ can be generated as follows: 
\begin{align}
\label{eq:ODE-alpha}
	\frac{\mathrm{d}X\big(\overline{\alpha}(t)\big)}{\mathrm{d}\overline{\alpha}(t)}
	=\frac{1}{2\overline{\alpha}(t)}\bigg(X\big(\overline{\alpha}(t)\big)+s^{\star}_{\overline{\alpha}(t)}\Big(X\big(\overline{\alpha}(t)\big)\Big)\bigg),
	 \quad X\big(\overline{\alpha}(T)\big) \sim q_T,
\end{align}
%
%
By taking $f\big(\gamma\big)=\frac{1}{\sqrt{\gamma}}X\big(\gamma\big)$, we can apply \eqref{eq:ODE-alpha} to derive
\[
\frac{\mathrm{d}f\big(\gamma\big)}{\mathrm{d}\gamma}=-\frac{1}{2\sqrt{\gamma^{3}}}X\big(\gamma\big)+\frac{1}{\sqrt{\gamma}}\frac{\mathrm{d}X\big(\gamma\big)}{\mathrm{d}\gamma}=\frac{1}{2\sqrt{\gamma^{3}}}s_{\gamma}^{\star}\Big(X\big(\gamma\big)\Big). 
\]
This taken together with $\overline{\alpha}_{t}=\overline{\alpha}_{t-1} \alpha_t$ (cf.~\eqref{eq:defn-alpha-bar-t}) immediately implies that
\begin{align}
	&\frac{1}{\sqrt{\overline{\alpha}_{t-1}}}X(\overline{\alpha}_{t-1})=\frac{1}{\sqrt{\overline{\alpha}_{t}}}X(\overline{\alpha}_{t})+\frac{1}{2}{\displaystyle \int}_{\overline{\alpha}_{t}}^{\overline{\alpha}_{t-1}}\frac{1}{\sqrt{\gamma^{3}}}s_{\gamma}^{\star}\big(X(\gamma)\big)\mathrm{d}\gamma,\notag\\
	\Longrightarrow\quad 
	&X(\overline{\alpha}_{t-1})=
\frac{1}{\sqrt{\alpha_{t}}}X(\overline{\alpha}_{t})+\frac{\sqrt{\overline{\alpha}_{t-1}}}{2}{\displaystyle \int}_{\overline{\alpha}_{t}}^{\overline{\alpha}_{t-1}}\frac{1}{\sqrt{\gamma^{3}}}s_{\gamma}^{\star}\big(X(\gamma)\big)\mathrm{d}\gamma. 
	\label{eq:X-alphat-int}
\end{align}

%
%
%
%
With this relation in mind, we are ready to discuss the following approximation in discrete time: 
%
\begin{itemize}
	\item {\em Scheme 1: } If we approximate $s_{\gamma}^{\star}\big(X(\gamma)\big)$ for $\gamma\in [\overline{\alpha}_{t},\overline{\alpha}_{t-1}]$  by 
		$s_{\gamma}^{\star}\big(X(\gamma)\big)\approx s_{\overline{\alpha}_t}^{\star}\big(X(\overline{\alpha}_t)\big) \approx s_{t}(X_{t})$, then we arrive at
		\begin{align*}
			X(\overline{\alpha}_{t-1})&\approx\frac{1}{\sqrt{\alpha_{t}}}X(\overline{\alpha}_{t})+\left(\frac{\sqrt{\overline{\alpha}_{t-1}}}{\sqrt{\overline{\alpha}_{t}}}-1\right)s_{t}(X_{t})\\
			&\approx\frac{1}{\sqrt{\alpha_{t}}}\left\{ X(\overline{\alpha}_{t})+\frac{1-\alpha_{t}}{2}s_{t}(X_{t})\right\} ,
		\end{align*}
		%
		where we use the facts that $\overline{\alpha}_{t}/\overline{\alpha}_{t-1}=\alpha_t$ and $\alpha_t\approx 1$. This coincides with the deterministic sampler \eqref{eq:original-DDIM}.

	\item {\em Scheme 2: } If we invoke a more refined approximation for $s_{\gamma}^{\star}\big(X(\gamma)\big)$ as
		\begin{align}
			s_{\gamma}^{\star}\big(X(\gamma)\big)&\approx s_{\overline{\alpha}_t}^{\star}\big(X(\overline{\alpha}_t)\big)+\frac{\mathrm{d}s_{\gamma}^{\star}\big(X(\gamma)\big)}{\mathrm{d}\gamma}\left(\gamma-\overline{\alpha}_{t}\right)\\&\approx s_{\overline{\alpha}_{t}}^{\star}(X(\overline{\alpha}_{t}))+\frac{\gamma-\overline{\alpha}_{t}}{\overline{\alpha}_{t}-\overline{\alpha}_{t+1}}\left(s_{\overline{\alpha}_{t}}^{\star}(X(\overline{\alpha}_{t}))-s_{\overline{\alpha}_{t+1}}^{\star}(X(\overline{\alpha}_{t+1}))\right) \nonumber \\
			& \approx s_{t}(X_{t})+\frac{\gamma-\overline{\alpha}_{t}}{\overline{\alpha}_{t}-\overline{\alpha}_{t+1}}\left(s_{t}\big(X_{t}\big)-s_{t+1}\big(X_{t+1}\big)\right),  \label{eq:fine_approx}
		\end{align}
		%
		then \eqref{eq:X-alphat-int} can be approximated by
		\begin{align}
&X(\overline{\alpha}_{t-1})\notag\\
 &~~ \approx\frac{1}{\sqrt{\alpha_{t}}}X(\overline{\alpha}_{t})+\frac{\sqrt{\overline{\alpha}_{t-1}}s_{t}\big(X_{t}\big)}{2}{\displaystyle \int}_{\overline{\alpha}_{t}}^{\overline{\alpha}_{t-1}}\frac{1}{\sqrt{\gamma^{3}}}\mathrm{d}\gamma+\frac{\sqrt{\overline{\alpha}_{t-1}}\Big(s_{t}\big(X_{t}\big)-s_{t+1}\big(X_{t+1}\big)\Big)}{2(\overline{\alpha}_{t}-\overline{\alpha}_{t+1})}{\displaystyle \int}_{\overline{\alpha}_{t}}^{\overline{\alpha}_{t-1}}\frac{\gamma-\overline{\alpha}_{t}}{\sqrt{\overline{\alpha}^{3}}}\mathrm{d}\gamma\notag\\
 &~~ \approx\frac{1}{\sqrt{\alpha_{t}}}\left\{ X(\overline{\alpha}_{t})+\frac{1-\alpha_{t}}{2}s_{t}(X_{t}) +\frac{\left(1-\alpha_{t}\right)^{2}}{4(1-\alpha_{t+1})}\Big(s_{t}\big(X_{t}\big)-\sqrt{\alpha_{t+1}}s_{t+1}\big(X_{t+1}\big)\Big)\right\} \label{eq:ode-imple} , 
\end{align}
		which resembles the proposed sampler \eqref{eqn:ode-extragradient}, and is computationally more appealing since it reuses the previous score function evaluation.

\end{itemize}
It is worth noting that similar approximation as in Scheme 2 has been invoked previously in \citet[Eqn~(3.6)]{lu2022dpm} to construct high-order ODE solvers 
(e.g., the DPM-Solver-2, with 2 indicating second-order ODEs). 
Consequently, the acceleration achieved by our sampler is achieved through ideas akin to the second-order ODE; 
in turn, our convergence guarantees shed light on the effectiveness of high-order ODE solvers like the popular DPM-Solver.

\subsection{Accelerated SDE-based sampler}

Next, we turn to stochastic samplers, and propose a new stochastic sampling procedure that enjoys improved convergence guarantees compared to the DDPM-type sampler \eqref{eq:original-DDPM}. 
To be precise, the proposed sampler begins by drawing $Y_T \sim \mathcal{N}(0, I_d)$ and adopts the following update rule:
\begin{subequations}
\label{eqn:sde-extragradient}
\begin{align} 
	Y_{t}^+ = \Phi_t(Y_t,Z_t),
	\quad  Y_{t-1} = \Psi_t(Y_t^+, Z_t^+) 
\end{align}
for $t= T,\dots,1$, where $Z_t, Z_t^+ \overset{\mathrm{i.i.d.}}{\sim} \mathcal{N}(0,I_d)$, and
\begin{align}
\Phi_t(x,z) &= x + \sqrt{\frac{1 - \alpha_t}{2}} z, \label{eqn:sde-extragradient-Phi} \\
\Psi_{t}(y,z) &=\frac{1}{\sqrt{\alpha_{t}}}\bigg(y + (1-\alpha_{t})s_t(y) + \sqrt{\frac{1 - \alpha_t}{2}} z\bigg).
	\label{eqn:sde-extragradient-Psi}
\end{align}
\end{subequations}
The key difference between the proposed sampler and the original DDPM-type sampler lies in the additional operation $\Phi_t(\cdot, \cdot)$. In this step, a random noise $Z_t$ is injected into the current sample $Y_t$ to obtain an intermediate point $Y_t^+$, which together with another random noise $Z_t^+$ is subsequently fed into $\Psi_t(\cdot,\cdot)$ --- a mapping identical to \eqref{eq:original-DDPM}.

\paragraph{Theoretical guarantees.} 
Let us present the convergence guarantees of the proposed stochastic sampler and their implications, 
followed by some interpretation of the design rationale of the algorithm. 
\begin{theorem}
    \label{thm:main-SDE}
    %
	Suppose that Assumptions~\ref{assump:assumption-data-bounded} and \ref{assumption:score-estimate} hold. 
    Then the proposed stochastic sampler \eqref{eqn:sde-extragradient} with the learning rate schedule \eqref{eqn:alpha-t} 
    achieves
    \begin{align} \label{eq:ratio-SDE}
        \mathsf{TV}&\big(q_1, p_1\big) 
	    \le \sqrt{\frac{1}{2}\mathsf{KL}\big(q_1 \parallel p_1)} \leq C_1\frac{d^3\log^{4.5} T}{T} + C_1\sqrt{d}\varepsilon_{\score}\log^{1.5} T
        \end{align}
	for some universal constant $C_1>0$. 
    %
\end{theorem}
Theorem~\ref{thm:main-SDE} provides non-asymptotic characterizations for the data generation quality of the accelerated stochastic sampler. 
In comparison with the convergence theory for the DDPM-type sampler --- which has a convergence rate proportional to $1/\sqrt{T}$ \citep{chen2022sampling,chen2023improved,li2023towards,benton2023linear} --- 
Theorem~\ref{thm:main-SDE} asserts that the proposed accelerated sampler achieves a faster convergence rate proportional to $1/T$. 
In contrast to Theorem~\ref{thm:main} for the ODE-based sampler, 
the SDE-based sampler does not require continuity of the Jacobian matrix (i.e., Assumption~\ref{assumption:score-estimate-Jacobi}). 
As before, the total-variation distance between $X_1$ and $Y_1$ is proportional to the $\ell_2$ score estimation error when $T$ is sufficiently large, 
which covers a broad range of target data distributions with no requirement on the smoothness or log-concavity of the data distribution.

\paragraph{Interpretation via higher-order approximation.} 
Now we provide some insights into the motivation of the proposed sampler. We start with the characterizations of conditional density $p_{X_{t-1} \mid X_t}$. 
Denoting $\mu_t^{\star}\left(x_t\right):=\frac{1}{\sqrt{\alpha_t}}\left(x_t+\left(1-\alpha_t\right) s_t^{\star}\left(x_t\right)\right)$ and $J_t(x_t)=-(1-\overline{\alpha}_t){J_{s^{*}_t}(x_t)}$, we can approximate $p_{X_{t-1} \mid X_t}$ by
\begin{align}
	&p_{X_{t-1} \mymid X_t}(x_{t-1}\mymid x_t) \approx \exp\Big(-\frac{\alpha_t}{2(1-\alpha_t)}  \cdot\Big\|\Big(I - \frac{1-\alpha_t}{2(\alpha_t-\overline{\alpha}_{t})}J_{t}(x_t)\Big)^{-1}(x_{t-1} - \mu_t^{\star}(x_t))\Big\|^2\Big).\label{sde-approximation}
\end{align}
which is tighter than the one used in analysis of the original SDE-based sampler~\citep{li2023towards} by adopting a higher-order expansion. 
This in turn motivates us to consider the following sequence
\begin{align}
Y_{t-1} &= \frac{1}{\sqrt{\alpha_{t}}}\Big(\underbrace{Y_t +  \sqrt{\frac{1 - \alpha_t}{2}} Z_t}_{\Phi(Y_t,Z_t)} +\sqrt{\frac{1 - \alpha_t}{2}} Z_t^+ +\underbrace{(1-\alpha_{t})s_t^{\star}(Y_t)  - \frac{(1-\alpha_{t})^{3/2}}{\sqrt{2}(\alpha_t-\overline{\alpha}_{t})}J_t(Y_t)Z_t}_{\approx s_t^{*}\big(\Phi(Y_t,Z_t)\big)}\Big)\notag
\end{align}
with $Z_t, Z_t^{+}\stackrel{\mathrm{i.i.d.}}{\sim }\mathcal{N}(0,I_d)$, and $p_{Y_{t-1} \mid Y_t}(x_{t-1}\mid x_t)$ follows 
\begin{align*}
	\mathcal{N}\left(\mu_t^{\star}(x_t), \frac{1-\alpha_t}{\alpha_t}\left(I-\frac{1-\alpha_t}{2\left(1-\bar{\alpha}_t\right)} J_t\left(x_t\right)\right)\left(I-\frac{1-\alpha_t}{2\left(1-\bar{\alpha}_t\right)} J_t\left(x_t\right)\right)^{\top}\right)
\end{align*}
which aligns with \eqref{sde-approximation}. In addition, if we further utilize
$(1-\alpha_t)s_t^{\star}(Y_t) - \frac{(1-\alpha_{t})^{3/2}}{\sqrt{2}(\alpha_t-\overline{\alpha}_{t})}J_t(Y_t)Z_t$ as a first-order approximation of $s_t^{\star}\big(Y_t+\sqrt{\frac{1 - \alpha_t}{2}} Z_t \big)$, we can then arrive at the update rule of the proposed sampler in \eqref{eqn:sde-extragradient}.

\newcommand{\VanillaODE}[1]{\includegraphics[width=0.28\textwidth]{VisualResults/Bedroom/VanillaODE_1000steps/#1}}
\newcommand{\NewODE}[1]{\includegraphics[width=0.28\textwidth]{VisualResults/Bedroom/NewODE_1000stepps/#1}}

\newcommand{\VanillaODECeleba}[1]{\includegraphics[width=0.28\textwidth]{VisualResults/CelebaHQ/#1}}

\newcommand{\VanillaODEChurches}[1]{\includegraphics[width=0.28\textwidth]{VisualResults/Churches/VanillaODE_1000steps/#1}}
\newcommand{\NewODEChurches}[1]{\includegraphics[width=0.28\textwidth]{VisualResults/Churches/NewODE_1000/#1}}

\newcommand{\FIDSampler}[1]{%
  \includegraphics[width=.33\linewidth]{VisualResults/FID/#1}%
}

    \section{Experiments}
    In this section, we illustrate the performance of the proposed accelerated samplers, focusing on highlighting the  relative comparisons with respect to the original DDIM/DDPM ones using the same pre-trained score functions. As an initial step, we focus on reporting result for the deterministic samplers, leaving the stochastic samplers to future work. 
    
    \subsection{Practical implementation} 

In practice, the pre-trained score functions are often available in the form of noise-prediction networks $\epsilon_t(\cdot)$, which are connected via the following relationship in view of \eqref{eq:defn-score-true}:
\begin{align}\label{eq:score_noise}
s_t^\star(X) : =  - \frac{1}{\sqrt{1-\overline{\alpha}_t}}\epsilon_t^{\star}(X),
\end{align}
and $\epsilon_t(\cdot)$ is the estimate of $\epsilon_t^{\star}(\cdot)$. To better align with the empirical practice, it is judicious that the integration in \eqref{eq:X-alphat-int} be approximated in terms of $\epsilon_t^{\star}(X)$, leading to an equivalent rewrite as
\begin{align*}
	X(\overline{\alpha}_{t-1})=
\frac{1}{\sqrt{\alpha_{t}}}X(\overline{\alpha}_{t})-\frac{\sqrt{\overline{\alpha}_{t-1}}}{2}{\displaystyle \int}_{\overline{\alpha}_{t}}^{\overline{\alpha}_{t-1}}\frac{1}{\sqrt{\gamma^{3}} \sqrt{1- \gamma }}\epsilon_{\gamma}^{\star}\big(X(\gamma)\big)\mathrm{d}\gamma. 
\end{align*}
Following similar discussions in Section~\ref{sec:accelerated_ODE}, we discuss its first-order and second-order approximations in discrete time. 
\begin{itemize}
	\item {\em Scheme 1: } If we approximate $\epsilon_{\gamma}^{\star}\big(X(\gamma)\big)$ for $\gamma\in [\overline{\alpha}_{t},\overline{\alpha}_{t-1}]$  by $\epsilon_{\gamma}^{\star}\big(X(\gamma)\big)\approx \epsilon_{\overline{\alpha}_t}^{\star}\big(X(\overline{\alpha}_t)\big) \approx \epsilon_{t}(X_{t})$, then we arrive at
		\begin{align} \label{eq:ddim_real}
		X(\overline{\alpha}_{t-1}) & \approx \frac{1}{\sqrt{\alpha_{t}}}X(\overline{\alpha}_{t})  + \left(\sqrt{1-\overline{\alpha}_{t-1}}-\frac{\sqrt{1-\overline{\alpha}_{t}}}{\sqrt{\alpha}_{t}}\right)\epsilon_{t}(X_{t}),
		\end{align}
		which matches exactly with the DDIM sampler in \citet{song2020denoising}.
		\item {\em Scheme 2: } If we invoke the refined approximation \eqref{eq:fine_approx} in terms of $\epsilon_{\gamma}^{\star}\big(X(\gamma)\big)$, we have
		\begin{align}
X(\overline{\alpha}_{t-1})
	& \approx\frac{1}{\sqrt{\alpha_{t}}}X(\overline{\alpha}_{t})-\frac{\sqrt{\overline{\alpha}_{t-1}}\epsilon_{t}\big(X_{t}\big)}{2}{\displaystyle \int}_{\overline{\alpha}_{t}}^{\overline{\alpha}_{t-1}}\frac{1}{\sqrt{\gamma^{3}(1-\gamma)}}\mathrm{d}\gamma \nonumber \\ 
	& \qquad -
	\frac{\sqrt{\overline{\alpha}_{t-1}}\left(\epsilon_{t}\big(X_{t}\big)-\epsilon_{t+1}\big(X_{t+1}\big)\right)}{2(\overline{\alpha}_{t}-\overline{\alpha}_{t+1})}{\displaystyle \int}_{\overline{\alpha}_{t}}^{\overline{\alpha}_{t-1}}\frac{(\gamma-\overline{\alpha}_{t})}{\sqrt{\gamma^{3}(1-\gamma)}}\mathrm{d}\gamma \nonumber ,
			\end{align}
which after integration becomes:
		\begin{align}
	& X(\overline{\alpha}_{t-1}) \approx \frac{1}{\sqrt{\alpha_{t}}}X(\overline{\alpha}_{t})  + \left(\sqrt{1-\overline{\alpha}_{t-1}}-\frac{\sqrt{1-\overline{\alpha}_{t}}}{\sqrt{\alpha}_{t}}\right)\epsilon_{t}(X_{t}) \notag\\
			& + \left(\frac{\sqrt{\overline{\alpha}_{t-1}}}{\overline{\alpha}_{t} - \overline{\alpha}_{t+1}}\right)\left(\overline{\alpha}_{t}\frac{\sqrt{1-\overline{\alpha}_{t-1}}}{\sqrt{\overline{\alpha}_{t-1}}} + \arcsin\sqrt{\overline{\alpha}_{t-1}}  - \overline{\alpha}_{t}\frac{\sqrt{1-\overline{\alpha}_{t}}}{\sqrt{\overline{\alpha}_{t}}} - \arcsin\sqrt{\overline{\alpha}_{t}}\right)(\epsilon_{t+1}(X_{t+1})-\epsilon_{t}(X_{t})). \label{eq:ode-imple-real}
		\end{align}
This is our new sampler for implementation.



		\end{itemize}

    \subsection{Experimental results}
    We use pre-trained score functions from Huggingface \citep{von-platen-etal-2022-diffusers} for three datasets: CelebA-HQ, LSUN-Bedroom and LSUN-Churches. The same score functions are used in all the samplers. Note that we have not attempted to optimize the speed nor the performance using additional tricks, e.g., employing better score functions, but aim to corroborate our theoretical findings regarding the acceleration of the new samplers without training additional functions when the implementations are otherwise kept the same. 
    

    We first compare the vanilla DDIM-type sampler (cf.~\eqref{eq:ddim_real}) and the accelerated DDIM-type sampler (cf.~\eqref{eq:ode-imple-real}
    ). To begin, Figure~\ref{fig:progCeleba} illustrates the progress of the generated samples over different numbers of function evaluation (NFEs) (between 4 and 50) from the same random seed, using pre-trained scores from the LSUN-Churches dataset. Here, the NFE is the same as the number of diffusion steps since each step takes one score evaluation. 
    
 \begin{figure*}[ht]
    \begin{center}
    \includegraphics[width=1.00\textwidth]{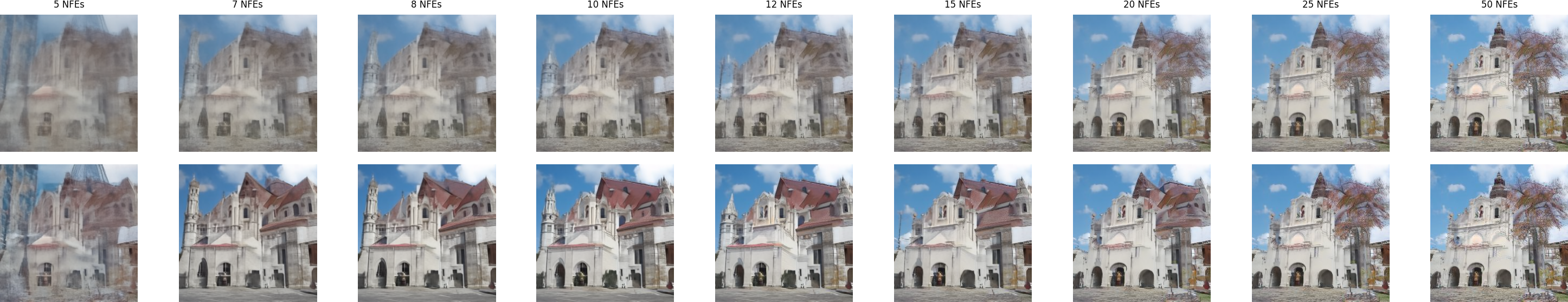} 
    \end{center}
    \caption{The progress of the generated samples over different numbers of NFEs (from 4 to 50), using pre-trained scores from the LSUN-Churches dataset. Top row: the vanilla DDIM-type sampler. Bottom row: the accelerated DDIM-type sampler (ours).}
    \label{fig:progCeleba}
\end{figure*}

To further demonstrate the quality of the sampled images, Figure~\ref{fig:ode_samplers} provides examples of sampled images from the DDIM-type samplers, using pre-trained scores from CelebA-HQ, LSUN-Bedroom and LSUN-Churches datasets, respectively. It can be seen that the sampled images are crisper and less noisy from the accelerated DDIM-type sampler, compared with from the original one, indicating the effectiveness of our method.  
\begin{figure*}[ht]
  \begin{center}
    \begin{minipage}{0.3\textwidth}
      \centering
      \includegraphics[width=\textwidth]{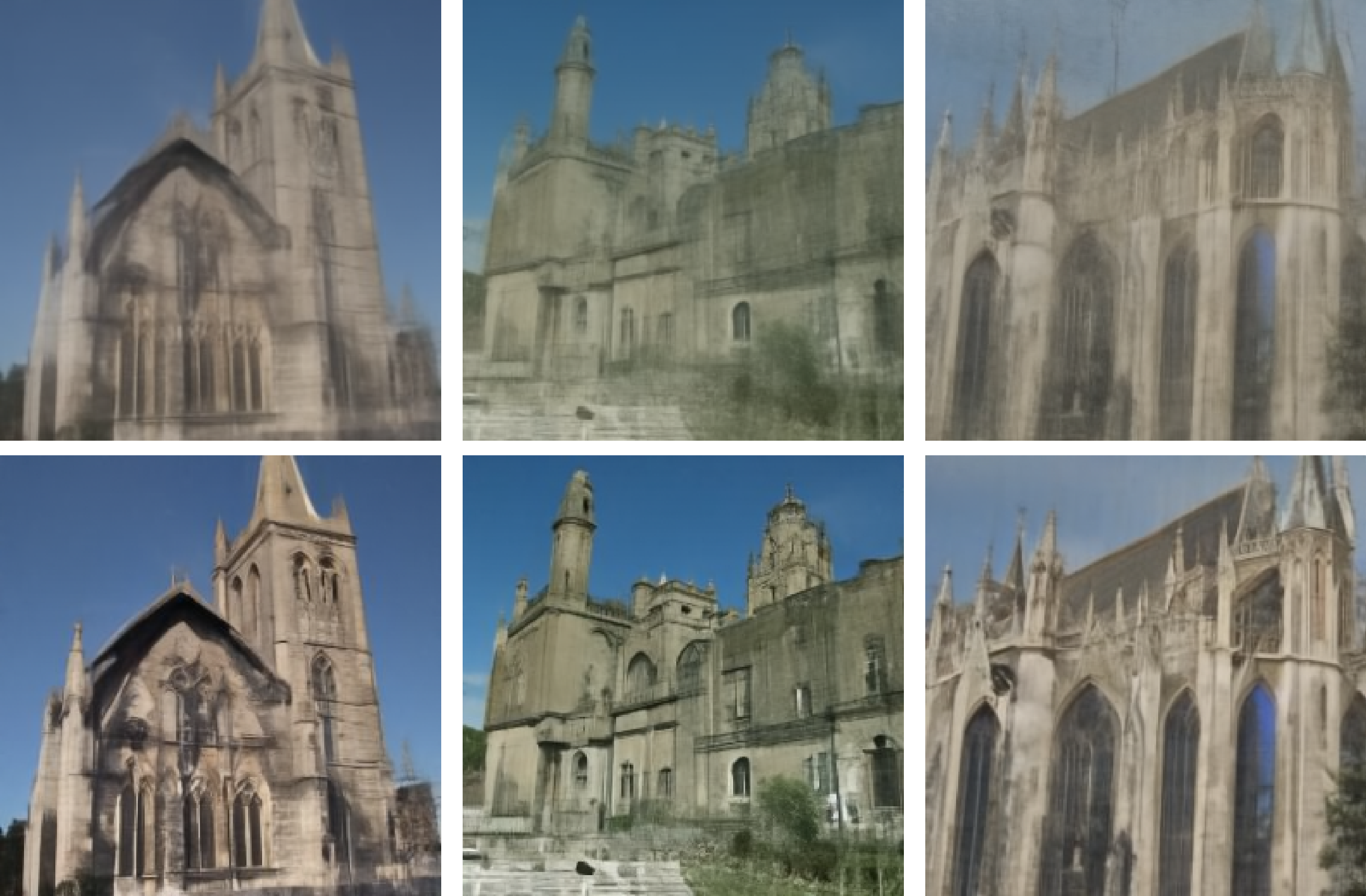} \\
  {\small  (a)  LSUN-Churches }
      \label{fig:compChurches}
    \end{minipage} \hspace{0.1in}
    \begin{minipage}{0.3\textwidth}
      \centering
      \includegraphics[width=\textwidth]{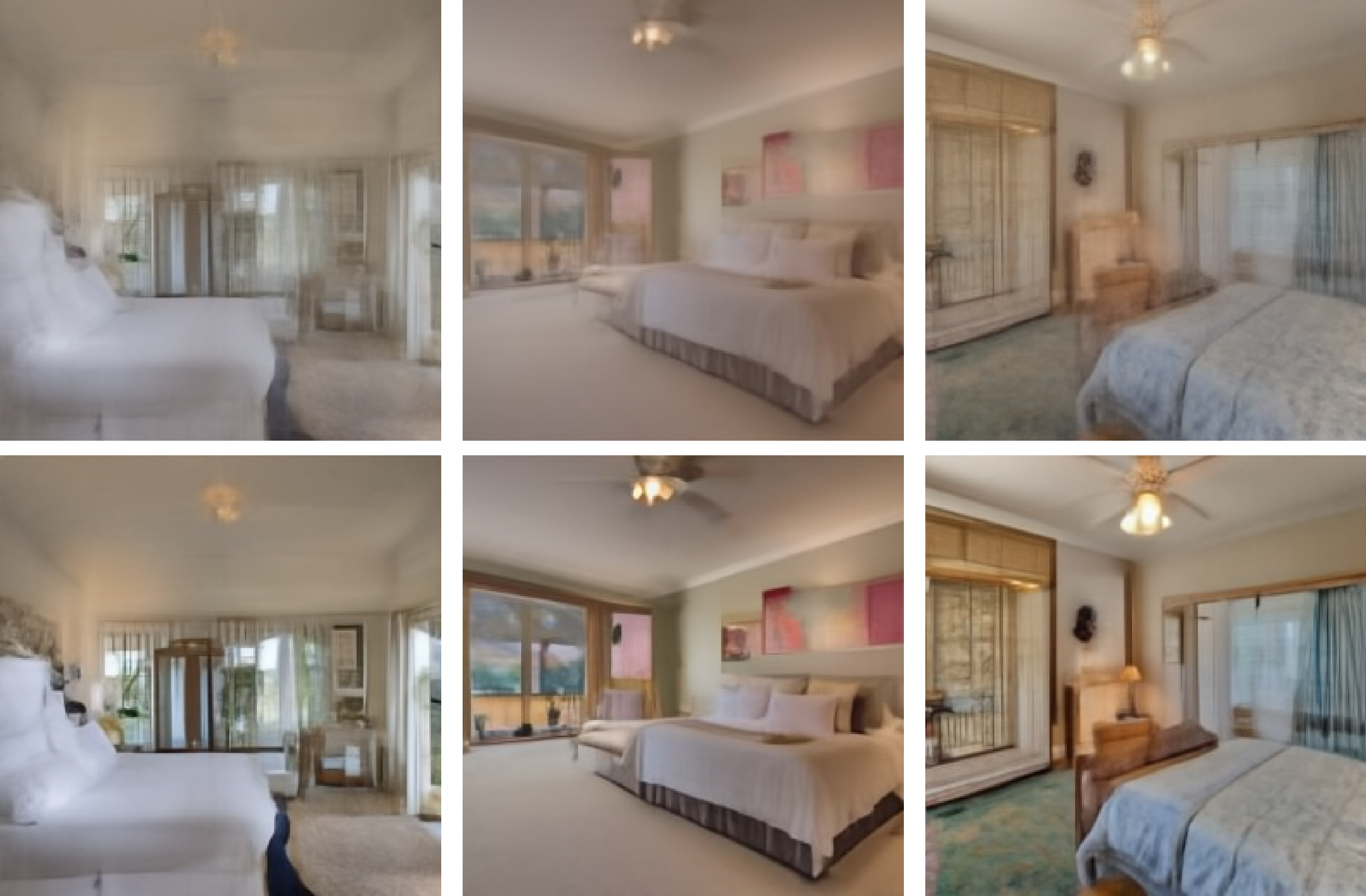}
    {\small  (b) LSUN-Bedroom }
      \label{fig:compBedroom}
    \end{minipage} \hspace{0.1in}
    \begin{minipage}{0.3\textwidth}
      \centering
      \includegraphics[width=\textwidth]{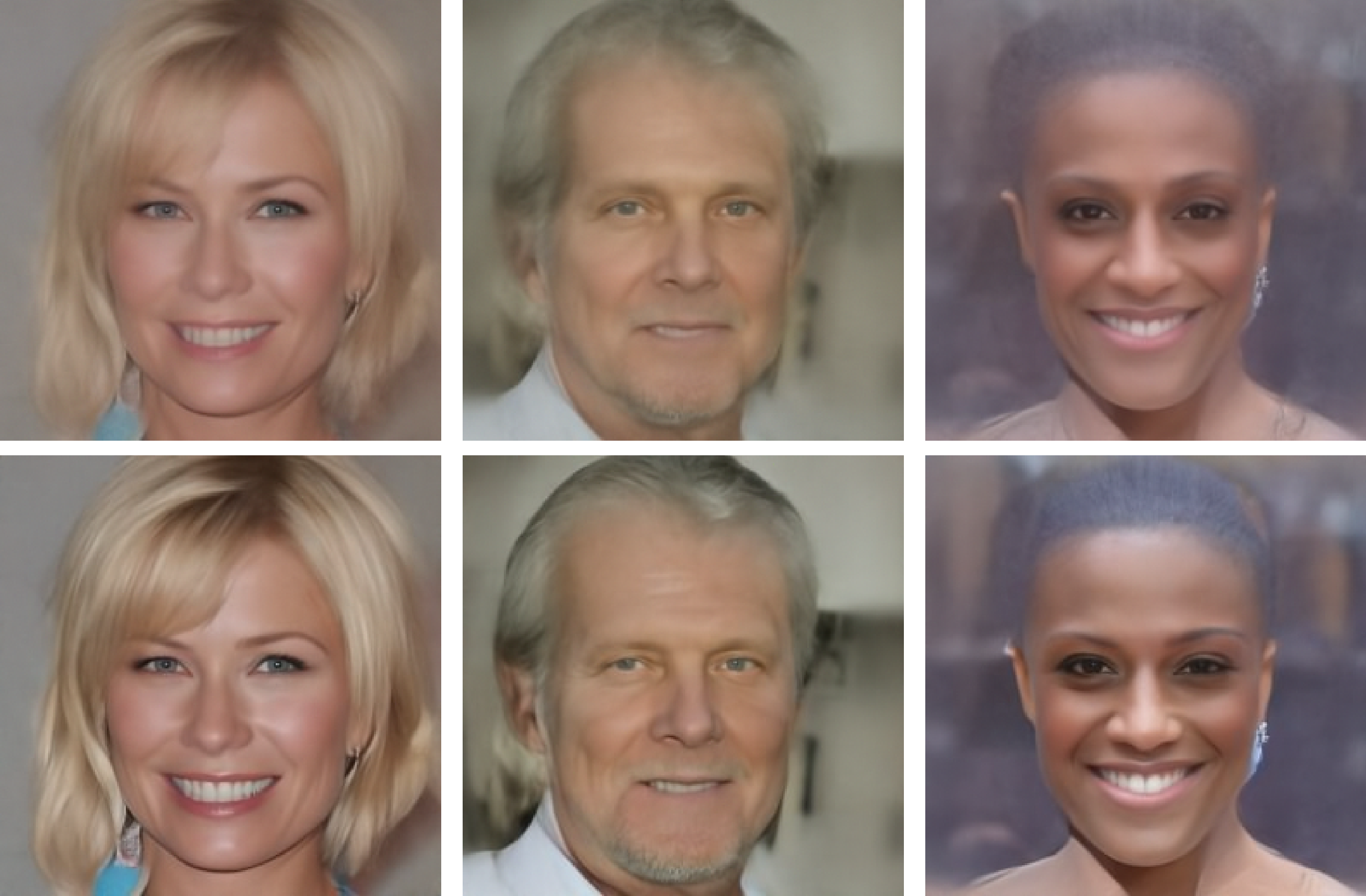} 
    {\small  (c) CelebA-HQ }
      \label{fig:compCeleba} 
    \end{minipage}
  \end{center}
  \caption{Examples of sampled images from the DDIM-type samplers with 5 NFEs, using pre-trained scores from the LSUN-Churches, LSUN-Bedroom, and CelebA-HQ datasets. For each dataset, the top image is the original DDIM-type sampler, and the bottom image is the accelerated DDIM-type sampler (ours). }  \label{fig:ode_samplers}
\end{figure*}

\section{Discussion}
\label{sec:discussion}

In this paper, we have developed novel strategies to achieve provable acceleration in score-based generative modeling. 
The proposed deterministic sampler achieves a convergence rate $1/T^2$ that substantially improves upon prior theory for the probability flow ODE approach, 
whereas the proposed stochastic sampler enjoys a converge rate $1/T$ that also significantly outperforms the convergence theory for the DDPM-type sampler.  
We have demonstrated the stability of these samplers, 
establishing non-asymptotic theoretical guarantees that hold in the presence of $\ell_2$-accurate score estimates. 
Our algorithm development for the deterministic case draws inspiration from higher-order ODE approximations in discrete time, 
which might shed light on understanding popular ODE-based samplers like the DPM-Solver. 
In comparison, the accelerated stochastic sampler is designed based on higher-order expansions of the conditional density.

Our findings further suggest multiple directions that are worthy of future exploration. 
For instance, 
our convergence theory remains sub-optimal in terms of the dependency on the problem dimension $d$, 
which calls for a more refined theory to sharpen dimension dependency. 
Additionally, given the conceptual similarity between our accelerated deterministic sampler and second-order ODE, 
it would be interesting to extend the algorithm and theory using ideas arising from third-order or even higher-order ODE. 
In particular, third-order ODE has been implemented in DPM-Solver-3, which is among the most effective DPM-Solvers in practice. 
Finally, it would be important to design higher-order solvers for SDE-based samplers, in order to unveil the degree of acceleration that can be achieved through high-order SDE.






\section*{Acknowledgements}

Y.~Wei is supported in part by the NSF grants DMS-2147546/2015447, CAREER award DMS-2143215, CCF-2106778, and the Google Research Scholar Award. The work of T. Efimov and Y.~Chi  is supported in part by the grants ONR N00014-19-1-2404, NSF DMS-2134080, ECCS-2126634 and FHWA 693JJ321C000013. Y.~Chen is supported in part by the Alfred P.~Sloan Research Fellowship, the Google Research Scholar Award, the AFOSR grant FA9550-22-1-0198, 
the ONR grant N00014-22-1-2354,  and the NSF grants CCF-2221009 and CCF-1907661.

\bibliographystyle{apalike}
\bibliography{reference-diffusion}

\appendix


\section{Preliminaries}
\label{sec:preliminary-facts}
Before delving into the proof, we make note of a couple of 
preliminary facts, primarily from \citet{li2023towards}.

\subsection{Basic facts}
\paragraph{Score functions.} We first give some characterizations of the score function, 
which follow from \citet[properties (38)]{li2023towards}. 
\begin{lemma}
	The true score function $s_t^{\star}$ is given by the conditional expectation below: 
 \begin{align}
	s_{t}^{\star}(x) & =\mathbb{E}\left[-\frac{1}{\sqrt{1-\overline{\alpha_{t}}}}W\,\bigg|\,\sqrt{\overline{\alpha_{t}}}X_{0}+\sqrt{1-\overline{\alpha}_{t}}W=x\right]=\frac{1}{1-\overline{\alpha}_{t}}\mathbb{E}\left[\sqrt{\overline{\alpha_{t}}}X_{0}-x\,\big|\,\sqrt{\overline{\alpha_{t}}}X_{0}+\sqrt{1-\overline{\alpha}_{t}}W=x\right] \notag\\
		& = - \frac{1}{1-\overline{\alpha}_{t}} \underset{\eqqcolon \, g_t(x) }{\underbrace{ {\displaystyle \int}_{x_{0}}\big(x-\sqrt{\overline{\alpha_{t}}}x_{0}\big)p_{X_{0}\mid X_{t}}(x_{0}\,|\, x)\mathrm{d}x_{0} }}.
		\label{eq:st-MMSE-expression}
	\end{align}
	Also, the Jacobian matrix  
\begin{equation}
J_{t}(x) \coloneqq 
	\frac{\partial g_t(x)}{\partial x} \label{eq:Jacobian-Thm4}
\end{equation}
associated with the function $g_t(\cdot)$ (defined in \eqref{eq:st-MMSE-expression}) satisfies
%
\begin{align}
J_{t}(x) & =I_{d}+\frac{1}{1-\overline{\alpha}_{t}}\bigg\{\mathbb{E}\big[X_{t}-\sqrt{\overline{\alpha}_{t}}X_{0}\mid X_{t}=x\big]\Big(\mathbb{E}\big[X_{t}-\sqrt{\overline{\alpha}_{t}}X_{0}\mid X_{t}=x\big]\Big)^{\top}\notag\\
 & \qquad\qquad\qquad-\mathbb{E}\Big[\big(X_{t}-\sqrt{\overline{\alpha}_{t}}X_{0}\big)\big(X_{t}-\sqrt{\overline{\alpha}_{t}}X_{0}\big)^{\top}\mid X_{t}=x\Big]\bigg\}.
	\label{eq:Jt-x-expression-ij-23}
\end{align}
\end{lemma}

\paragraph{Learning rates.}
The learning rates $\{\alpha_t\}$  as specified in \eqref{eqn:alpha-t} enjoy several properties that will be used multiple times in the analysis. 
We record several of these properties below, which have been proven in \citet[properties (39)]{li2023towards}.
\begin{lemma}
	\begin{subequations}
		\label{eqn:properties-alpha-proof}
		The learning rates specified in \eqref{eqn:alpha-t} obey
		\begin{align}
			\alpha_t &\geq1-\frac{c_{1}\log T}{T}  \ge \frac{1}{2},\qquad\qquad\quad~~ 1\leq t\leq T \label{eqn:properties-alpha-proof-00}\\
			\frac{1}{2}\frac{1-\alpha_{t}}{1-\overline{\alpha}_{t}} \leq \frac{1}{2}\frac{1-\alpha_{t}}{\alpha_t-\overline{\alpha}_{t}}
			&\leq \frac{1-\alpha_{t}}{1-\overline{\alpha}_{t-1}}  \le \frac{4c_1\log T}{T},\qquad\quad~~ 2\leq t\leq T  \label{eqn:properties-alpha-proof-1}\\
			1&\leq\frac{1-\overline{\alpha}_{t}}{1-\overline{\alpha}_{t-1}} \leq1+\frac{4c_{1}\log T}{T} ,\qquad2\leq t\leq T  \label{eqn:properties-alpha-proof-3} \\
			\overline{\alpha}_{T} & \le \frac{1}{T^{c_2}}, \label{eqn:properties-alpha-proof-alphaT} 
		\end{align}
		provided that $T$ is large enough. 
		Here, $c_1$ is defined in \eqref{eqn:alpha-t}, and $c_2\geq 1000$ is some large numerical constant. 
		In addition, if $\frac{d(1-\alpha_{t})}{\alpha_{t}-\overline{\alpha}_{t}}\lesssim 1$, then one has
		\begin{align}
		\Big(\frac{1-\overline{\alpha}_{t}}{\alpha_{t}-\overline{\alpha}_{t}}\Big)^{d/2} 
			& =1+\frac{d(1-\alpha_{t})}{2(\alpha_{t}-\overline{\alpha}_{t})}+\frac{d(d-2)(1-\alpha_{t})^{2}}{8(\alpha_{t}-\overline{\alpha}_{t})^{2}}+O\bigg(d^{3}\Big(\frac{1-\alpha_{t}}{\alpha_{t}-\overline{\alpha}_{t}}\Big)^{3}\bigg).
			\label{eq:expansion-ratio-1-alpha}	
		\end{align}
		\end{subequations}
\end{lemma}
\paragraph{The forward process.} 
Next, we gather several conditional tail bounds for the random vector $X_0$ of the forward process, which have been established in \citet[Lemmas~1 and 2]{li2023towards}.


\begin{lemma} \label{lem:x0}
	Suppose that there exists some numerical constant $c_R>0$ obeying 
	\begin{equation}
		\mathbb{P}\big( \|X_0 \|_2 \leq R\big) = 1
		\qquad \text{and} \qquad 
		R = T^{c_R}.  
		\label{eq:cR-defn-lem}
	\end{equation}
	Consider any $y \in \real$, and let
	\begin{align}
	\label{eqn:choice-y-prelim}
		\theta(y) \coloneqq \max\bigg\{ \frac{-\log {p_{X_t}(y)}}{d\log T} , c_6 \bigg\}
	\end{align}
	for some large enough constant $c_6\geq 2c_R+c_0$. Then for any quantity $c_5 \ge 2$, 
	conditioned on $X_t=y$ one has
	\begin{align}
		\big\|\sqrt{\overline{\alpha}_{t}}X_0 - y \big\|_2 \leq 5c_5\sqrt{\theta(y) d(1-\overline{\alpha}_{t})\log T} 
		\label{eq:P-xt-X0-124}
	\end{align} 
	with probability at least $1 - \exp\big(-c_5^2\theta(y) d\log T \big)$.
	In addition, it holds that
	\begin{subequations}
	\begin{align}
		\mathbb{E}\left[\big\| \sqrt{\overline{\alpha}_{t}}X_{0} - y \big\|_{2}\,\big|\,X_{t}=y\right] &\leq 12\sqrt{\theta(y) d(1-\overline{\alpha}_{t})\log T},\label{eq:E-xt-X0} \\
		\mathbb{E}\left[\big\| \sqrt{\overline{\alpha}_{t}}X_{0} - y \big\|^2_{2}\,\big|\,X_{t}=y\right] &\leq 120\theta(y) d(1-\overline{\alpha}_{t})\log T,\label{eq:E2-xt-X0} \\
		\mathbb{E}\left[\big\| \sqrt{\overline{\alpha}_{t}}X_{0} - y \big\|^3_{2}\,\big|\,X_{t}=y\right] &\leq 1040\big(\theta(y) d(1-\overline{\alpha}_{t})\log T\big)^{3/2},\label{eq:E3-xt-X0}\\
		\mathbb{E}\left[\big\| \sqrt{\overline{\alpha}_{t}}X_{0} - y \big\|^4_{2}\,\big|\,X_{t}=y\right] &\leq 10080\big(\theta(y) d(1-\overline{\alpha}_{t})\log T\big)^{2}.\label{eq:E4-xt-X0}
	\end{align}
	\end{subequations}
	\end{lemma}

\begin{lemma} \label{lem:y-t}
	For some $\theta > c_6$,
	assume that $\|y - x\|_2 \lesssim (1 - \alpha_{t+1})\sqrt{\frac{\theta d\log T}{1-\overline{\alpha}_{t}}}$ and $\log {p_{X_t}(x)} \ge -\theta d\log T$. 
	Then we have
	\begin{align}
	&p_{X_0 \mymid X_{t+1}}(x_0 \mymid \sqrt{\alpha_{t+1}}y)= \notag
	p_{X_0 \mymid X_t}(x_0 \mymid x)
	\bigg\{1+\frac{(1 - \alpha_{t+1})\big\|x - \sqrt{\overline{\alpha}_{t}}x_0\big\|_2^2}{2(1-\overline{\alpha}_{t})^2} 
	-\frac{(x - \sqrt{\overline{\alpha}_{t}}x_0)^{\top}(y - x)}{1-\overline{\alpha}_{t}}- \notag\\
	& \int_{x_0} \bigg(\frac{(1 - \alpha_{t+1})\big\|x - \sqrt{\overline{\alpha}_{t}}x_0\big\|_2^2}{2(1-\overline{\alpha}_{t})^2} - \frac{(x - \sqrt{\overline{\alpha}_{t}}x_0)^{\top}(y - x)}{1-\overline{\alpha}_{t}}\bigg)p_{X_0 \mymid X_t}(x_0 \mymid x)  \mathrm{d} x_0 + O\Big(\theta d\Big(\frac{1-\alpha_{t+1}}{1-\overline{\alpha}_{t}}\Big)^{3/2}\log T\Big)\bigg\}. 
	\end{align}
	\end{lemma}
	\begin{proof} See Section~\ref{sec:proof-lem:y-t}. \end{proof}

\paragraph{Proximity of $p_{X_T}$ and $q_{Y_T}$.} 
When the number of steps $T$ is sufficiently large, the distribution of $p_{X_T}$ and that of $p_{Y_T}$ become exceedingly close, as asserted by the following lemma 
 (see \citet[Lemma 3]{li2023towards}).
\begin{lemma}
	\label{lem:KL-T}
	For any large enough $T$, one has
	\begin{align}
		\mathsf{TV}(p_{X_{T}}\parallel p_{Y_{T}})^2 \leq \frac{1}{2}\mathsf{KL}(p_{X_{T}}\parallel p_{Y_{T}}) \lesssim \frac{1}{T^{200}}. 
	\end{align}
\end{lemma}

\paragraph{Score estimation errors.}  
Consider any vector $x\in \mathbb{R}^d$. For any $1 < t\leq T$,  define
\begin{align}
	\varepsilon_{\score, t}(x) \coloneqq \big\|s_t(x) - s_t^{\star}(x) \big\|_2
\qquad\text{and}\qquad
	\varepsilon_{\Jacobi, t}(x) \coloneqq \big\| J_{s_t}(x)  -  J_{s_t^{\star}} (x) \big\|,
	\label{eq:pointwise-epsilon-score-J}
\end{align} 
where we use $J_{s_t}$ and $J_{s_t^{\star}}$ to represent the Jacobian matrices of $s_t(\cdot)$ and $s_t^{\star}(\cdot)$, respectively.
We also have, under Assumptions~\ref{assumption:score-estimate} and \ref{assumption:score-estimate-Jacobi}, 
 that
\begin{subequations}
	\label{eq:score-assumptions-equiv}
\begin{align}
	\frac{1}{T}\sum_{t=1}^{T} \mathop{\mathbb{E}}\limits_{X\sim q_{t}}\big[\varepsilon_{\score,t}(X)\big] & \leq\bigg(\frac{1}{T}\sum_{t=1}^{T}\mathop{\mathbb{E}}\limits_{X\sim q_{t}}\left[\varepsilon_{\score,t}(X)^{2}\right]\bigg)^{1/2}\leq\varepsilon_{\score}, \\
%
%
	\frac{1}{T}\sum_{t=1}^{T}\mathop{\mathbb{E}}\limits_{X\sim q_{t}}\big[\varepsilon_{\Jacobi,t}(X)\big] & \leq\varepsilon_{\Jacobi}.
\end{align}
\end{subequations}

%
\subsection{Proof of Lemma~\ref{lem:y-t}}
\label{sec:proof-lem:y-t}
To establish this lemma, we observe that
	\begin{align*}
	&p_{X_0 \mymid X_{t+1}}(x_0 \mymid \sqrt{\alpha_{t+1}}y) 
	= \frac{p_{X_0}(x_0)p_{X_{t+1} \mymid X_0}(y \mymid x_0)}{\int_{x_0} p_{X_0}(x_0)p_{X_{t+1} \mymid X_0}(y \mymid x_0) \mathrm{d} x_0} \\
	&= \frac{p_{X_0}(x_0)\exp\Big(-\frac{\|y - \sqrt{\overline{\alpha}_{t}}x_0\|_2^2}{2(\alpha_{t+1}^{-1}-\overline{\alpha}_{t})}\Big)}{\int_{x_0} p_{X_0}(x_0)\exp\Big(-\frac{\|y - \sqrt{\overline{\alpha}_{t}}x_0\|_2^2}{2(\alpha_{t+1}^{-1}-\overline{\alpha}_{t})}\Big) \mathrm{d} x_0} \\
	&= \frac{p_{X_0}(x_0)\exp\Big(-\frac{\|x - \sqrt{\overline{\alpha}_{t}}x_0\|_2^2}{2(1-\overline{\alpha}_{t})}\Big)\Big(1-\frac{(1 - \alpha_{t+1}^{-1})\|x - \sqrt{\overline{\alpha}_{t}}x_0\|_2^2}{2(1-\overline{\alpha}_{t})^2} - \frac{(x - \sqrt{\overline{\alpha}_{t}}x_0)^{\top}(y - x)}{1-\overline{\alpha}_{t}} + O\Big(\theta d\Big(\frac{1-\alpha_{t+1}}{1-\overline{\alpha}_{t}}\Big)^{3/2}\log T\Big)\Big)}
	{\int_{x_0} p_{X_0}(x_0)\exp\Big(-\frac{\|x - \sqrt{\overline{\alpha}_{t}}x_0\|_2^2}{2(1-\overline{\alpha}_{t})}\Big)\Big(1-\frac{(1 - \alpha_{t+1}^{-1})\|x - \sqrt{\overline{\alpha}_{t}}x_0\|_2^2}{2(1-\overline{\alpha}_{t})^2} - \frac{(x - \sqrt{\overline{\alpha}_{t}}x_0)^{\top}(y - x)}{1-\overline{\alpha}_{t}} + O\Big(\theta d\Big(\frac{1-\alpha_{t+1}}{1-\overline{\alpha}_{t}}\Big)^{3/2}\log T\Big)\Big) \mathrm{d} x_0} \\
	&= p_{X_0 \mymid X_t}(x_0 \mymid x)
	\bigg\{1+\frac{(1 - \alpha_{t+1})\big\|x - \sqrt{\overline{\alpha}_{t}}x_0\big\|_2^2}{2(1-\overline{\alpha}_{t})^2} 
	-\frac{(x - \sqrt{\overline{\alpha}_{t}}x_0)^{\top}(y - x)}{1-\overline{\alpha}_{t}} \\
	&\qquad- \int_{x_0} \bigg(\frac{(1 - \alpha_{t+1})\big\|x - \sqrt{\overline{\alpha}_{t}}x_0\big\|_2^2}{2(1-\overline{\alpha}_{t})^2} - \frac{(x - \sqrt{\overline{\alpha}_{t}}x_0)^{\top}(y - x)}{1-\overline{\alpha}_{t}}\bigg)p_{X_0 \mymid X_t}(x_0 \mymid x)  \mathrm{d} x_0 \\
	&\qquad+ O\Big(\theta d\Big(\frac{1-\alpha_{t+1}}{1-\overline{\alpha}_{t}}\Big)^{3/2}\log T\Big)\bigg\}, 
	\end{align*}
	which follows from the following property: 
	\begin{align*}
	&\frac{\|y - \sqrt{\overline{\alpha}_{t}}x_0\|_2^2}{2(\alpha_{t+1}^{-1}-\overline{\alpha}_{t})} 
	- \frac{\|x - \sqrt{\overline{\alpha}_{t}}x_0\|_2^2}{2(1-\overline{\alpha}_{t})} \\
	&~~= \frac{(1 - \alpha_{t+1}^{-1})\|y - \sqrt{\overline{\alpha}_{t}}x_0\|_2^2}{2(1-\overline{\alpha}_{t})(\alpha_{t+1}^{-1}-\overline{\alpha}_{t})} 
	+ \frac{\|y - \sqrt{\overline{\alpha}_{t}}x_0\|_2^2 - \|x - \sqrt{\overline{\alpha}_{t}}x_0\|_2^2}{2(1-\overline{\alpha}_{t})} \\
	&~~= -\frac{(1 - \alpha_{t+1})\|x - \sqrt{\overline{\alpha}_{t}}x_0\|_2^2}{2(1-\overline{\alpha}_{t})^2} 
	+ \frac{2(x - \sqrt{\overline{\alpha}_{t}}x_0)^{\top}(y - x)}{2(1-\overline{\alpha}_{t})} + O\Big(\theta d\Big(\frac{1-\alpha_{t+1}}{1-\overline{\alpha}_{t}}\Big)^{3/2}\log T\Big).
	\end{align*}


\section{Analysis for the accelerated ODE sampler (proof of \Cref{thm:main})}	
In this section, we present our non-asymptotic analysis for the accelerated ODE sampler.   Considering the total variation distance is always upper bounded by \(1\), we can reasonably assume the following conditions throughout the proof, which are necessary for the claimed result \cref{eq:ratio-ODE} to be non-trivial.
\begin{subequations}
	\label{eq:assumption-T-score-Jacob}
	\begin{align}
		T & \geq \sqrt{C_1} d^3 \log ^3 T, \\
\varepsilon_{\text {score }} & \leq \frac{1}{C_1 \sqrt{d} \log^2  T},  \\
\varepsilon_{\text {Jacobi }} & \leq \frac{1}{C_1 d \log^2  T}. 
	\end{align}
\end{subequations}

\subsection{Main steps of the proof}



\paragraph{Preparation.}
To begin with, let us introduce the following functions that help ease presentation:  
\begin{subequations}
	\label{defn:phit-x}
\begin{align}
	\phi^{\star}_t(x) & \coloneqq x-\frac{1-\alpha_{t+1}}{2}s^{\star}_{t}(x),\\
	\phi_t(x) &\coloneqq x-\frac{1-\alpha_{t+1}}{2}s_{t}(x),\\
	\psi^{\star}_t(x) &\coloneqq x+\frac{1-\alpha_{t}}{2}s^{\star}_{t}(x) + \frac{(1-\alpha_{t})^2}{4(1-\alpha_{t+1})}\Big(s^{\star}_{t}(x) - \sqrt{\alpha_{t+1}}s^{\star}_{t+1}\big(\phi^{\star}_t(x)\big)\Big),\\
	\psi_t(x) &\coloneqq x+\frac{1-\alpha_{t}}{2}s_{t}(x) + \frac{(1-\alpha_{t})^2}{4(1-\alpha_{t+1})}\Big(s_{t}(x) - \sqrt{\alpha_{t+1}}s_{t+1}\big(\phi_t(x)\big)\Big).
	\end{align}
	Armed with these functions, one can equivalently rewrite our update rule \eqref{eqn:ode-extragradient} as follows:
	\begin{align}
		Y_{t-1}=\Psi_t\big(Y_t, \Phi_t(Y_t)\big)=\frac{1}{\sqrt{\alpha_t}} \psi_t\left(Y_t\right). \label{eq:Yt-psi}
	\end{align}
\end{subequations}
Additionally, for any point $y_T \in \mathbb{R}^d$, 
we introduce the corresponding sequence
\begin{equation}
y_{t-1}=\frac{1}{\sqrt{\alpha_t}} \psi_t\left(y_t\right), \quad y_{t}^{-}=\sqrt{\alpha_{t+1}}\phi_t(y_t)
, \quad t=T, T-1, \cdots
	\label{eq:yt-sequence-deterministic}
\end{equation}
%
%
Furthermore, 
it is also useful to single out the following error-related quantities  
for any point $y_T \in \mathbb{R}^d$ and its associated sequence $\{y_t\}_{t=1}^{T-1}$ and $\{y_t^-\}_{t=2}^{T}$:
\begin{subequations}
	\label{eq:defn-xik-Stk-proof}
\begin{align}
	\xi_t(y_T) &\coloneqq \frac{\log T}{T}\left\{d\big(\varepsilon_{\Jacobi, t}(y_t) + \varepsilon_{\Jacobi, t+1}(y_t^-)\big) 
	+ \sqrt{d\log T}\big(\varepsilon_{\score, t}(y_t) + \varepsilon_{\score, t+1}(y_t^-) \big)\right\}; \\ 
	S_{t}(y_T) &\coloneqq \sum_{1 < k \le t} \xi_k(y_k), \quad \text{for }t\geq 2,
	\qquad \text{ and } \qquad
	S_{1}(y_T) = 0,
\end{align}
\end{subequations}
where we recall the definitions of $\varepsilon_{\Jacobi, t}(\cdot)$ and $\varepsilon_{\score, t}$ in \eqref{eq:pointwise-epsilon-score-J}. 
To understand these quantities, note that if we start from a point $y_T$, 
then 
$\xi_t(y_t)$ reflects a certain weighted score estimation error in the $t$-th step, 
while $S_{t}(y_T)$ aggregates these weighted score estimation errors from the very beginning to the  $t$-th iteration.

In addition, there are several objects that play crucial 
roles in the subsequent analysis, which we single out as follows; here and throughout, we suppress their dependency on $x$ to streamline presentation.  
\begin{subequations}
	\label{eqn:ODE-constant}
	\begin{align}
	A_{t} & \defn\frac{1}{1-\overline{\alpha}_{t}}\int p_{X_{0}\mymid X_{t}}(x_{0}\mymid x)\big\| x-\sqrt{\overline{\alpha}_{t}}x_{0}\big\|_{2}^{2}\mathrm{d}x_{0};\\
	B_{t} & \defn\frac{1}{1-\overline{\alpha}_{t}}\bigg\|\int p_{X_{0}\mymid X_{t}}(x_{0}\mymid x)\big(x-\sqrt{\overline{\alpha}_{t}}x_{0}\big)\mathrm{d}x_{0}\bigg\|_{2}^{2};\\
	C_{t} & \defn\frac{1}{(1-\overline{\alpha}_{t})^{2}}\int p_{X_{0}\mymid X_{t}}(x_{0}\mymid x)\big\| x-\sqrt{\overline{\alpha}_{t}}x_{0}\big\|_{2}^{4}\mathrm{d}x_{0};\\
	D_{t} & \defn\frac{1}{(1-\overline{\alpha}_{t})^{2}}\int p_{X_{0}\mymid X_{t}}(x_{0}\mymid x)\Big(\big\langle g_{t}(x),\,x-\sqrt{\overline{\alpha}_{t}}x_{0}\big\rangle\Big)^{2}\mathrm{d}x_{0};\\
	E_{t} & \defn\frac{1}{(1-\overline{\alpha}_{t})^{2}}\int p_{X_{0}\mymid X_{t}}(x_{0}\mymid x)\big\| x-\sqrt{\overline{\alpha}_{t}}x_{0}\big\|_{2}^{2}\big\langle g_{t}(x),\,x-\sqrt{\overline{\alpha}_{t}}x_{0}\big\rangle\mathrm{d}x_{0}.
	\end{align}
	\end{subequations}
With the above preparation in place, we can now readily proceed to our proof.

\paragraph{Step 1: bounding the density ratios.}
To begin with, we make the observation that
\begin{align}
\frac{p_{Y_{t-1}}\left(y_{t-1}\right)}{p_{X_{t-1}}\left(y_{t-1}\right)} & =\frac{p_{\sqrt{\alpha_t} Y_{t-1}}\left(\sqrt{\alpha_t} y_{t-1}\right)}{p_{\sqrt{\alpha_t} X_{t-1}}\left(\sqrt{\alpha_t} y_{t-1}\right)}\notag \\
& =\frac{p_{\psi_t(Y_{t})}\left(\psi_t(y_{t})\right)}{p_{Y_t}\left(y_t\right)} \cdot\left(\frac{p_{\sqrt{\alpha_t} X_{t-1}}\left(\psi_t(y_{t})\right)}{p_{X_t}\left(y_t\right)}\right)^{-1} \cdot \frac{p_{Y_t}\left(y_t\right)}{p_{X_t}\left(y_t\right)},\label{eq:recursion}
\end{align}
which is an elementary identity that allows one to link 
the target density ratio $\frac{p_{Y_{t-1}}}{p_{X_{t-1}}}$ at the $(t-1)$-th step with the density ratio $\frac{p_{Y_{t}}}{p_{X_{t}}}$ at the $t$-th step. 
This relation reveals the importance of bounding 
$\frac{p_{\psi_t(Y_{t})}\left(\psi_t(y_{t})\right)}{p_{Y_t}\left(y_t\right)}$ and 
$\frac{p_{\sqrt{\alpha_t} X_{t-1}}\left(\psi_t(y_{t})\right)}{p_{X_t}\left(y_t\right)}$, 
towards which we resort to the following lemma.


\begin{lemma} 
	\label{lem:main-ODE}
	For any $x \in \real^d$, suppose that 
	\begin{align}
	\label{eqn:choice-y}
		-\frac{\log {p_{X_t}(x)}}{d\log T} \le c_6
	\end{align}
	for some large enough constant $c_6 \geq 2c_R+c_0$, and that 
	$
		\frac{40 c_1\varepsilon_{\score,t}(x)\log^{\frac{3}{2}}T}{T}\leq\sqrt{d}
	$.
	Then one has 
	\begin{align} 
		\frac{p_{X_{t+1}/\sqrt{\alpha_{t+1}}}\big(\phi_{t}(x)\big)}{p_{X_{t}}(x)}&\geq\exp\bigg(-\Big(\varepsilon_{\score, t}(x)\sqrt{d\log T}+ d\log T\Big)\frac{c_7\log T}{T}\bigg) \label{eq:xt_lb}
	\end{align}
	for some universal constant $c_7>0$. 
	If, in addition, we have 
	$
	C_{10}\frac{d\log^{2}T+(\varepsilon_{\score, t}(x) + \varepsilon_{\score, t+1}(\Phi_t(x)))\sqrt{d\log^{3}T}}{T}\leq1
	$
	for some large enough constant $C_{10}>0$, then it holds that 
	\begin{subequations}
	\label{eq:ODE}
	\begin{align} 
	&\frac{p_{\sqrt{\alpha_{t}}X_{t-1}}(\psi_{t}(x))}{p_{X_{t}}(x)} =1 + \frac{(1-\alpha_{t})(d+B_t-A_t)}{2(1-\overline{\alpha}_{t})} \notag\\
	&\qquad+ \frac{(1-\alpha_{t})^2}{8(1-\overline{\alpha}_{t})^2}\big[d(d+2) + (4+2d)(B_t-A_t) - B_t^2 + C_t + 2D_t - 3E_t + A_tB_t\big]
	 \notag\\
	 & \qquad
	+O\bigg(\frac{d^{3}\log^{6}T}{T^3} + \frac{\sqrt{d\log^3 T}}{T}\Big(\varepsilon_{\score, t}(x) + \varepsilon_{\score, t+1}\big(\Phi_t(x)\big)\Big)\bigg).
		\label{eq:xt}
	\end{align}

	Moreover, for any random vector $Y$, one has 
	\begin{align} 
	 & \frac{p_{\psi_{t}(Y)}(\psi_{t}(x))}{p_{Y}(x)} =1 + \frac{(1-\alpha_{t})(d+B_t-A_t)}{2(1-\overline{\alpha}_{t})} \notag\\
	&\qquad+ \frac{(1-\alpha_{t})^2}{8(1-\overline{\alpha}_{t})^2}\big[d(d+2) + (4+2d)(B_t-A_t) - B_t^2 + C_t + 2D_t - 3E_t + A_tB_t\big]
	 \notag\\
	 & \qquad
	+O\bigg(\frac{d^{6}\log^{6}T}{T^3} + \frac{\sqrt{d\log^3 T}}{T}\varepsilon_{\score, t}(x) + \frac{d\log T}{T}\Big(\varepsilon_{\Jacobi, t}(x) + \varepsilon_{\Jacobi, t+1}\big(\Phi_t(x)\big)\Big)\bigg),
	\label{eq:yt}
	\end{align}
	\end{subequations}
	provided that 
	$
		C_{11}\frac{d^{2}\log^{2}T+\varepsilon_{\score, t}(x)\sqrt{d\log^3 T} + d(\varepsilon_{\Jacobi, t}(x) + \varepsilon_{\Jacobi, t+1}(\Phi_t(x)))\log T}{T}  \leq 1
	$
	holds for some large enough constant $C_{11}>0$. 
	
	Additionally, if $\frac{d^{2}\log^{2}T+\sqrt{d\log T}\varepsilon_{\score, t}(x)+d\varepsilon_{\Jacobi,t}(x)\log T}{T} \lesssim 1$, 
	then we have
	\begin{align}
	\frac{p_{\Phi_{t}(X_{t})}(\Phi_{t}(x))}{p_{X_{t+1}}(\Phi_{t}(x))} = 1+ O\bigg(\frac{d^2\log^4T}{T^2} + \frac{d^6\log^6T}{T^3} + \frac{\sqrt{d\log T}\varepsilon_{\score, t}(x)+d\varepsilon_{\Jacobi,t}(x)\log T}{T}\bigg). \label{eq:phit}
	\end{align}
	\end{lemma}
The proof of this lemma is postponed to Section~\ref{sec:proof-lemma-main-ODE}. 
Notheworthily, the main terms in \eqref{eq:xt} and \eqref{eq:yt} coincide, a crucial fact that allows one to focus on the lower-order term later on. 
Moreover, the relation \eqref{eq:phit} captures the effect of performing one iteration of the probability flow ODE sampler (as captured by the mapping $\Phi_t(\cdot)$).

\paragraph{Step 2: decomposing the TV distance of interest.}
We now move on to look at the TV distance of interest. Akin to \citet{li2023towards}, 
we first single out the following set: 
\begin{align}
	\mathcal{E} &\coloneqq \Big\{y : q_{1}(y) > \max\big\{ p_{1}(y),\, \exp\big(- c_{6} d\log T \big) \big\} \Big\},
\end{align}
where $c_{6}>0$ is some large enough numerical constant introduced in \Cref{lem:main-ODE}. 
The points in $\mathcal{E}$ satisfy two properties: (i) $q_1(y)>p_1(y)$, and (ii) $q_1(y)$ is not too small, so that $y$ falls within a more normal range (w.r.t.~$p_{X_1}(\cdot)$).

Following similar calculations as in \citet[Step 2 of Section 5.2]{li2023towards} and invoking the properties of the forward process in \Cref{lem:x0}, we can demonstrate that
\begin{align}
\mathsf{TV}\big(q_{1},p_{1}\big) 
	&\le \mathop{\mathbb{E}}\limits_{Y_{1}\sim p_{1}}\bigg[\Big(\frac{q_{1}(Y_{1})}{p_{1}(Y_{1})}-1\Big)\ind\left\{ Y_{1}\in\mathcal{E}\right\} \bigg] + \exp\big(-c_{6}d\log T\big),
	\label{eqn:ode-tv-10}
\end{align}
and hence it suffices to focus attention on what happens on the set $\mathcal{E}$. 
To proceed, for any point $y_T$, we define
\begin{equation}
	\tau(y_T)
	\coloneqq
	\max\Big\{2\le t\le T+1:S_{t-1}\big(y_{T}\big)\leq c_{14}\text{ and }-\log q_k(y_k)\leq c_{\tau} d\log T,\text{ for all }k < t\Big\}\label{eq:defn-tao-i}
\end{equation}
for some universal constant $c_{\tau}>0$. We shall often abbreviate $\tau(y_T)$ as $\tau$ as long as it is clear from the context. 
Taking $\{y_t\}_{t=1}^{T-1}$ to be the associated sequence of our deterministic sampler initialized at $y_T$ (cf.~\eqref{eq:yt-sequence-deterministic}), 
we can further define 
\begin{subequations}
	\label{eq:defn-I2-I3-I4-ode}
\begin{align}
\mathcal{I}_0 &\coloneqq \Big\{y_T : \tau(y_T) = T+1\Big\}, \label{eq:defn-I2-I3-I4-ode-I0} \\
	\mathcal{I}_1 &\coloneqq \Big\{y_T : S_{\tau-1}(y_{T})\leq c_{14}\text{ and }-\log q_{\tau}(y_{\tau}) > c_{\tau} d\log T\Big\}, \label{eq:defn-I2-I3-I4-ode-I1}\\
\mathcal{I}_{2} & \coloneqq\Big\{ y_T : c_{14}\leq S_{\tau}\big(y_{T}\big)\leq2c_{14}\Big\} \cap \mathcal{I}_1^{\mathrm{c}},
	\label{eq:defn-I2-I3-I4-ode-I2}\\
\mathcal{I}_{3} & \coloneqq\bigg\{ y_T : S_{\tau-1}\big(y_{T}\big)\leq c_{14},\xi_{\tau}\big(y_T\big)\geq c_{14},\frac{q_{\tau-1}(y_{\tau-1})}{p_{\tau-1}(y_{\tau-1})}\leq\frac{8q_{\tau}(y_{\tau})}{p_{\tau}(y_{\tau})}\bigg\} \cap \mathcal{I}_1^{\mathrm{c}},
	\label{eq:defn-I2-I3-I4-ode-I3}\\
	\mathcal{I}_{4} & \coloneqq\bigg\{ y_T : S_{\tau-1}\big(y_{T}\big)\leq c_{14},\xi_{\tau}\big(y_T\big)\geq c_{14},\frac{q_{\tau-1}(y_{\tau-1})}{p_{\tau-1}(y_{\tau-1})}>\frac{8q_{\tau}(y_{\tau})}{p_{\tau}(y_{\tau})}\bigg\} \cap \mathcal{I}_1^{\mathrm{c}}.
\label{eq:defn-I2-I3-I4-ode-I4}
\end{align}
\end{subequations}
As an immediate consequence of the above definitions, one has
$$\mathcal{I}_{0} \cup\mathcal{I}_{1} \cup \mathcal{I}_{2} \cup \mathcal{I}_{3} \cup \mathcal{I}_{4} = \mathbb{R}^d.$$ 
In the following, we shall look at each of these sets separately, and combine the respective bounds to control the first term  on the right-hand side of \eqref{eqn:ode-tv-10}.

\paragraph{Step 3: coping with the set $\mathcal{I}_0$.} 
In order to obtain a useful bound when restricting attention to $\mathcal{I}_0$ (cf.~\eqref{eq:defn-I2-I3-I4-ode-I0}), 
we resort to the following key lemma, whose proof is provided in Section~\ref{sec:proof-lem:density-ratio-tau}. 
\begin{lemma}
	\label{lem:density-ratio-tau}
	Consider any $y_T$, along with the deterministic sequences $\{y_{T-1},\cdots,y_1\}$ and $\{y_{T}^{-},\cdots,y_2^{-}\}$. 
	Set $\tau=\tau(y_T)$  (cf.~\eqref{eq:defn-tao-i}). Then one has 
\begin{subequations}
	\label{eq:pt-qt-equiv-ODE-St}
\begin{align}
	\frac{q_{1}(y_{1})}{p_{1}(y_{1})}  = &\left\{ 1+O\Bigg(\frac{d^{6}\log^{6}T}{T^{2}}+S_{\tau-1}(y_{\tau-1})\Bigg)\right\} \frac{q_{\tau-1}(y_{\tau-1})}{p_{\tau-1}(y_{\tau-1})},	
	\label{eq:pt-qt-equiv-ODE-St-taui} \\
	&\text{and}
	\qquad \frac{q_{k}(y_{k})}{2p_{k}(y_{k})} \leq \frac{q_{1}(y_{1})}{p_{1}(y_{1})} \leq 2 \frac{q_{k}(y_{k})}{p_{k}(y_{k})}, \qquad \forall k < \tau. 
	\label{eq:pt-qt-equiv-ODE-St-k}
\end{align}
\end{subequations}
\end{lemma}
With this lemma in mind, we are ready to cope with the set $\mathcal{I}_0$.  Taking $\tau(y_T)=T+1$  in Lemma~\ref{lem:density-ratio-tau} yields
\begin{align}
	&  \mathop{\mathbb{E}}_{Y_{T}\sim p_{T}}\bigg[\Big(\frac{q_{1}(Y_{1})}{p_{1}(Y_{1})}-1\Big)\ind\left\{ Y_{1}\in\mathcal{E},Y_{T}\in\mathcal{I}_0 \right\} \bigg]\nonumber\\
	& \overset{\text{(i)}}{=} \mathop{\mathbb{E}}_{Y_{T}\sim p_{T}}\left[\left(\left\{ 1+O\Bigg(\frac{d^{6}\log^{6}T}{T^{2}}+S_{T}(y_{T})\Bigg)\right\} \frac{q_{T}(Y_{T})}{p_{T}(Y_{T})}-1\right)\ind\left\{ Y_{1}\in\mathcal{E},Y_{T}\in\mathcal{I}_0 \right\} \right]\nonumber\\
	& = {\displaystyle \int}\left\{ \left(1+O\Bigg(\frac{d^{6}\log^{6}T}{T^{2}}+S_{T}(y_{T})\Bigg)\right)q_{T}(y_{T})-p_{T}(y_{T})\right\} \ind\left\{ y_{1}\in\mathcal{E},y_{T}\in\mathcal{I}_0 \right\} \mathrm{d}y_{T}\nonumber\\
	& \overset{\text{(ii)}}{\leq} 
	   {\displaystyle \int}\big|q_{T}(y_{T})-p_{T}(y_{T})\big|\mathrm{d}y_{T}
	   + O\Bigg(\frac{d^{6}\log^{6}T}{T^{2}}\Bigg){\displaystyle \int}q_{T}(y_{T})\mathrm{d}y_{T}+O\left(\sqrt{d\log^{3}T}\varepsilon_{\score}+(d\log T)\varepsilon_{\Jacobi}\right)\nonumber\\
	& \overset{\text{(iii)}}{\lesssim} \frac{d^{6}\log^{6}T}{T^{2}}+\sqrt{d\log^{3}T}\varepsilon_{\score}+(d\log T)\varepsilon_{\Jacobi}. 
	   \label{eq:I0-expectation-UB-ode}
   \end{align}
   Here, (i) invokes Lemma~\ref{lem:density-ratio-tau}, whereas
   (iii) holds since $\mathsf{TV}(p_T,q_T)\lesssim T^{-100}$ (according to Lemma~\ref{lem:main-ODE}).   
   To see why (ii) is valid, it suffices to make the following observation: 
   \begin{align*}
 & {\displaystyle \int}S_{T}(y_{T})q_{T}(y_{T})\ind\left\{ y_{1}\in\mathcal{E},y_{T}\in\mathcal{I}_{0}\right\} \mathrm{d}y_{T}=\\
 & \quad=\frac{\log T}{T}\sum_{t=1}^{T}{\displaystyle \int}\left\{ d\big(\varepsilon_{\Jacobi,t}(y_{t})+\varepsilon_{\Jacobi,t+1}(y_{t}^{-})\big)+\sqrt{d\log T}\big(\varepsilon_{\score,t}(y_{t})+\varepsilon_{\score,t+1}(y_{t}^{-})\big)\right\} \\
 & \qquad\qquad\qquad\cdot q_{T}(y_{T})\ind\left\{ y_{1}\in\mathcal{E},y_{T}\in\mathcal{I}_{0}\right\} \mathrm{d}y_{T}\\
 & \quad\stackrel{\text{(iv)}}{\leq}\frac{4\log T}{T}\sum_{t=1}^{T}{\displaystyle \int}\left\{ d\big(\varepsilon_{\Jacobi,t}(y_{t})+\varepsilon_{\Jacobi,t+1}(y_{t}^{-})\big)+\sqrt{d\log T}\big(\varepsilon_{\score,t}(y_{t})+\varepsilon_{\score,t+1}(y_{t}^{-})\big)\right\} \\
 & \qquad\qquad\qquad\qquad\cdot\frac{q_{t}(y_{t})}{p_{t}(y_{t})}p_{T}(y_{T})\ind\left\{ y_{1}\in\mathcal{E},y_{T}\in\mathcal{I}_{0}\right\} \mathrm{d}y_{T}\\
 & \quad\leq\frac{4\log T}{T}\sum_{t=1}^{T}\mathop{\mathbb{E}}\limits _{Y_{T}\sim p_{T}}\bigg[\Big\{ d\big(\varepsilon_{\Jacobi,t}(Y_{t})+\varepsilon_{\Jacobi,t+1}(Y_{t}^{-})\big)+\sqrt{d\log T}\big(\varepsilon_{\score,t}(Y_{t})+\varepsilon_{\score,t+1}(Y_{t}^{-})\big)\Big\}\\
 & \qquad\qquad\qquad\qquad\left.\cdot\frac{q_{t}(Y_{t})}{p_{t}(Y_{t})}\ind\left\{ \sqrt{d\log T}\varepsilon_{\score,t}(Y_{t})+d\varepsilon_{\Jacobi,t}(Y_{t})\log T\lesssim T\right\} \right]\\
 & \quad=\frac{4\log T}{T}\sum_{t=1}^{T}\mathop{\mathbb{E}}\limits _{Y_{t}\sim p_{t}}\bigg[\Big\{ d\big(\varepsilon_{\Jacobi,t}(Y_{t})+\varepsilon_{\Jacobi,t+1}(Y_{t}^{-})\big)+\sqrt{d\log T}\big(\varepsilon_{\score,t}(Y_{t})+\varepsilon_{\score,t+1}(Y_{t}^{-})\big)\Big\}\\
 & \qquad\qquad\qquad\qquad\left.\cdot\frac{q_{t}(Y_{t})}{p_{t}(Y_{t})}\ind\left\{ \sqrt{d\log T}\varepsilon_{\score,t}(Y_{t})+d\varepsilon_{\Jacobi,t}(Y_{t})\log T\lesssim T\right\} \right]\\
 & \quad=\frac{4\log T}{T}\sum_{t=1}^{T}\mathop{\mathbb{E}}\limits _{Y_{t}\sim q_{t}}\bigg[\Big\{ d\big(\varepsilon_{\Jacobi,t}(Y_{t})+\varepsilon_{\Jacobi,t+1}(Y_{t}^{-})\big)+\sqrt{d\log T}\big(\varepsilon_{\score,t}(Y_{t})+\varepsilon_{\score,t+1}(Y_{t}^{-})\big)\Big\}\\
 & \qquad\qquad\qquad\qquad\cdot\ind\left\{ \sqrt{d\log T}\varepsilon_{\score,t}(Y_{t})+d\varepsilon_{\Jacobi,t}(Y_{t})\log T\lesssim T\right\} \bigg]\\
 & \quad\stackrel{\text{(v)}}{\lesssim}\frac{\log T}{T}\sum_{t=1}^{T}\mathop{\mathbb{E}}\limits _{Y_{t}\sim q_{t}}\left[d\varepsilon_{\Jacobi,t}(Y_{t})+\sqrt{d\log T}\varepsilon_{\score,t}(Y_{t})\right]\\
 & \quad\stackrel{\text{(vi)}}{\lesssim}(d\log T)\varepsilon_{\Jacobi}+\sqrt{d\log^{3}T}\varepsilon_{\score}. 	
   \end{align*}
   Here, (iv) follows from \Cref{lem:density-ratio-tau}, 
   while (vi) comes from \eqref{eq:score-assumptions-equiv}. To understand why (v) is valid, let us denote the probability density of $\Phi(X_t)$ by $q_t^{-}$ and, by referring to \eqref{eq:phit}, we derive that
      \begin{subequations}\label{eq:yminus-control}
	\begin{align}
		&\mathop{\mathbb{E}}\limits_{Y_{t}\sim q_{t}}\left[\varepsilon_{\score,t+1}(Y_{t}^-)\ind\Big\{\sqrt{d\log T}\varepsilon_{\score, t}(Y_{t})+d\varepsilon_{\Jacobi,t}(Y_{t})\log T \lesssim T\Big \}\right]\notag\\
		&~~~~=\mathop{\mathbb{E}}\limits_{Y_{t}^{-}\sim q^{-}_{t}}\left[\varepsilon_{\score,t+1}(Y_{t}^-)\ind\Big\{\sqrt{d\log T}\varepsilon_{\score, t}(Y_{t})+d\varepsilon_{\Jacobi,t}(Y_{t})\log T \lesssim T\Big \}\right]\notag\\
		&~~~~\lesssim \mathop{\mathbb{E}}\limits_{Y_{t}^-\sim q_{t+1}}\left[\varepsilon_{\score,t+1}(Y_{t}^-)\right]. 
		\end{align}
		Similarly, we have 
			\begin{align}
		\mathop{\mathbb{E}}\limits_{Y_{t}\sim q_{t}}\left[\varepsilon_{\Jacobi,t+1}(Y_{t}^-)\ind\Big\{\sqrt{d\log T}\varepsilon_{\score, t}(Y_{t})+d\varepsilon_{\Jacobi,t}(Y_{t})\log T \lesssim T\Big\}\right]
		&\lesssim \mathop{\mathbb{E}}\limits_{Y_{t}^-\sim q_{t+1}}\left[\varepsilon_{\Jacobi,t+1}(Y_{t}^-)\right]. 
		\end{align}
   \end{subequations}
   %
   %

\paragraph{Step 4: coping with the set $\mathcal{I}_1$.} In view of \Cref{lem:density-ratio-tau}, the condition $S_{\tau-1}(y_{\tau-1})\leq c_{14}$ implies that
\begin{align*}
\frac{q_{1}(y_{1})}{p_{1}(y_{1})} \leq \frac{2q_{\tau-1}(y_{\tau-1})}{p_{\tau-1}(y_{\tau-1})}.
\end{align*}
This in turn allows one to obtain
\begin{align}
\mathop{\mathbb{E}}\limits_{Y_{T}\sim p_{T}}\bigg[\frac{q_{1}(Y_{1})}{p_{1}(Y_{1})}\ind\left\{ Y_{1}\in\mathcal{E},Y_{T}\in\mathcal{I}_{1}\right\} \bigg] 
	& \leq 2 \mathop{\mathbb{E}}\limits_{Y_{T}\sim p_{T}}\bigg[\frac{q_{\tau-1}(Y_{1})}{p_{\tau-1}(Y_{1})}\ind\left\{ Y_{1}\in\mathcal{E},Y_{T}\in\mathcal{I}_{1}\right\} \bigg]\notag\\
&= 2 \sum_{t=2}^{T}\mathop{\mathbb{E}}\limits_{Y_{T}\sim p_{T}}\bigg[\frac{q_{t-1}(Y_{1})}{p_{t-1}(Y_{1})}\ind\left\{ Y_{1}\in\mathcal{E},Y_{T}\in\mathcal{I}_{1}\right\} \ind\{\tau=t\}\bigg]\notag\\
	& \leq 2\sum_{t=2}^{T}\mathop{\mathbb{E}}\limits_{Y_{T}\sim p_{T}}\bigg[\frac{q_{t-1}(Y_{t-1})}{p_{t-1}(Y_{t-1})}\ind\left\{Y_{1}\in\mathcal{E}, Y_{t}\in \mathcal{J}_{t}\right\} \bigg], 
	\label{eq:sum-I4-UB-13579}
\end{align}
where the last line comes from the definition of $\mathcal{I}_1$ (cf.~\eqref{eq:defn-I2-I3-I4-ode-I1}), with $\mathcal{J}_t$ defined as 
\begin{align}
	\mathcal{J}_t &\coloneqq \Big\{y_t : -\log q_{t}(y_{t})\geq c_{\tau} d\log T\Big\}.
\end{align}
%

To bound the right-hand side of \eqref{eq:sum-I4-UB-13579}, we make note of the following identities: 
\begin{subequations}
\begin{align}
	1 	
	&= \mathop{\mathbb{E}}\limits_{Y_{t}\sim p_{t}}\bigg[\frac{q_{t}(Y_{t})}{p_{t}(Y_{t})}\bigg]
	= \mathop{\mathbb{E}}\limits_{Y_{T}\sim p_{T}}\bigg[\frac{q_{t}(Y_{t})}{p_{t}(Y_{t})}\bigg]
	= \mathop{\mathbb{E}}\limits_{Y_{T}\sim p_{T}}\bigg[\frac{q_{t}(Y_{t})}{p_{t}(Y_{t})}\big(\ind\left\{Y_{t}\in \mathcal{J}_{t}\right\}+\ind\left\{Y_{t}\in \mathcal{J}^c_{t}\right\}\big) \bigg], \\
	1 	
	&
	= \mathop{\mathbb{E}}\limits_{Y_{T}\sim p_{T}}\bigg[\frac{q_{t-1}(Y_{t-1})}{p_{t-1}(Y_{t-1})}\bigg]
	= \mathop{\mathbb{E}}\limits_{Y_{T}\sim p_{T}}\bigg[\frac{q_{t-1}(Y_{t-1})}{p_{t-1}(Y_{t-1})}\big(\ind\left\{Y_{t}\in \mathcal{J}_{t}\right\}+\ind\left\{Y_{t}\in \mathcal{J}^c_{t}\right\}\big) \bigg],
\end{align}
\end{subequations}
which in turn imply that
\begin{align*}
	&\mathop{\mathbb{E}}\limits_{Y_{T}\sim p_{T}}\bigg[\frac{q_{t-1}(Y_{t-1})}{p_{t-1}(Y_{t-1})}\ind\left\{Y_{t}\in \mathcal{J}_{t}\right\} \bigg]\\
	&\qquad = 1-\mathop{\mathbb{E}}\limits_{Y_{T}\sim p_{T}}\bigg[\frac{q_{t-1}(Y_{t-1})}{p_{t-1}(Y_{t-1})}\ind\left\{Y_{t}\in \mathcal{J}^c_{t}\right\}\bigg]
	\\
	&\qquad = \mathop{\mathbb{E}}\limits_{Y_{T}\sim p_{T}}\bigg[\left(\frac{q_{t}(Y_{t})}{p_{t}(Y_{t})}-\frac{q_{t-1}(Y_{t-1})}{p_{t-1}(Y_{t-1})}\right)\ind\left\{Y_{t}\in \mathcal{J}^c_{t}\right\}\bigg]+\mathop{\mathbb{E}}\limits_{Y_{T}\sim p_{T}}\bigg[\frac{q_{t}(Y_{t})}{p_{t}(Y_{t})}\ind\left\{Y_{t}\in \mathcal{J}_{t}\right\} \bigg] \\
	&\qquad \leq  \mathop{\mathbb{E}}\limits_{Y_{T}\sim p_{T}}\bigg[\left(\frac{q_{t}(Y_{t})}{p_{t}(Y_{t})}-\frac{q_{t-1}(Y_{t-1})}{p_{t-1}(Y_{t-1})}\right)\ind\left\{Y_{t}\in \mathcal{J}^c_{t}\right\}\bigg]+ O\Big( \exp\big(-c_{6}d\log T\big) \Big).
\end{align*}
Here, the last line follows since
%
\begin{align*}
\mathop{\mathbb{E}}\limits _{Y_{T}\sim p_{T}}\bigg[\frac{q_{t}(Y_{t})}{p_{t}(Y_{t})}\ind\left\{ Y_{t}\in\mathcal{J}_{t}\right\} \bigg] & =\mathop{\mathbb{E}}\limits _{Y_{t}\sim p_{t}}\bigg[\frac{q_{t}(Y_{t})}{p_{t}(Y_{t})}\ind\left\{ Y_{t}\in\mathcal{J}_{t}\right\} \bigg]=\mathop{\mathbb{E}}\limits _{Y_{t}\sim q_{t}}\big[\ind\left\{ Y_{t}\in\mathcal{J}_{t}\right\} \big]\\
 & \lesssim\exp\big(-c_{6}d\log T\big),
\end{align*}
provided that $c_{20}>0$ is large enough.  
As a consequence, 
the right-hand side of \eqref{eq:sum-I4-UB-13579} can be bounded by 
\begin{align}
&\sum_{t=2}^{T}\mathop{\mathbb{E}}\limits_{Y_{T}\sim p_{T}}\bigg[\frac{q_{t-1}(Y_{t-1})}{p_{t-1}(Y_{t-1})}\ind\left\{Y_{1}\in\mathcal{E}, Y_{t}\in \mathcal{J}_{t}\right\} \bigg]\notag\\
&\qquad \leq \sum_{t=2}^{T} \mathop{\mathbb{E}}\limits_{Y_{T}\sim p_{T}}\bigg[\bigg(\frac{q_{t}(Y_{t})}{p_{t}(Y_{t})} - \frac{q_{t-1}(Y_{t-1})}{p_{t-1}(Y_{t-1})}\bigg)\ind\left\{Y_{t}\in \mathcal{J}_{t}^{\mathrm{c}}\right\} \bigg] +
	O\Big( T\exp\big(-c_{6}d\log T\big) \Big) .
	\label{eq:sum-I4-UB-2468}
\end{align}

Moreover, the first term in the last line \eqref{eq:sum-I4-UB-2468} can be decomposed as follows
\begin{align}
	&\mathop{\mathbb{E}}\limits_{Y_{T}\sim p_{T}}\bigg[\bigg(\frac{q_{t}(Y_{t})}{p_{t}(Y_{t})} - \frac{q_{t-1}(Y_{t-1})}{p_{t-1}(Y_{t-1})}\bigg)\ind\left\{Y_{t}\in \mathcal{J}_{t}^{\mathrm{c}}\right\} \bigg] \notag\\
&\qquad = \mathop{\mathbb{E}}\limits_{Y_{T}\sim p_{T}}\bigg[\bigg(\frac{q_{t}(Y_{t})}{p_{t}(Y_{t})} - \frac{q_{t-1}(Y_{t-1})}{p_{t-1}(Y_{t-1})}\bigg)\ind\left\{Y_{t}\in \mathcal{J}_{t}^{\mathrm{c}}, \xi_t(Y_{t}) < c_{14}\right\} \bigg] \notag\\
&\qquad\qquad~+ \mathop{\mathbb{E}}\limits_{Y_{T}\sim p_{T}}\bigg[\bigg(\frac{q_{t}(Y_{t})}{p_{t}(Y_{t})} - \frac{q_{t-1}(Y_{t-1})}{p_{t-1}(Y_{t-1})}\bigg)\ind\left\{Y_{t}\in \mathcal{J}_{t}^{\mathrm{c}}, \xi_t(Y_{t}) \ge c_{14}\right\} \bigg],
	\label{eq:E-sum-Jtc-135}
\end{align}
leaving us with two terms to control. 
\begin{itemize}
	\item 
With regards to the first term on the right-hand side of \eqref{eq:E-sum-Jtc-135}, for $Y_t$ satisfies $\log q_t(Y_t)\leq -c_{\tau}d \log T$ and $\xi_t(Y_{t}) < c_{14}$,  we can directly apply \Cref{lem:main-ODE} to obtain a one-step version of \Cref{lem:density-ratio-tau} to control the difference between the density ratio $\frac{q_{t}(Y_{t})}{p_{t}(Y_{t})}$ and $\frac{q_{t-1}(Y_{t-1})}{p_{t-1}(Y_{t-1})}$  as follows 
\begin{align*}
	\frac{p_{t-1}(y_{t-1})}{q_{t-1}(y_{t-1})} & =\frac{p_{t}(y_{t})}{q_{t}(y_{t})} \cdot \Bigg\{ 1+ O\Bigg(\frac{d^{6}\log^{6}T}{T^{3}}\Bigg)+\\
	&~~~~O\Bigg(\frac{\big(\varepsilon_{\score,t}(y_{t})+\varepsilon_{\score,t}(\Phi_t(y_t)) \big)\sqrt{d\log^{3}T}}{T}+\frac{d\log T\big(\varepsilon_{\Jacobi,t}(y_{t})+\varepsilon_{\Jacobi,t}(\Phi_t(y_{t}))\big)}{T}\Bigg) \Bigg\} 
	\end{align*}
which is an intermediate step in the proof of \Cref{lem:density-ratio-tau} (referring \eqref{eq:relation-ratioy-one-step} in Section~\ref{sec:proof-lem:density-ratio-tau} for details). The above relation yields:
\begin{align*}
	& \sum_{t=2}^{T}\mathop{\mathbb{E}}\limits_{Y_{T}\sim p_{T}}\bigg[\bigg(\frac{q_{t}(Y_{t})}{p_{t}(Y_{t})} - \frac{q_{t-1}(Y_{t-1})}{p_{t-1}(Y_{t-1})}\bigg)\ind\left\{Y_{t}\in \mathcal{J}_{t}^{\mathrm{c}}, \xi_t(Y_{t}) < c_{14}\right\} \bigg] \\
	&\qquad\qquad \lesssim \frac{d^{6}\log^{6}T}{T^2} + \sqrt{d\log^3 T}\varepsilon_{\score} + {d\log T}\varepsilon_{\Jacobi} .
	\end{align*}

\item When it comes to the second term on the right-hand side of \eqref{eq:E-sum-Jtc-135}, we  can invoke  similar arguments as in \citet{li2023towards} (i.e., the arguments therein to bound  $\mathcal{I}_3$), as well as the relation \eqref{eq:yminus-control} for the score error of ${Y}_t^{-}$,  to obtain
\begin{align*}
	\sum_{t=2}^{T}\mathop{\mathbb{E}}\limits_{Y_{T}\sim p_{T}}
	\bigg[\frac{q_{t}(Y_{t})}{p_{t}(Y_{t})} \ind\big\{Y_{t}\in \mathcal{J}_{t}^{\mathrm{c}}, \xi_t(Y_{t}) \ge c_{14}\big\} \bigg]\lesssim  \sqrt{d\log^3 T}\varepsilon_{\score} + {d\log T}\varepsilon_{\Jacobi}.
\end{align*}
\end{itemize}
Putting all this together, we arrive at
	\begin{align}
		\mathop{\mathbb{E}}\limits_{Y_{T}\sim p_{T}}\bigg[\frac{q_{1}(Y_{1})}{p_{1}(Y_{1})}\ind\big\{ Y_{1}\in\mathcal{E},Y_{T}\in\mathcal{I}_{1}\big\} \bigg] 
			\lesssim \frac{d^{6}\log^{6}T}{T^2} + \sqrt{d\log^3 T}\varepsilon_{\score} + {d\log T}\varepsilon_{\Jacobi}. 	   \label{eq:I1-expectation-UB-ode}
		\end{align}

\paragraph{Step 5: coping with the remaining sets.}
The analyses for $\mathcal{I}_2, \mathcal{I}_3$ and $\mathcal{I}_4$ are similar to~\citet{li2023towards}. 
For the sake of brevity, 
we state the combined result in the lemma below and omit the proof. 

\begin{lemma}
	\label{lem:I2-I3-I4-bound}
	It holds that
	\begin{align}
		\mathop{\mathbb{E}}_{Y_{T}\sim p_{T}}\bigg[\frac{q_{1}(Y_{1})}{p_{1}(Y_{1})}\ind\left\{ Y_{1}\in\mathcal{E},Y_{T}\in\mathcal{I}_2\cup \mathcal{I}_3 \cup \mathcal{I}_4\right\} \bigg]
		&\lesssim \frac{d^{6}\log^{6}T}{T^{2}}+\sqrt{d\log^{3}T}\varepsilon_{\score}+(d\log T)\varepsilon_{\Jacobi}. 
		\label{eq:I2-I3-I4-bound}
	\end{align}
\end{lemma}

\paragraph{Step 6: putting all pieces together.} Given the preceding results on $\mathcal{I}_0$ to $\mathcal{I}_4$, we can substitute the upper bounds derived in \eqref{eq:I0-expectation-UB-ode},\eqref{eq:I1-expectation-UB-ode} and \eqref{eq:I2-I3-I4-bound} back into \eqref{eqn:ode-tv-10}, which leads to the following conclusion:
$$
\begin{aligned}
\operatorname{TV}\left(p_1, q_1\right) &\leq  \underset{Y_T \sim p_T}{\mathbb{E}}\left[\left(\frac{q_1\left(Y_1\right)}{p_1\left(Y_1\right)}-1\right) \ind\left\{Y_1 \in \mathcal{E}, Y_T \in \mathcal{I}_0\right\}\right]+\underset{Y_T \sim p_T}{\mathbb{E}}\left[\frac{q_1\left(Y_1\right)}{p_1\left(Y_1\right)} \ind\left\{Y_1 \in \mathcal{E}, Y_T \in \mathcal{I}_1\right\}\right] \\
	&~~~~ +\underset{Y_T \sim p_T}{\mathbb{E}}\left[\frac{q_1\left(Y_1\right)}{p_1\left(Y_1\right)} \ind\left\{Y_1 \in \mathcal{E}, Y_T \in \mathcal{I}_2 \cup \mathcal{I}_3 \cup \mathcal{I}_4\right\}\right]+ O\Big( \exp \left(-c_6 d \log T\right) \Big) \\
&\lesssim \frac{d^{6}\log^{6}T}{T^{2}}+\sqrt{d\log^{3}T}\varepsilon_{\score}+(d\log T)\varepsilon_{\Jacobi}.
\end{aligned}
$$

\subsection{Proof of Lemma~\ref{lem:main-ODE}}
\label{sec:proof-lemma-main-ODE}

\subsubsection{Proof of property~\eqref{eq:xt_lb}}\label{sec-proof-xt_lb}
This property can be established in a similar way to \citet[Lemma 4]{li2023towards}. Before proceeding, let us introduce the following vector:
$$
\begin{aligned}
	v_t(x) & \coloneqq x-\phi_t(x)=x-\phi_t^{\star}(x)+\phi_t^{\star}(x)-\phi_t(x) \\
	& =-\frac{1-\alpha_{t+1}}{2\left(1-\overline{\alpha}_t\right)} \int_{x_0}\left(x-\sqrt{\overline{\alpha}_t} x_0\right) p_{X_0 \mid X_t}\left(x_0 \mid x\right) \mathrm{d} x_0+\frac{1-\alpha_{t+1}}{2}\big(s_t(x)-s_t^{\star}(x)\big) ,
	\end{aligned}
$$
where we have invoked the definitions of $\phi_t^{\star}$ and $\phi_t$, as well as the property \eqref{eq:st-MMSE-expression}. 
For notational simplicity, we shall abbreviate  $v=v_t(x)$ in the following analysis.

Recognizing that 
$$
X_t \stackrel{\mathrm{d}}{=} \sqrt{\overline{\alpha}_t} X_0+\sqrt{1-\overline{\alpha}_t} W \quad \text { with } W \sim \mathcal{N}\left(0, I_d\right)
$$
and making use of the Bayes rule, we can express the conditional distribution $p_{X_0 \mid X_t}$ as
$$
p_{X_0 \mid X_t}\left(x_0 \mid x\right)=\frac{p_{X_0}\left(x_0\right)}{p_{X_t}(x)} p_{X_t \mid X_0}\left(x \mid x_0\right)=\frac{p_{X_0}\left(x_0\right)}{p_{X_t}(x)} \cdot \frac{1}{\left(2 \pi\left(1-\overline{\alpha}_t\right)\right)^{d / 2}} \exp \left(-\frac{\left\|x-\sqrt{\overline{\alpha}_t} x_0\right\|_2^2}{2\left(1-\overline{\alpha}_t\right)}\right). 
$$
Additionally, recalling that
$$
\frac{1}{\sqrt{\alpha_{t+1}}}X_{t+1}\stackrel{\mathrm{d}}{=} \frac{1}{\sqrt{\alpha_{t+1}}}\left(\sqrt{\overline{\alpha}_{t+1}} X_0+\sqrt{1-\overline{\alpha}_{t+1}} W\right)
= \sqrt{\overline{\alpha}_t} X_0+\sqrt{\frac{1}{\alpha_{t+1}}-\overline{\alpha}_t} W ,
$$
one can  demonstrate that
\begin{align*}
\frac{p_{X_{t+1}/\sqrt{\alpha_{t+1}}}\big(\phi_{t}(x)\big)}{p_{X_{t}}(x)} &=\frac{1}{p_{X_t}(x)} \int_{x_0} p_{X_0}\left(x_0\right) \frac{1}{\left(2 \pi\left(\alpha_{t+1}^{-1}-\overline{\alpha}_{t}\right)\right)^{d / 2}} \exp \left(-\frac{\left\|\phi_t(x)-\sqrt{\overline{\alpha}_t} x_0\right\|_2^2}{2\left(\alpha_{t+1}^{-1}-\overline{\alpha}_{t}\right)}\right) \mathrm{d} x_0\\
&=\frac{1}{p_{X_t}(x)} \int_{x_0} p_{X_0}\left(x_0\right) \frac{1}{\left(2 \pi\left(\alpha_{t+1}^{-1}-\overline{\alpha}_t\right)\right)^{d / 2}} \exp \left(-\frac{\left\|x-\sqrt{\overline{\alpha}_t} x_0\right\|_2^2}{2\left(1-\overline{\alpha}_t\right)}\right)\\
&\qquad~~~~~~~~ \cdot \exp\bigg(-\frac{(1-\alpha_{t+1}^{-1})\big\| x-\sqrt{\overline{\alpha}_{t}}x_{0}\big\|_{2}^{2}}{2(\alpha_{t+1}^{-1}-\overline{\alpha}_{t})(1-\overline{\alpha}_{t})}-\frac{\|v\|_{2}^{2}-2v^{\top}\big(x-\sqrt{\overline{\alpha}_{t}}x_{0}\big)}{2(\alpha_{t+1}^{-1}-\overline{\alpha}_{t})}\bigg)\mathrm{d}x_{0}\\
&= \Big(\frac{1-\overline{\alpha}_{t}}{\alpha_{t+1}^{-1}-\overline{\alpha}_{t}}\Big)^{d/2}\cdot\int_{x_{0}}p_{X_{0}\mymid X_{t}}(x_{0}\mymid x)\cdot\notag\\
 & \qquad\qquad\qquad~~~~~~~~\exp\bigg(\frac{(\alpha_{t+1}^{-1}-1)\big\| x-\sqrt{\overline{\alpha}_{t}}x_{0}\big\|_{2}^{2}}{2(\alpha_{t+1}^{-1}-\overline{\alpha}_{t})(1-\overline{\alpha}_{t})}-\frac{\|v\|_{2}^{2}-2v^{\top}\big(x-\sqrt{\overline{\alpha}_{t}}x_{0}\big)}{2(\alpha_{t+1}^{-1}-\overline{\alpha}_{t})}\bigg)\mathrm{d}x_{0}. 
\end{align*}
To lower bound the above expression, we can focus on controlling the second term within the exponential part of the integral over the set $\mathcal{G}_c^{\text {typical }}$ defined as follows, since the first term is always non-negative:
\begin{align}
	\mathcal{G}_c^{\text {typical }}
\coloneq \left\{x_0:\left\|x-\sqrt{\overline{\alpha}_t} x_0\right\|_2 \leq  5c\sqrt{c_6 d\left(1-\overline{\alpha}_t\right) \log T}\right\}. \label{eqn:G-set}
\end{align}
%
Furthermore, we make the following observations:
\begin{subequations}\label{eqn: lemma1-1}
	\begin{itemize}
		\item  When \eqref{eqn:choice-y} holds, \Cref{lem:x0}  implies that 
		\begin{align}
			\mathbb{P}\left(\left\|\sqrt{\overline{\alpha}_t} X_0-x\right\|_2>5 c_5 \sqrt{c_6 d\left(1-\overline{\alpha}_t\right) \log T} \mid X_t=x\right) \leq \exp \left(-c_5^2 c_6 d \log T\right) 
		\end{align}
		for any quantity $c_5 \geq 2$, provided that $c_6 \geq 2 c_R+c_0$.
		\item The we can makes use of  the above relation   to bound $v$ as follows: 
		\begin{align}
		\|v\|_2 & \leq \frac{1-\alpha_{t+1}}{2} \varepsilon_{\text {score }, t}(x)+\frac{1-\alpha_{t+1}}{2\left(1-\overline{\alpha}_t\right)} \mathbb{E}\left[\left\|\sqrt{\overline{\alpha}_t} X_0-x\right\|_2 \mid X_t=x\right] \nonumber\\
		& \leq \frac{1-\alpha_{t+1}}{2} \varepsilon_{\text {score }, t}(x)+\frac{6\left(1-\alpha_{t+1}\right)}{1-\overline{\alpha}_t} \sqrt{c_6 d\left(1-\overline{\alpha}_t\right) \log T}.
		\end{align}
	\end{itemize}
\end{subequations}
Therefore, 
we obtain that: 
for any 
$x_0\in\mathcal{G}$, 
$$
\begin{aligned}
\frac{\|v\|_2^2}{2\left(\alpha_{t+1}^{-1}-\overline{\alpha}_t\right)} & \stackrel{(i)}{\leq} \frac{\left(1-\alpha_{t+1}\right)^2}{4\left(\alpha_{t+1}^{-1}-\overline{\alpha}_{t}\right)} \varepsilon_{\text {score }, t}(x)^2+\frac{36\left(1-\alpha_{t+1}\right)^2}{\left(1-\overline{\alpha}_t\right)^2\left(\alpha_{t+1}^{-1}-\overline{\alpha}_t\right)} c_6 d \log T \\
& \stackrel{(ii)}{\leq} \frac{2 c_1^2 \log ^2 T}{T^2} \varepsilon_{\text {score }, t}(x)^2+\frac{2304 c_1^2}{T^2} c_6 d \log ^3 T; \\
\bigg|\frac{v^{\top}\big(x-\sqrt{\overline{\alpha}_{t}}x_{0}\big)}{\alpha_{t+1}^{-1}-\overline{\alpha}_{t}}\bigg|  & \stackrel{(iii)}{\leq} \frac{\|v\|_2\left\|x-\sqrt{\overline{\alpha}_t} x_0\right\|_2}{\alpha^{-1}_{t+1}-\overline{\alpha}_t}\\
&\stackrel{(iv)}{\leq} \frac{20 c c_1}{T} \varepsilon_{\text {score }, t}(x) \sqrt{c_6 d\left(1-\bar{\alpha}_t\right) \log ^3 T}+ \frac{240 c c_1 c_6 d \log ^2 T}{T} 
\end{aligned}
$$
Here, (i) is due to \eqref{eqn: lemma1-1}; (ii) holds by the choice of learning rates in \eqref{eqn:properties-alpha-proof}; (iii) follows from the Cauchy-Schwarz inequality; and (iv) comes from the definition of $\mathcal{G}$ and \eqref{eqn:properties-alpha-proof}. Moreover,  \eqref{eqn:properties-alpha-proof} also guarantees that
\begin{align*}
	\Big(\frac{1-\overline{\alpha}_{t}}{\alpha_{t+1}^{-1}-\overline{\alpha}_{t}}\Big)^{d/2}=
	\left(1-\frac{1-\alpha_{t+1}}{1-\overline{\alpha}_{t+1}}\right)^{\frac{d}{2}} &\ge \exp\bigg(-\frac{4c_1d\log T}{T}\bigg).
\end{align*}
%
Combine the above relations to yield 
\begin{align*}
	&\frac{p_{X_{t+1}/\sqrt{\alpha_{t+1}}}\big(\phi_{t}(x)\big)}{p_{X_{t}}(x)} \\	
	&~~\geq  \Big(\frac{1-\overline{\alpha}_{t}}{\alpha_{t+1}^{-1}-\overline{\alpha}_{t}}\Big)^{d/2}\cdot\int_{x_{0}\in\mathcal{G}}p_{X_{0}\mymid X_{t}}(x_{0}\mymid x)\cdot\notag\exp\bigg(-\frac{\|v\|_{2}^{2}-2v^{\top}\big(x-\sqrt{\overline{\alpha}_{t}}x_{0}\big)}{2(\alpha_{t+1}^{-1}-\overline{\alpha}_{t})}\bigg)\mathrm{d}x_{0} 
	 \\&~~\ge \exp\bigg(-\Big(20c_1\varepsilon_{\score, t}(x)\sqrt{c_6d\log T}+240c_1d\log T\Big)\frac{c\log T}{T} \bigg)\\
	 &~~~~\cdot \exp\bigg(-\frac{4c_1d\log T}{T}-\frac{2304 c_1^2}{T^2} c_6 d \log ^3 T- \frac{2c_1\varepsilon_{\score, t}(x)^2\log^2 T}{T^2}\bigg)\\
	 &\ge \exp\bigg(-\Big(20\varepsilon_{\score, t}(x)\sqrt{d\log T}+300d\log T\Big)\frac{cc_1\log T}{T} 
	 \bigg)
\end{align*}
provided that $T\geq \frac{386}{5}c_1c_6\log T$ and $\frac{40 c_1 \varepsilon_{\text {score }, t}(x) \log ^{\frac{3}{2}} T}{T} \leq \sqrt{d}$. Taking any fixed $c\geq 2$ we obtain the desired result.

\subsubsection{Proof of property~\eqref{eq:xt}}
Before embarking on the proof, 
we first single out some useful properties about $\psi_t^{\star}$ and  $s_t^{*}$. 
\begin{lemma}	\label{lem:psi-prop}
	Under
	 the same conditions as \Cref{lem:main-ODE}, 
	 it holds that  
	\begin{subequations}
		\begin{align}
			&\|\psi_t(x) -\psi_t^{\star}(x)\|_2= O\bigg(\frac{\log T}{T}\Big\{\varepsilon_{\score, t}(x) + \varepsilon_{\score, t+1}\big(\Phi_t(x)\big)\Big\}\bigg),  
			\label{eq:psi-error}
			%
				\end{align}
				and 
				\begin{align}
					\frac{\partial \psi_t(x, \Phi_t(x))}{\partial x}&= \frac{\partial \psi^{\star}_t(x, \Phi^{\star}_t(x))}{\partial x} +\widetilde{\zeta_{t}}\notag\\  &=\Bigg(I - \frac{1-\alpha_{t}}{2(1-\overline{\alpha}_t)}J_{t}(x) + \frac{(1-\alpha_{t})^2}{4(1-\alpha_{t+1})}\frac{\partial \big(s_{t}^{\star}(x) - \sqrt{\alpha_{t+1}}s_{t+1}^{\star}(\Phi_t^{\star}(x))\big)}{\partial x} \Bigg)+\widetilde{\zeta_{t}}
					\label{eq:partial-psi-error}
				\end{align}
				where the residual term $\widetilde{\zeta_{t}}$ satisfies $$\|\widetilde{\zeta_{t}}\| = O\bigg(\frac{\log T}{T}\Big\{\varepsilon_{\score, t}(x)/\sqrt{d} + \varepsilon_{\Jacobi, t}(x) + \varepsilon_{\Jacobi, t+1}\big(\Phi_t(x)\big)\Big\}\bigg).$$
	\end{subequations}
\end{lemma}

Additionally, we introduce the following notations for simplicity: 
\begin{subequations}
	\label{eq:defn-w-z-lem}
\begin{align}
w &= \Phi_{t}^{\star}(x) = \sqrt{\alpha_{t+1}}\bigg(x-\frac{1-\alpha_{t+1}}{2}s_{t}^{\star}(x)\bigg), \\
z &= -(1 - \overline{\alpha}_{t})s_{t}^{\star}(x) = g_t(x).
\end{align}
\end{subequations}
%
Then the characterization of $s_{t}^{*}$ is summarized in the following lemma.  
%
\begin{lemma}
\label{lem:relation-yt}
Under
	 the same conditions as \Cref{lem:main-ODE}, equipped with the notation \eqref{eq:defn-w-z-lem}, we can write
\begin{align}
s_{t}^{\star}(x) - \sqrt{\alpha_{t+1}}s_{t+1}^{\star}(w) &= - \bigg(\frac{1-\alpha_{t+1}}{2(1-\overline{\alpha}_{t})^2} - \frac{1-\alpha_{t+1}}{2(1-\overline{\alpha}_{t})^3}\|z\|_2^2\bigg) z 
	 \notag\\
&- \frac{1-\alpha_{t+1}}{2(1-\overline{\alpha}_{t})^3}\int_{x_0} p_{X_0 \mymid X_{t}}(x_0 \mymid x)\big(x - \sqrt{\overline{\alpha}_{t}}x_0\big)\big(x - \sqrt{\overline{\alpha}_{t}}x_0\big)^{\top}z \mathrm{d} x_0 \notag\\
&+ \frac{1-\alpha_{t+1}}{2(1-\overline{\alpha}_{t})^3}\int_{x_0} p_{X_0 \mymid X_{t}}(x_0 \mymid x)\big\|x - \sqrt{\overline{\alpha}_{t}}x_0\big\|_2^2\big(x - \sqrt{\overline{\alpha}_{t}}x_0 - z\big) \mathrm{d} x_0 + \zeta_{s^{*}_{t}}
	\label{eq:lem-relation-yt-1st}
\end{align}
where $\|\zeta_{s^{*}_{t}}\|_2 = O\bigg(\frac{(d(1-\alpha_{t+1})\log T)^{3/2}}{(1 - \overline{\alpha}_{t})^2}\bigg)$,
and
\begin{align}
&\frac{\partial \big(s_{t}^{\star}(x) - \sqrt{\alpha_{t+1}}s_{t+1}^{\star}(w)\big)}{\partial x} =-\frac{1-\alpha_{t+1}}{2(1-\overline{\alpha}_{t})^2}J_t(x) + \frac{1-\alpha_{t+1}}{2(1-\overline{\alpha}_{t})^3}(H_1 + H_4 + H_2 - H_3) + \zeta_{J_{t}}
\label{eq:lem-relation-yt-2nd}
\end{align}
where $\|\zeta_{J_{t}}\| =  O\Big(d^2\frac{(1-\alpha_{t+1})^{3/2}}{(1-\overline{\alpha}_{t+1})^{5/2}}\log^2 T\Big)$. Here, we denote $J_t = \frac{\partial z}{\partial x}$, and  
\begin{subequations}
	\begin{align*}
		H_1  &\coloneqq\frac{\partial}{\partial x} \|z\|_2^2z, 
		\\
		H_2  &\coloneqq\frac{\partial}{\partial x} \int_{x_0} p_{X_0 \mymid X_{t}}(x_0 \mymid x)\big\|x - \sqrt{\overline{\alpha}_{t}}x_0\big\|_2^2\big(x - \sqrt{\overline{\alpha}_{t}}x_{0}\big) \mathrm{d} x_0, \\
	H_3&\coloneqq\frac{\partial}{\partial x} \int_{x_0} p_{X_0 \mymid X_{t}}(x_0 \mymid x)\big\|x - \sqrt{\overline{\alpha}_{t}}x_0\big\|_2^2z ,\\
	H_4&\coloneqq\frac{\partial}{\partial x} \int_{x_0} p_{X_0 \mymid X_{t}}(x_0 \mymid x)\big(x - \sqrt{\overline{\alpha}_{t}}x_0\big)\big(x - \sqrt{\overline{\alpha}_{t}}x_0\big)^{\top}z \mathrm{d} x_0 .
	\end{align*}
	\end{subequations}
\end{lemma}
\noindent 
To streamline presentation, we leave the proofs of \Cref{lem:psi-prop} and \Cref{lem:relation-yt} to Section~\ref{sec:proof-additional-lems-ODE}.   

Equipped with the  relations in \Cref{lem:relation-yt}, we can  derive that 
\begin{subequations}\label{eqn-relation-ss}
	\begin{align}
		\big\|s_{t}^{\star}(x) - \sqrt{\alpha_{t+1}}s_{t+1}^{\star}(w)\big\|_2 &\lesssim (1-\alpha_{t+1})\Big(\frac{d\log T}{1-\overline{\alpha}_{t}}\Big)^{3/2}, \\
		\Big\|\frac{\partial \big(s_{t}^{\star}(x) - \sqrt{\alpha_{t+1}}s_{t+1}^{\star}(w)\big)}{\partial x}\Big\| &\lesssim \frac{d^2(1-\alpha_{t+1})\log^2 T}{(1-\overline{\alpha}_{t})^2}.
		\end{align}
\end{subequations}
 The proof of \eqref{eqn-relation-ss} is also deferred to Section~\ref{sec:proof-additional-lems-ODE}.
%


Next, let us introduce the following vectors:
$$
\begin{aligned}
	u_t(x) & :=x-\psi_t(x),\\
	u_t^{\star}(x)& :=x-\psi^{\star}_t(x). 
	\end{aligned}
$$
For notational simplicity, we shall abbreviate  $u=u_t(x)$ and $u^{\star}=u^{\star}_t(x)$ in the following analysis.  
Akin to the calculations in \Cref{sec-proof-xt_lb}, we can obtain
\begin{align}
	&p_{\sqrt{\alpha_t}X_{t-1}}\big(\psi_t(x)\big) 
= p_{X_{t}}(x)\bigg(\frac{1-\overline{\alpha}_{t}}{\alpha_t-\overline{\alpha}_t}\bigg)^{d/2} \notag\\
	&\qquad\qquad \cdot\int_{x_0}  p_{X_0 \mymid X_{t}}(x_0 \mymid x)\exp\bigg(-\frac{(1-\alpha_{t})\big\|x - \sqrt{\overline{\alpha}_{t}}x_0\big\|_2^2}{2(\alpha_t-\overline{\alpha}_{t})(1-\overline{\alpha}_{t})}-\frac{\|u\|_2^2 -2u^{\top}\big(x - \sqrt{\overline{\alpha}_{t}}x_0\big)}{2(\alpha_t-\overline{\alpha}_{t})}\bigg)\mathrm{d} x_0. 
\end{align}
%
Similarly, by focusing mainly on the following set given $x$:
\begin{align}
	\mathcal{G} \defn \big\{x_0 : \big\|x - \sqrt{\overline{\alpha}_{t}}x_0\big\|_2 \lesssim \sqrt{d(1 - \overline{\alpha}_{t})\log T}\big\}, 
	\label{eq:defn-E-mathcal-lemma6}
\end{align}
we can derive 
\begin{align}
&\int_{x_0}  p_{X_0 \mymid X_{t}}(x_0 \mymid x)\exp\bigg(-\frac{(1-\alpha_{t})\big\|x - \sqrt{\overline{\alpha}_{t}}x_0\big\|_2^2}{2(\alpha_t-\overline{\alpha}_{t})(1-\overline{\alpha}_{t})}-\frac{\|u\|_2^2 -2u^{\top}\big(x - \sqrt{\overline{\alpha}_{t}}x_0\big)}{2(\alpha_t-\overline{\alpha}_{t})}\bigg)\mathrm{d} x_0=O\big( \exp(- c_8 d\log T ) \big) \nonumber\\
&~~+\int_{x_0 \in \mathcal{E}} p_{X_0 \mid X_t}\left(x_0 \mid x\right) \exp \left(-\frac{\left(1-\alpha_t\right)\left\|x-\sqrt{\overline{\alpha}_t} x_0\right\|_2^2}{2\left(\alpha_t-\overline{\alpha}_t\right)\left(1-\overline{\alpha}_t\right)}-\frac{\|u\|_2^2-2 u^{\top}\left(x-\sqrt{\overline{\alpha}_t} x_0\right)}{2\left(\alpha_t-\overline{\alpha}_t\right)}\right) \mathrm{d} x_0\eqqcolon \text{RHS}
\end{align}
for some numerical constant $c_8>0$. 
To further control the right-hand side above, recall that the learning rates are selected such that $\frac{1-\alpha_{t}}{1-\overline{\alpha}_{t-1}} \le \frac{4c_1\log T}{T}$ for $1 < t\leq T$ (see \eqref{eqn:properties-alpha-proof-1}).
 In view of the Taylor expansion $e^{-x}=1-x+\frac{1}{2} x^2+O\left(x^3\right)$ for $x \leq 1 / 2$, we can derive
 %
 %
\begin{align}
\label{eqn:ode-lemma-fast-rhs}
&\text{RHS}= O\big( \exp(- c_8 d\log T ) \big) + O\bigg(\frac{d^3\log^6 T}{T^3} + \frac{\sqrt{d\log^3 T}}{T}\varepsilon_{\score, t}(x) + \frac{\sqrt{d\log^3 T}}{T}\varepsilon_{\score, t+1}\big(\Phi_t(x)\big)\bigg) \notag\\
	&\quad+\int_{x_0\in \mathcal{E}} p_{X_0 \mymid X_{t}}(x_0 \mymid x)\bigg\{
1-\frac{(1-\alpha_{t})\big\|x - \sqrt{\overline{\alpha}_{t}}x_0\big\|_2^2}{2(\alpha_t-\overline{\alpha}_{t})(1-\overline{\alpha}_{t})}-\frac{\frac{(1-\alpha_{t})^2}{4(1-\overline{\alpha}_{t})^2}\|z\|_2^2 -2u^{\star\top}\big(x - \sqrt{\overline{\alpha}_{t}}x_0\big)}{2(\alpha_t-\overline{\alpha}_{t})} \notag\\
&\quad+\frac{(1-\alpha_{t})^2}{8(\alpha_t-\overline{\alpha}_{t})^2(1-\overline{\alpha}_{t})^2}\Big(\big\|x - \sqrt{\overline{\alpha}_{t}}x_0\big\|_2^2-z^{\top}\big(x - \sqrt{\overline{\alpha}_{t}}x_0\big)\Big)^2 
	\bigg\} \mathrm{d} x_0. 
\end{align}
%
%
Here, we have made use of the following facts:
	\begin{align}
		\left\| u - \frac{1-\alpha_{t}}{2(1-\overline{\alpha}_{t})}z \right\|_2 
		&= 	\left\| x-\psi_t(x) + \frac{1-\alpha_{t}}{2}s_t^{*}(x) \right\|_2\nonumber\\
		&\stackrel{\text{(i)}}{\leq} \frac{(1-\alpha_{t})^2}{4(1-\alpha_{t+1})} \big\|s_{t}^{\star}(x) - \sqrt{\alpha_{t+1}}s_{t+1}^{\star}\big(\Phi_t^{\star}(x)\big) \big\|_2 
		+ O\bigg(\frac{\log T}{T}\Big\{\varepsilon_{\score, t}(x) + \varepsilon_{\score, t+1}\big(\Phi_t(x)\big)\Big\}\bigg)\nonumber\\
		&\stackrel{\text{(ii)}}{\lesssim} (1-\alpha_{t})^2\Big(\frac{d\log T}{1-\overline{\alpha}_{t}}\Big)^{3/2} + \frac{\log T}{T}\varepsilon_{\score, t}(x) + \frac{\log T}{T}\varepsilon_{\score, t+1}\big(\Phi_t(x)\big),
		\label{eqn-relation-ut-zt}
	\end{align}
where (i) follows from \eqref{eq:psi-error} in \Cref{lem:psi-prop} and (ii) follows from \eqref{eqn-relation-ss}.

Moreover, for any $x_0\in\mathcal{E}$, using the definition of $\mathcal{E}$ (cf.~\eqref{eq:defn-E-mathcal-lemma6}) and combining it with the properties \eqref{eqn:properties-alpha-proof} of the learning rates, we reach
\begin{align*}
\frac{(1-\alpha_{t})\big\|x - \sqrt{\overline{\alpha}_{t}}x_0\big\|_2^2}{2(\alpha_t-\overline{\alpha}_{t})(1-\overline{\alpha}_{t})} &= O\Big(\frac{d\log^2 T}{T}\Big).
\end{align*}
As a result, we can derive  
\begin{align*}
&\frac{\|u\|_2^2 -2u^{\top}\big(x - \sqrt{\overline{\alpha}_{t}}x_0\big)}{2(\alpha_t-\overline{\alpha}_{t})} \\
	&\stackrel{\text{(i)}}{=} \frac{\frac{(1-\alpha_{t})^2}{4(1-\overline{\alpha}_{t})^2}\|z\|_2^2 -2u^{\star\top}\big(x - \sqrt{\overline{\alpha}_{t}}x_0\big)}{2(\alpha_t-\overline{\alpha}_{t})} + O\bigg(\frac{d^2\log^5 T}{T^3} + \frac{\sqrt{d\log^3 T}}{T}\varepsilon_{\score, t}(x) + \frac{\sqrt{d\log^3 T}}{T}\varepsilon_{\score, t+1}\big(\Phi_t(x)\big)\bigg) \\
&= \frac{z^{\top}\big(x - \sqrt{\overline{\alpha}_{t}}x_0\big)}{2(\alpha_t-\overline{\alpha}_{t})(1-\overline{\alpha}_{t})} + O\bigg(\frac{d^2\log^4 T}{T^2} + \frac{\sqrt{d\log^3 T}}{T}\varepsilon_{\score, t}(x) + \frac{\sqrt{d\log^3 T}}{T}\varepsilon_{\score, t+1}\big(\Phi_t(x)\big)\bigg) \\
&= O\bigg(\frac{d\log^2 T}{T} + \frac{\sqrt{d\log^3 T}}{T}\varepsilon_{\score, t}(x) + \frac{\sqrt{d\log^3 T}}{T}\varepsilon_{\score, t+1}\big(\Phi_t(x)\big)\bigg),
\end{align*}
where (i) follows from \eqref{eqn-relation-ut-zt} and Lemma~\ref{lem:relation-yt}. Taking the above results together and using the following basic properties regarding quantities $A_{t},\ldots, E_{t}$
(defined in \eqref{eqn:ODE-constant}) 
\begin{align*}
	\int p_{X_0 \mymid X_{t}}(x_0 \mymid x)\big\|x - \sqrt{\overline{\alpha}_{t}}x_0\big\|_2^2 \mathrm{d} x_0 &= (1-\overline{\alpha}_{t})A_t, \\
	\int p_{X_0 \mymid X_{t}}(x_0 \mymid x)\|z\|_2^2 \mathrm{d} x_0 &= (1-\overline{\alpha}_{t})B_t, \\
	\int p_{X_0 \mymid X_{t}}(x_0 \mymid x){u^{\star}}^{\top}\big(x - \sqrt{\overline{\alpha}_{t}}x_0\big) \mathrm{d} x_0 &= \frac{1-\alpha_{t}}{2}B_t + \frac{(1-\alpha_{t})^2}{8(1-\overline{\alpha}_{t})}\big[B_t - B_t^2 + D_t - E_t + A_tB_t\big], \\
	\int p_{X_0 \mymid X_{t}}(x_0 \mymid x)\Big(\big\|x - \sqrt{\overline{\alpha}_{t}}x_0\big\|_2^2-z^{\top}\big(x - \sqrt{\overline{\alpha}_{t}}x_0\big)\Big)^2 \mathrm{d} x_0 &= (1-\overline{\alpha}_{t})^2\big[C_t + D_t - 2E_t\big],
\end{align*}
we arrive at
\begin{align*}
\eqref{eqn:ode-lemma-fast-rhs}
&= 1 - \frac{(1-\alpha_{t})(A_t-B_t)}{2(\alpha_t-\overline{\alpha}_{t})}
+ \frac{(1-\alpha_{t})^2}{8(1-\overline{\alpha}_{t})^2}\big[- B_t^2 + C_t + 2D_t - 3E_t + A_tB_t\big] 
+ O\bigg(\frac{d^3\log^6 T}{T^3}\bigg).
\end{align*}
Once again, we note that integrating over the set $\mathcal{E}$ and over all possible $x_0$ only incurs a difference at most as large as $ O\big( \exp(- c_8d\log T ) \big)$.

Putting the preceding results together establishes the claimed property~\eqref{eq:xt}.

\subsubsection{Proof of property~\eqref{eq:yt}}
Consider any random vector $Y$, and recall the basic transformation
\begin{align*}
	p_{\psi_t(Y)}(\psi_t(x)) &= \mathsf{det}\Big(\frac{\partial \psi_t(x)}{\partial x}\Big)^{-1}p_{Y}(x), 
\end{align*}
where $\frac{\partial \psi_t(x)}{\partial x}$ denotes the Jacobian matrix. 
It then comes down to controlling the quantity $\operatorname{det}\left(\frac{\partial \psi_t(x)}{\partial x}\right)^{-1}$.

Towards this end, note that the determinant of a matrix obeys
\begin{align*}
\mathsf{det}(I + A + \Delta)^{-1} = 1 - \mathsf{Tr}(A) + \frac{1}{2}\big[\mathsf{Tr}(A)^2 + \|A\|_{\mathrm{F}}^2\big] + O\big(d^3\|A\|^3 + d\|\Delta\|\big),
\end{align*}
with the proviso that $d\|A\| \lesssim 1$.
 This relation taken together  with $\frac{\partial \psi_t(x)}{\partial x}=I-\frac{\partial u^{\star}}{\partial x}+\frac{\partial (u^{\star}-u)}{\partial x}$ leads to 
\begin{align}
\label{eqn:jude}
p_{\psi_t(Y)}(\psi_t(x)) 
&= \mathsf{det}\Big(\frac{\partial \psi_t(x)}{\partial x}\Big)^{-1}p_{Y}(x) \notag\\
&= \bigg\{1 + \mathsf{Tr}\Big(\frac{\partial u^{\star}}{\partial x}\Big) + \frac{1}{2}\Big[\mathsf{Tr}\Big(\frac{\partial u^{\star}}{\partial x}\Big)^2 + \Big\|\frac{\partial u^{\star}}{\partial x}\Big\|_{\mathrm{F}}^2\Big] 
+ O\Big(d^3 \Big\|\frac{\partial u^{\star}}{\partial x}\Big\| + d\Big\|\frac{\partial (u-u^{\star})}{\partial x}\Big\|\Big)\bigg\} p_{Y}(x). 
\end{align}
To further control the right-hand side of the above display, let us first make note of several identities introduced in \Cref{lem:relation-yt}:
\begin{subequations}\label{eq:H1-H4}
	\begin{align}
	J_t 
	&= I+\frac{1}{1-\overline{\alpha}_{t}}\bigg\{\Big(\int_{x_0} p_{X_0 \mymid X_{t}}(x_0 \mymid x)\big(x - \sqrt{\overline{\alpha}_{t}}x_{0}\big) \mathrm{d} x_0\Big)\Big(\int_{x_0} p_{X_0 \mymid X_{t}}(x_0 \mymid x)\big(x - \sqrt{\overline{\alpha}_{t}}x_{0}\big) \mathrm{d} x_0\Big)^{\top} \notag \\
	&\qquad\qquad-\int_{x_0} p_{X_0 \mymid X_{t}}(x_0 \mymid x)\big(x - \sqrt{\overline{\alpha}_{t}}x_{0}\big)\big(x - \sqrt{\overline{\alpha}_{t}}x_{0}\big)^{\top} \mathrm{d} x_0\bigg\};
	\end{align}
	\begin{align}
H_1&= \|z\|_2^2J_t + 2zz^{\top}J_t;\\
H_2	&= \int_{x_0} p_{X_0 \mymid X_{t}}(x_0 \mymid x)\big\|x - \sqrt{\overline{\alpha}_{t}}x_0\big\|_2^2 \mathrm{d} x_0 I + 2\int_{x_0} p_{X_0 \mymid X_{t}}(x_0 \mymid x)\big(x - \sqrt{\overline{\alpha}_{t}}x_0\big)\big(x - \sqrt{\overline{\alpha}_{t}}x_0\big)^{\top} \mathrm{d} x_0 \notag\\
	&\qquad+ \frac{1}{1-\overline{\alpha}_{t}}\Big(\Big(\int_{x_0} p_{X_0 \mymid X_{t}}(x_0 \mymid x)\big\|x - \sqrt{\overline{\alpha}_{t}}x_0\big\|_2^2\big(x - \sqrt{\overline{\alpha}_{t}}x_{0}\big)\Big)\Big(\int_{x_0} p_{X_0 \mymid X_{t}}(x_0 \mymid x)\big(x - \sqrt{\overline{\alpha}_{t}}x_{0}\big) \mathrm{d} x_0\Big)^{\top} \notag \\
	&\qquad-\int_{x_0} p_{X_0 \mymid X_{t}}(x_0 \mymid x)\big\|x - \sqrt{\overline{\alpha}_{t}}x_0\big\|_2^2\big(x - \sqrt{\overline{\alpha}_{t}}x_0\big)\big(x - \sqrt{\overline{\alpha}_{t}}x_0\big)^{\top} \mathrm{d} x_0\Big); \\
H_3	&
	= \int_{x_0} p_{X_0 \mymid X_{t}}(x_0 \mymid x)\big\|x - \sqrt{\overline{\alpha}_{t}}x_0\big\|_2^2 \mathrm{d} x_0 J_t + 2zz^{\top} \notag+ \frac{1}{1-\overline{\alpha}_{t}}\bigg(\Big(\int_{x_0} p_{X_0 \mymid X_{t}}(x_0 \mymid x)\big\|x - \sqrt{\overline{\alpha}_{t}}x_0\big\|_2^2\Big)zz^{\top}\notag \\
	&\quad~~~ -z\Big(\int_{x_0} p_{X_0 \mymid X_{t}}(x_0 \mymid x)\big\|x - \sqrt{\overline{\alpha}_{t}}x_0\big\|_2^2\big(x - \sqrt{\overline{\alpha}_{t}}x_{0}\big) \mathrm{d} x_0\Big)^{\top}\bigg) ;\\
	H_4&= \|z\|_2^2I + zz^{\top} \notag\\
	&\qquad+\int_{x_0} p_{X_0 \mymid X_{t}}(x_0 \mymid x)\big(x - \sqrt{\overline{\alpha}_{t}}x_{0}\big)\big(x - \sqrt{\overline{\alpha}_{t}}x_{0}\big)^{\top}J_t \mathrm{d} x_0 \notag\\
	&\qquad+\frac{1}{1-\overline{\alpha}_{t}}\int_{x_0} p_{X_0 \mymid X_{t}}(x_0 \mymid x)\big(z^{\top}\big(x - \sqrt{\overline{\alpha}_{t}}x_{0}\big)\big)\big(x - \sqrt{\overline{\alpha}_{t}}x_{0}\big)z^{\top} \mathrm{d} x_0 \notag\\
	&\qquad-\frac{1}{1-\overline{\alpha}_{t}}\int_{x_0} p_{X_0 \mymid X_{t}}(x_0 \mymid x)\big(z^{\top}\big(x - \sqrt{\overline{\alpha}_{t}}x_{0}\big)\big)\big(x - \sqrt{\overline{\alpha}_{t}}x_{0}\big)\big(x - \sqrt{\overline{\alpha}_{t}}x_{0}\big)^{\top} \mathrm{d} x_0.
	\end{align}
	\end{subequations}
	The above identities can be directly verified through elementary calculation involving Gaussian integration and derivatives, which are omitted here for the sake of brevity. 

Recall the definition of $u^{\star}$ that
\begin{align*}
	\frac{\partial u^{\star}}{\partial x}&=-\frac{1-\alpha_{t}}{2}J_{s^{\star}_t}(x) - \frac{(1-\alpha_{t})^2}{4(1-\alpha_{t+1})}\frac{\partial \big(s_{t}^{\star}(x) - \sqrt{\alpha_{t+1}}s_{t+1}^{\star}(w)\big)}{\partial x} \\
	&\stackrel{\text{(i)}}{=}\frac{1-\alpha_{t}}{2(1-\overline{\alpha}_{t})}J_t(x)- \frac{(1-\alpha_{t})^2}{4(1-\alpha_{t+1})}\cdot\\
	&\qquad \left(-\frac{1-\alpha_{t+1}}{2(1-\overline{\alpha}_{t})^2}J_t(x) + \frac{1-\alpha_{t+1}}{2(1-\overline{\alpha}_{t})^3}(H_1 + H_4 + H_2 - H_3) + \zeta_{J_t}\right),
\end{align*}
where $\|\zeta_{J_t} \| \lesssim d^2\frac{(1-\alpha_{t+1})^{3/2}}{(1-\overline{\alpha}_{t+1})^{5/2}}\log^2 T$. 
Here, (i) follows from \Cref{lem:relation-yt}.
Then, invoking \eqref{eq:H1-H4} and the definitions of $A_t$ to $E_t$ gives
\begin{subequations}
	\begin{align}
	\Big\|\frac{\partial u^{\star}}{\partial x}\Big\| &\lesssim \frac{d(1-\alpha_{t})\log T}{1-\overline{\alpha}_{t}},\\
	\mathsf{Tr}\Big(\frac{\partial u^{\star}}{\partial x}\Big) &= \frac{(1-\alpha_t)\big(d + B_t - A_t\big)}{2(1-\overline{\alpha}_{t})} \notag\\
	&\qquad+ \frac{(1-\alpha_t)^2}{8(1-\overline{\alpha}_{t})^2}\big(d -2A_t - A_t^2 + 3A_tB_t + 2B_t - 3B_t^2 + C_t + 4D_t - 3E_t - F_t\big), \\
	\Big\|\frac{\partial u^{\star}}{\partial x}\Big\|_{\mathrm{F}}^2 &= \frac{(1-\alpha_t)^2}{4(1-\overline{\alpha}_{t})^2}\Big\|\frac{\partial z}{\partial x}\Big\|_{\mathrm{F}}^2 + O\Big(d^5\Big(\frac{1-\alpha_{t}}{\alpha_{t}-\overline{\alpha}_{t}}\Big)^3\log^3 T\Big) \notag\\
	&= \frac{(1-\alpha_t)^2}{4(1-\overline{\alpha}_{t})^2}\big(d + 2(B_t-A_t) + B_t^2 + F_t - 2D_t\big) + O\Big(d^5\Big(\frac{1-\alpha_{t}}{\alpha_{t}-\overline{\alpha}_{t}}\Big)^3\log^3 T\Big),
	\end{align}
	as long as $d^2\big(\frac{1-\alpha_{t}}{\alpha_{t}-\overline{\alpha}_{t}}\big)\log T \lesssim 1$.  
	Here, 
	\begin{align}
	F_t(x) &\coloneqq \Big\|\frac{1}{1-\overline{\alpha}_{t}}\int_{x_0} p_{X_0 \mymid X_{t}}(x_0 \mymid x)\big(x - \sqrt{\overline{\alpha}_{t}}x_{0}\big)\big(x - \sqrt{\overline{\alpha}_{t}}x_{0}\big)^{\top} \mathrm{d} x_0\Big\|_{\mathrm{F}}^2.
	\end{align}
	Further note that $\frac{\partial (u^{\star}-u)}{\partial x}=\widetilde{\zeta_t}$, where $\widetilde{\zeta_t}$ is the residual term defined in \eqref{eq:partial-psi-error} from \Cref{lem:psi-prop}, and satisfies
	\begin{align}
		\|\widetilde{\zeta_{t}}\| = O\bigg(\frac{\log T}{T}\Big\{\varepsilon_{\score, t}(x)/\sqrt{d} + \varepsilon_{\Jacobi, t}(x) + \varepsilon_{\Jacobi, t+1}\big(\Phi_t(x)\big)\Big\}\bigg).
	\end{align}
\end{subequations}
	Substituting these results into inequality~\eqref{eqn:jude} leads to  
	\begin{align}
	p_{\psi_t(Y)}(\psi_t(x)) 
	&= p_{Y}(x)\bigg\{1 + \frac{(1-\alpha_{t})(d+B_t-A_t)}{2(1-\overline{\alpha}_{t})} + O\Big(d^6\Big(\frac{1-\alpha_{t}}{\alpha_{t}-\overline{\alpha}_{t}}\Big)^3\log^3 T\Big) \notag\\
	&\qquad+O\bigg(\frac{\sqrt{d \log ^3 T}}{T} \varepsilon_{\score, t}(x)+\frac{d \log T}{T}\Big(\varepsilon_{\Jacobi, t}(x)+\varepsilon_{\Jacobi, t+1}\big(\Phi_t(x)\big)\Big)\bigg)\\
	&\qquad+ \frac{(1-\alpha_{t})^2}{8(1-\overline{\alpha}_{t})^2}\big[d(d+2) + (4+2d)(B_t-A_t) - B_t^2 + C_t + 2D_t - 3E_t + A_tB_t\big] \bigg\}. 
	\end{align}
	%


\subsubsection{Proof of property~\eqref{eq:phit}}
Following similar arguments as in \citet[relation (58a)]{li2023towards}, we can obtain
\begin{align}
	&\frac{p_{X_{t+1}/\sqrt{\alpha_{t+1}}}\big(\phi_{t}(x)\big)}{p_{X_{t}}(x)}\notag\\ &~~=
	\Big(\frac{\alpha_{t+1}-\overline{\alpha}_{t+1}}{1-\overline{\alpha}_{t+1}}\Big)^{d/2}\cdot\int_{x_{0}}p_{X_{0}\mymid X_{t}}(x_{0}\mymid x)\cdot\notag\\
	& \qquad\qquad\qquad\exp\bigg(-\frac{(1-\alpha_{t+1}^{-1})\big\| x-\sqrt{\overline{\alpha}_{t}}x_{0}\big\|_{2}^{2}}{2(\alpha_{t+1}^{-1}-\overline{\alpha}_{t})(1-\overline{\alpha}_{t})}-\frac{\|v\|_{2}^{2}-2v^{\top}\big(x-\sqrt{\overline{\alpha}_{t}}x_{0}\big)}{2(\alpha_{t+1}^{-1}-\overline{\alpha}_{t})}\bigg)\mathrm{d}x_{0}\notag \\
	&~~=
	\Big(1-\frac{d(1-\alpha_{t+1})}{2(1-\overline{\alpha}_{t+1})}+O(\frac{d^2(1-\alpha_{t+1})^2}{(1-\overline{\alpha}_{t+1})^2}) \Big)\cdot\int_{x_{0}}p_{X_{0}\mymid X_{t}}(x_{0}\mymid x)\cdot\notag\\
	& \qquad\qquad\qquad\exp\bigg(\frac{(1-\alpha_{t+1})\big\| x-\sqrt{\overline{\alpha}_{t}}x_{0}\big\|_{2}^{2}}{2(1-\overline{\alpha}_{t+1})(1-\overline{\alpha}_{t})}-\frac{\|v\|_{2}^{2}-2v^{\top}\big(x-\sqrt{\overline{\alpha}_{t}}x_{0}\big)}{2(\alpha_{t+1}^{-1}-\overline{\alpha}_{t})}\bigg)\mathrm{d}x_{0}\notag \\
	& ~~=1-\frac{d(1-\alpha_{t+1})}{2(1-\overline{\alpha}_{t+1})}+O\bigg(c_6^2d^{2}\Big(\frac{1-\alpha_{t+1}}{1-\overline{\alpha}_{t+1}}\Big)^{2}\log^{2}T + \varepsilon_{\score, t}(x)\sqrt{c_6d\log T}\Big(\frac{1-\alpha_{t+1}}{1-\overline{\alpha}_{t+1}}\Big)\bigg)+\notag\\
	& ~~~~\quad\frac{(1-\alpha_{t+1})\big(\int_{x_{0}}p_{X_{0}\mymid X_{t}}(x_{0}\mymid x)\big\| x-\sqrt{\overline{\alpha}_{t}}x_{0}\big\|_{2}^{2}\mathrm{d}x_{0}+\alpha_{t+1}\big\|\int_{x_{0}}p_{X_{0}\mymid X_{t}}(x_{0}\mymid x)\big(x-\sqrt{\overline{\alpha}_{t}}x_{0}\big)\mathrm{d}x_{0}\big\|_{2}^{2}\big)}{2(1-\overline{\alpha}_{t+1})(1-\overline{\alpha}_{t})}.
	\end{align}
		Similarly, using the arguments as in \citet[relation (58b)]{li2023towards}, we can deduce that
		\begin{align}
			&\frac{p_{\phi_t(X_t)}\big(\phi_{t}(x)\big)}{p_{X_{t}}(x)} \notag =1-\frac{d(1-\alpha_{t+1})}{2(1-\overline{\alpha}_{t+1})}+\\
			&~~\frac{(1-\alpha_{t+1})\big(\int_{x_{0}}p_{X_{0}\mymid X_{t}}(x_{0}\mymid x)\big\| x-\sqrt{\overline{\alpha}_{t}}x_{0}\big\|_{2}^{2}\mathrm{d}x_{0}+\alpha_{t+1}\big\|\int_{x_{0}}p_{X_{0}\mymid X_{t}}(x_{0}\mymid x)\big(x-\sqrt{\overline{\alpha}_{t}}x_{0}\big)\mathrm{d}x_{0}\big\|_{2}^{2}\big)}{2(1-\overline{\alpha}_{t+1})(1-\overline{\alpha}_{t})}+\notag\\
			& ~~\quad O\bigg(c_6^2d^{2}\Big(\frac{1-\alpha_{t+1}}{1-\overline{\alpha}_{t+1}}\Big)^{2}\log^{2}T +c_6^3d^{6}\log^{3}T\Big(\frac{1-\alpha_{t+1}}{1-\overline{\alpha}_{t+1}}\Big)^{3} + (1-\alpha_{t+1})d\varepsilon_{\Jacobi, t}(x) \bigg).
			\end{align}
		Consequently, it is readily seen that
		\begin{align*}
			\frac{p_{\Phi_{t}(X_{t})}(\Phi_{t}(x))}{p_{X_{t+1}}(\Phi_{t}(x))} &=\frac{p_{\phi_t(X_t)}\big(\phi_{t}(x)\big)}{p_{X_{t+1}/\sqrt{\alpha_{t+1}}}\big(\phi_{t}(x)\big)}= \frac{p_{\phi_t(X_t)}\big(\phi_{t}(x)\big)}{p_{X_{t}}(x)}\cdot\left(\frac{p_{X_{t+1}/\sqrt{\alpha_{t+1}}}\big(\phi_{t}(x)\big)}{p_{X_{t}}(x)} \right)^{-1}\\
			&= 1+ O\bigg(\frac{d^2\log^4T}{T^2} + \frac{d^6\log^6T}{T^3} + \frac{\sqrt{d\log T}\varepsilon_{\score, t}(x)+d\varepsilon_{\Jacobi,t}(x)\log T}{T}\bigg), 
			\end{align*}
			thus completing the proof of \Cref{lem:main-ODE}.

\subsubsection{Proof of additional lemmas}
\label{sec:proof-additional-lems-ODE}
 To establish \Cref{lem:psi-prop} and \Cref{lem:relation-yt}, making use of \Cref{lem:x0}, we first summarize the following norm properties of the score function $s_t^{\star}$ and the Jacobian matrix $J_{s^{\star}_t}$ for $x$ statisfying $\log p_{X_t}(x)\geq -c_6d\log T$:
\begin{subequations}\label{eq:prop-norm}
	\begin{align}
		&\left\|s^{*}_t\left(x_t\right)\right\|_2 \leq \frac{1}{(1-\bar{\alpha}_{t})}\mathbb{E}\left[\left\|X_t-\sqrt{\bar{\alpha}_t} X_0\right\|_2 \mid X_t=x_t\right] \lesssim \sqrt{\frac{d\log T}{1-\bar{\alpha}_{t}}},\label{eq:prop-norm-s}\\
		&\left\|J_{s^{\star}_t}\left(x\right)\right\| \lesssim \frac{1}{(1-\bar{\alpha}_{t})^2} \mathbb{E}\left[\left\|x-\sqrt{\bar{\alpha}_{t}} X_0\right\|_2^2 \mid X_{t}=x\right] \asymp d \log T,\label{eq:prop-norm-J}\\
		&\|\nabla_x u^{\top}J_{s^{\star}_t}(x)u\|_2  \lesssim d^{3/2}\log^{3/2} T,\quad \text{ for } u \in \mathbb{S}^{d-1}.\label{eq:prop-norm-partJ}
	\end{align}
\end{subequations}
The detailed calculation for the second property is presented in \citet[Lemma 8]{li2023towards}. The third property follows a rationale akin to that for \(J_{s^{\star}_t} = \frac{\partial s^{\star}_t(x)}{\partial x}\), and is therefore omitted here for the sake of brevity.
\paragraph{Proof of \Cref{lem:psi-prop}.}
To start with, it follows from the definitions of $\psi_t$ and $\psi_t^{\star}$ (cf.~\eqref{defn:phit-x}) that
\begin{align}
\psi_t(x)-\psi_t^{\star}(x) 
&=  \left(\frac{1-\alpha_{t}}{2}+\frac{(1-\alpha_{t})^2}{4(1-\alpha_{t+1})}\right) \big(s_{t}(x)- s_{t}^{\star}(x)\big) \notag\\
&\qquad - \frac{(1-\alpha_{t})^2 \sqrt{\alpha_{t+1}}}{4(1-\alpha_{t+1})}\Big( s_{t+1}\big(\Phi_t(x)\big)-s^{\star}_{t+1}\big(\Phi_t(x)\big)
	+s^{\star}_{t+1}\big(\Phi_t(x)\big)-s_{t+1}^{\star}\big(\Phi_t^{\star}(x)\big)\Big) . \notag
\end{align}
Armed with this relation, to derive \eqref{eq:psi-error}, we only need to control the following term:   
\begin{align*}
(1-\alpha_{t+1})\Big\|s_{t+1}^{\star}\big(\Phi_t(x)\big) - s_{t+1}^{\star}\big(\Phi_t^{\star}(x)\big)\Big\|_2 
	&\stackrel{(i)}{\lesssim} \frac{\log T}{T}\Big\|s_{t+1}^{\star}\big(\Phi_t(x)\big) - s_{t+1}^{\star}\big(\Phi_t^{\star}(x)\big)\Big\|_2 \\
	&\stackrel{(ii)}{\lesssim} \frac{\log T}{T} d (\log T) \big\|\Phi_t(x) - \Phi_t^{\star}(x) \big\|_2 \\
&\stackrel{(iii)}{\lesssim}\frac{ d\log^3 T}{T^2}\varepsilon_{\score, t}(x). 
\end{align*}
Here, (i) follows directly from the choice of learning rate in \eqref{eqn:properties-alpha-proof};  (ii) holds by observing that both $\Phi_t(x)$ and $\Phi_t^{\star}(x)$ remain within the typical set with $\log p_{X_{t+1}}\Phi_t(x), \log p_{X_{t+1}}\Phi_t^{\star}(x)\geq  -c_6d\log T$ due to  \eqref{eq:xt_lb}, and then invoking \eqref{eq:prop-norm-J}; (iii) is due to the definition of $\varepsilon_{\score, t}(x)$ (cf. \eqref{eq:pointwise-epsilon-score-J}). Combining the above bound with \eqref{eq:pointwise-epsilon-score-J} and \eqref{eqn:properties-alpha-proof}, we arrive at   
\begin{align*}
	\|\psi_t(x)-\psi_t^{\star}(x)\| 
	&\lesssim   \frac{\log T}{T}\varepsilon_{\score, t}(x) + \frac{\log T}{T}\varepsilon_{\score, t+1}\big(\Phi(x)\big)+
	\frac{ d\log^3 T}{T^2}\varepsilon_{\score, t}(x)  \notag\\
	&\lesssim \frac{\log T}{T}\Big(\varepsilon_{\score, t}(x) +\varepsilon_{\score, t+1}\big(\Phi(x)\big)\Big). 
	\end{align*}

%
%
For \eqref{eq:partial-psi-error}, by direct calculation, we have 
\begin{align*}
	&\frac{\partial \psi_t(x, \Phi_t(x))}{\partial x} -\frac{\partial \psi^{\star}_t(x, \Phi^{\star}_t(x))}{\partial x}= \Big(\frac{1-\alpha_{t}}{2}+\frac{(1-\alpha_{t})^2}{4(1-\alpha_{t+1})}\Big)(J_{s_t}(x)-J_{s^{\star}_t}(x))\\
	& ~~+ \frac{\sqrt{\alpha_{t+1}}(1-\alpha_{t})^2}{4(1-\alpha_{t+1})}\bigg(J_{s_{t+1}}\big(\Phi_t(x)\big)\Big(I-\frac{1-\alpha_{t+1}}{2}J_{s_t}(x)\Big)-J_{s^{\star}_{t+1}}\big(\Phi^{\star}_t(x)\big)\Big(I-\frac{1-\alpha_{t+1}}{2}J_{s^{\star}_t}(x)\Big)\bigg). 
	\end{align*}
	The term in the first line can be directly bounded by the definition of $\varepsilon_{\Jacobi, t}(x)$ and \eqref{eqn:properties-alpha-proof} as follows
	\begin{align}
		\bigg\|\Big(\frac{1-\alpha_{t}}{2}+\frac{(1-\alpha_{t})^2}{4(1-\alpha_{t+1})}\Big)(J_{s_t}(x)-J_{s^{\star}_t}(x))\bigg\|\lesssim \frac{d\log T}{T}\varepsilon_{\Jacobi, t}(x).\label{eq:partial-psi-error-I1}
	\end{align}
	Turning to the second line, we have 
	\begin{align}
			&J_{s_{t+1}}\big(\Phi_t(x)\big)\Big(I-\frac{1-\alpha_{t+1}}{2}J_{s_t}(x)\Big)-J_{s^{\star}_{t+1}}\big(\Phi^{\star}_t(x)\big)\Big(I-\frac{1-\alpha_{t+1}}{2}J_{s^{\star}_t}(x)\Big)\notag
			\\
			&=\Big(J_{s_{t+1}}\big(\Phi_t(x)\big)-J_{s^{\star}_{t+1}}\big(\Phi^{\star}_t(x)\big)\Big)\Big(I-\frac{1-\alpha_{t+1}}{2}J_{s^{\star}_t}(x)\Big)+ \frac{1-\alpha_{t+1}}{2} J_{s_{t+1}}\big(\Phi_t(x)\big)\Big(J_{s^{\star}_{t}}\big(x\big)-J_{s_{t}}\big(x\big)\Big).\label{eq:partial-psi-error-I2}
	\end{align}
To proceed, we further observe that
\begin{align*}
	\big\|J_{s_{t+1}^{\star}}\big(\Phi_t(x)\big) - J_{s_{t+1}^{\star}}\big(\Phi_t^{\star}(x)\big)\big\| 
	\lesssim 
	( d\log T)^{\frac32}\|\Phi_t(x) - \Phi_t^{\star}(x)\|_2 
	\lesssim \frac{d^{\frac32}\log^{\frac52} T}{T}\varepsilon_{\score, t}(x), 
	\end{align*}
	which is obtained by invoking \eqref{eq:prop-norm-partJ}, \eqref{eqn:properties-alpha-proof}  and \eqref{eq:pointwise-epsilon-score-J}. This bound together with \eqref{eq:prop-norm-J} allows us to control the first term in \eqref{eq:partial-psi-error-I2} as follows
\begin{align}
	&\Bigg\|\Big(J_{s_{t+1}}\big(\Phi_t(x)\big)-J_{s^{\star}_{t+1}}\big(\Phi^{\star}_t(x)\big)\Big)\Big(I-\frac{1-\alpha_{t+1}}{2}J_{s^{\star}_t}(x)\Big)\Bigg\|\notag\\
	&~~\lesssim \Bigg\|\Big(J_{s_{t+1}}\big(\Phi_t(x)\big)-J_{s^{\star}_{t+1}}\big(\Phi_t(x)\big)\Big)\Bigg\|+\Bigg\|\Big(J_{s^{\star}_{t+1}}\big(\Phi_t(x)\big)-J_{s^{\star}_{t+1}}\big(\Phi^{\star}_t(x)\big)\Big)\Bigg\|\notag\\
	&~~\lesssim \frac{d\log T}{T}\varepsilon_{\Jacobi, t+1}\big(\Phi_t(x)\big)+\frac{d^{\frac32}\log^{\frac52} T}{T}\varepsilon_{\score, t}(x).\label{eq:control-I2-1}
\end{align}
The second term in \eqref{eq:partial-psi-error-I2} can be controlled by \eqref{eq:prop-norm-J} and \eqref{eq:pointwise-epsilon-score-J} as follows
\begin{align}
	\bigg\|\frac{1-\alpha_{t+1}}{2} J_{s_{t+1}}\big(\Phi_t(x)\big)\Big(J_{s^{\star}_{t}}\big(x\big)-J_{s_{t}}\big(x\big)\Big)\bigg\|\lesssim \frac{\log T}{T}\Big(\varepsilon_{\Jacobi, t+1}\big(\Phi_t(x)\big)+d\log T\Big)\varepsilon_{\Jacobi, t}(x).\label{eq:control-I2-2}
\end{align}
Substituting \eqref{eq:control-I2-1} and \eqref{eq:control-I2-2} into \eqref{eq:partial-psi-error-I2}, together with \eqref{eq:partial-psi-error-I1}, we obtain
\begin{align*}
	\Big\|\frac{\partial \psi_t(x, \Phi_t(x))}{\partial x} -\frac{\partial \psi^{\star}_t(x, \Phi^{\star}_t(x))}{\partial x}\Big\|& \lesssim \frac{d\log T}{T}\varepsilon_{\Jacobi, t}(x)+\frac{d\log  T}{T}\varepsilon_{\score, t+1}\big(\Phi_t(x)\big)+\frac{d^{\frac32}\log^{\frac72} T}{T^2}\varepsilon_{\score, t}(x)\\
	&~~+\frac{\log^2 T}{T^2}\Big(\varepsilon_{\Jacobi, t+1}\big(\Phi_t(x)\big)+d\log T\Big)\varepsilon_{\Jacobi, t}(x)\\
	& \lesssim \frac{d\log T}{T}\varepsilon_{\Jacobi, t}(x)+\frac{d\log  T}{T}\varepsilon_{\score, t+1}\big(\Phi_t(x)\big)+\frac{ \log T}{d^{\frac12}T}\varepsilon_{\score, t}(x) 
\end{align*}
where the last inequality invokes conditions for $d$ and $T$ in  \eqref{eq:assumption-T-score-Jacob}. 
	\hfill\qedsymbol{}

\paragraph{Proof of \Cref{lem:relation-yt}.} 
To begin with, applying \eqref{eq:prop-norm-s} to $w$ leads to 
\begin{align}
	\left\| \frac{1}{\sqrt{\alpha_{t+1}}} w - x\right\|_2 = (1-\alpha_{t+1})\|s_t^{*}(x)\|_2\lesssim (1-\alpha_{t+1})\sqrt{\frac{d\log T}{1-\overline{\alpha}_{t}}}.
\end{align}
Then denoting $\widehat{w} \coloneqq w/\sqrt{\alpha_{t+1}}$, we can apply \Cref{lem:y-t} to $(\widehat{w},x)$ to obtain
\begin{align*}
K_1 &\coloneqq \int_{x_0} p_{X_0 \mymid X_{t+1}}(x_0 \mymid w)\big(w/\sqrt{\alpha_{t+1}} - \sqrt{\overline{\alpha}_{t}}x_0\big) \mathrm{d} x_0 \\
&= \widehat{w} - x + \int_{x_0} p_{X_0 \mymid X_{t}}(x_0 \mymid x)
\bigg\{1+\frac{(1 - \alpha_{t+1})\big\|x - \sqrt{\overline{\alpha}_{t}}x_0\big\|_2^2}{2\alpha_{t+1}(1-\overline{\alpha}_{t})^2} 
-\frac{(x - \sqrt{\overline{\alpha}_{t}}x_0)^{\top}(\widehat{w} - x)}{1-\overline{\alpha}_{t}} \\
&\qquad- \int_{x_0} \bigg(\frac{(1 - \alpha_{t+1})\big\|x - \sqrt{\overline{\alpha}_{t}}x_0\big\|_2^2}{2\alpha_{t+1}(1-\overline{\alpha}_{t})^2} - \frac{(x - \sqrt{\overline{\alpha}_{t}}x_0)^{\top}(\widehat{w} - x)}{1-\overline{\alpha}_{t}}\bigg)p_{X_0 \mymid X_t}(x_0 \mymid x)  \mathrm{d} x_0 \\
&\qquad+ O\Big(d\Big(\frac{1-\alpha_{t+1}}{1-\overline{\alpha}_{t}}\Big)^{3/2}\log T\Big)\bigg\}
\big(x - \sqrt{\overline{\alpha}_{t}}x_0\big) \mathrm{d} x_0. 
\end{align*}
Plugging $\widehat{w}-x=-\frac{1-\alpha_{t+1}}{2} s_t^{\star}(x)=\frac{1-\alpha_{t+1}}{2(1-\overline{\alpha}_{t})}z$ into the above equation and combining with the expression of $s_t^*(x)$ in \eqref{eq:st-MMSE-expression}, we can obtain
\begin{align*}
	K_1 
	&= \Big(1 + \frac{1-\alpha_{t+1}}{2(1-\overline{\alpha}_{t})} + \frac{1-\alpha_{t+1}}{2(1-\overline{\alpha}_{t})^2}\|z\|_2^2\Big)z + \frac{1-\alpha_{t+1}}{2(1-\overline{\alpha}_{t})^2}\int_{x_0} p_{X_0 \mymid X_{t}}(x_0 \mymid x)
	\Big\{\big\|x - \sqrt{\overline{\alpha}_{t}}x_0\big\|_2^2 \\
	&\quad
	-(x - \sqrt{\overline{\alpha}_{t}}x_0)^{\top}z 
	- \int_{x_0} \big\|x - \sqrt{\overline{\alpha}_{t}}x_0\big\|_2^2 p_{X_0 \mymid X_t}(x_0 \mymid x)  \mathrm{d} x_0\Big\}
		\big(x - \sqrt{\overline{\alpha}_{t}}x_0\big) \mathrm{d} x_0 + \zeta_{K_1},
	\end{align*}
where the residual term obeys
\[
	\|\zeta_{K_1}\|_2 \lesssim \frac{\big(d(1-\alpha_{t+1})\log T\big)^{3/2}}{1 - \overline{\alpha}_{t}}. 
\]

Then we can immediately establish the first claim \eqref{eq:lem-relation-yt-1st} by recognizing that
\begin{align*}
s_{t}^{\star}(x) - \sqrt{\alpha_{t+1}}s_{t+1}^{\star}(w) &= \frac{1}{\alpha_{t+1}^{-1} - \overline{\alpha}_{t}}K_1 - \frac{1}{1 - \overline{\alpha}_t}z.
\end{align*}

Similarly, one sees that
\begin{align*}
&K_2\\& \coloneqq \int_{x_0} p_{X_0 \mymid X_{t+1}}(x_0 \mymid w)\big(w/\sqrt{\alpha_{t+1}} - \sqrt{\overline{\alpha}_{t}}x_0\big)\big(w/\sqrt{\alpha_{t+1}} - \sqrt{\overline{\alpha}_{t}}x_0\big)^{\top} \mathrm{d} x_0 \\
&= \int_{x_0} p_{X_0 \mymid X_{t+1}/\sqrt{\alpha_{t+1}}}(x_0 \mymid \widehat{w})\Big[\big(\widehat{w} - x\big)\big(\widehat{w} - \sqrt{\overline{\alpha}_{t}}x_0\big)^{\top} + \big(\widehat{w} - \sqrt{\overline{\alpha}_{t}}x_0\big)\big(\widehat{w} - x\big)^{\top} - \big(\widehat{w} - x\big)\big(\widehat{w} - x\big)^{\top}\Big] \mathrm{d} x_0 \\
&\quad+ \int_{x_0} p_{X_0 \mymid X_{t}}(x_0 \mymid x)
\Bigg\{1+\frac{(1 - \alpha_{t+1})\big\|x - \sqrt{\overline{\alpha}_{t}}x_0\big\|_2^2}{2\alpha_{t+1}(1-\overline{\alpha}_{t})^2} 
-\frac{(x - \sqrt{\overline{\alpha}_{t}}x_0)^{\top}(\widehat{w} - x)}{1-\overline{\alpha}_{t}} + O\Big(d\Big(\frac{1-\alpha_{t+1}}{1-\overline{\alpha}_{t}}\Big)^{3/2}\log T\Big) \\
&\quad
~~~~~~- \int_{x_0} \bigg(\frac{(1 - \alpha_{t+1})\big\|x - \sqrt{\overline{\alpha}_{t}}x_0\big\|_2^2}{2\alpha_{t+1}(1-\overline{\alpha}_{t})^2} - \frac{(x - \sqrt{\overline{\alpha}_{t}}x_0)^{\top}(\widehat{w} - x)}{1-\overline{\alpha}_{t}}\bigg)p_{X_0 \mymid X_t}(x_0 \mymid x)  \mathrm{d} x_0\Bigg\}\\
&~~~~~~~~~~~\cdot\big(x - \sqrt{\overline{\alpha}_{t}}x_0\big)\big(x - \sqrt{\overline{\alpha}_{t}}x_0\big)^{\top} \mathrm{d} x_0 \\
&= \int_{x_0} p_{X_0 \mymid X_{t}}(x_0 \mymid x)\big(x - \sqrt{\overline{\alpha}_{t}}x_0\big)\big(x - \sqrt{\overline{\alpha}_{t}}x_0\big)^{\top} \mathrm{d} x_0 \\
&\quad+ \frac{1-\alpha_{t+1}}{1-\overline{\alpha}_{t}}zz^{\top} +\frac{1-\alpha_{t+1}}{2(1-\overline{\alpha}_{t})^2}\int_{x_0} p_{X_0 \mymid X_{t}}(x_0 \mymid x)
\Bigg\{\big\|x - \sqrt{\overline{\alpha}_{t}}x_0\big\|_2^2 -(x - \sqrt{\overline{\alpha}_{t}}x_0)^{\top}z \\
&\quad
~~~~-\int_{x_0} \big\|x - \sqrt{\overline{\alpha}_{t}}x_0\big\|_2^2 p_{X_0 \mymid X_t}(x_0 \mymid x)  \mathrm{d} x_0
- \|z\|_2^2\Bigg\}
\big(x - \sqrt{\overline{\alpha}_{t}}x_0\big)\big(x - \sqrt{\overline{\alpha}_{t}}x_0\big)^{\top} \mathrm{d} x_0 
	+ \zeta_{K_2},
\end{align*}
where the residual term $\zeta_{K_2}$ satisfies
\[
	\|\zeta_{K_2}\|\lesssim d^2\frac{(1-\alpha_{t+1})^{3/2}}{(1-\overline{\alpha}_{t})^{1/2}}\log^2 T.
\]
Then the second claim \eqref{eq:lem-relation-yt-2nd} immediately follows by recognizing
%
\begin{align*}
 & \frac{\partial\big(s_{t}^{\star}(x)-\sqrt{\alpha_{t+1}}s_{t+1}^{\star}(w)\big)}{\partial x}\notag\\
 & =\frac{\sqrt{\alpha_{t+1}}}{1-\overline{\alpha}_{t+1}}J_{t+1}(w)\frac{\partial w}{\partial x}-\frac{1}{1-\overline{\alpha}_{t}}J_{t}(x)\notag\\
 & =\frac{1}{\alpha_{t+1}^{-1}-\overline{\alpha}_{t}}\Big(I+\frac{1}{\alpha_{t+1}^{-1}-\overline{\alpha}_{t}}\Big(K_{1}K_{1}^{\top}-K_{2}\Big)\Big)\Big(I+\frac{1-\alpha_{t+1}}{2(1-\overline{\alpha}_{t})}J_{t}(x)
 \Big)-\frac{1}{1-\overline{\alpha}_{t}}J_{t}(x)\notag\\
 & =-\frac{1-\alpha_{t+1}}{2(1-\overline{\alpha}_{t})^{2}}J_{t}(x)+\frac{1-\alpha_{t+1}}{2(1-\overline{\alpha}_{t})^{3}}(H_{1}+H_{4}+H_{2}-H_{3})+{\zeta}_{J_t},
\end{align*}
where the residual term ${\zeta}_{J_t}$ satisfies
\[
\big\|{\zeta}_{J_t}\big\|\lesssim d^{2}\frac{(1-\alpha_{t+1})^{3/2}}{(1-\overline{\alpha}_{t+1})^{5/2}}\log^{2}T.
\]\hfill\qedsymbol{}
\paragraph{Proof of properties \eqref{eqn-relation-ss}.}  To prove these properties, we first note that \Cref{lem:x0} implies that 
\begin{subequations}
	\begin{align}
		& \left\|z\right\|_2 \leq \mathbb{E}\left[\left\|X_t-\sqrt{\bar{\alpha}_t} X_0\right\|_2 \mid X_t=x\right] \lesssim \sqrt{d\left(1-\bar{\alpha}_t\right) \log T},  \\
		& \left\|\int_{x_0} p_{X_0 \mid X_t}\left(x_0 \mid x\right)\left(x-\sqrt{\overline{\alpha_t}} x_0\right)\left(X_t-\sqrt{\overline{\alpha_t}} x_0\right)^{\top} z \mathrm{~d} x_0\right\|_2\notag\\
		&~
		\lesssim \left\|z\right\|_2\mathbb{E}\left[\left\|X_t-\sqrt{\bar{\alpha}_t} X_0\right\|_2^2 \mid X_t=x\right]\lesssim  \left(\frac{d \log T}{1-\bar{\alpha}_t}\right)^{3 / 2},\\
		&\left\|\int_{x_0} p_{X_0 \mid X_t}\left(x_0 \mid x\right)\left\|X_t-\sqrt{\bar{\alpha}_t} x_0\right\|_2^2\left(x-\sqrt{\bar{\alpha}_t} x_0-z\right) \mathrm{d} x_0\right\|_2\notag\\
		&~\leq \left\|\mathbb{E}\left[\left\|X_t-\sqrt{\bar{\alpha}_t} X_0\right\|_2^3 \mid X_t=x\right]\right\|_2+\left\|z\right\|_2\left\|\mathbb{E}\left[\left\|X_t-\sqrt{\bar{\alpha}_t} X_0\right\|_2^2 \mid X_t=x\right]\right\|_2\notag\\
		&~\lesssim \left(\frac{d \log T}{1-\bar{\alpha}_t}\right)^{3 / 2}.
		\end{align}
\end{subequations}
Substituting the above bounds into \eqref{eq:lem-relation-yt-1st} yields the first claim of \eqref{eqn-relation-ss}. Similarly, the second claim follows by applying \Cref{lem:x0} and utilizing \eqref{eq:H1-H4} in the context of   \eqref{eq:lem-relation-yt-2nd}. \hfill\qedsymbol{}

\subsection{Proof of \Cref{lem:density-ratio-tau} }
\label{sec:proof-lem:density-ratio-tau}

To begin with, it follows 
from the definition \eqref{eq:defn-tao-i} of $\tau(y_T)$ that 
$$
-\log q_t(y_t)\leq c_{\tau} d\log T,\qquad \forall t<\tau(y_T).
$$
Our proof is mainly built upon Lemma~\ref{lem:main-ODE}.   
Specifically, 
combining Lemma~\ref{lem:x0}, \eqref{eqn:properties-alpha-proof} and the definition \eqref{eq:defn-tao-i} of $\tau(y_T)$ gives 
%
\begin{subequations}
	\label{eqn:ODE-constant-bounds}
	\begin{align}
	\big|B_{t}\big|\leq\big|A_{t}\big| & \lesssim\frac{1}{1-\overline{\alpha}_{t}}\cdot d(1-\overline{\alpha}_{t})\log T\asymp d\log T\\
	\big|C_{t}\big| & \lesssim\frac{1}{(1-\overline{\alpha}_{t})^{2}}d^{2}(1-\overline{\alpha}_{t})^{2}\log^{2}T\asymp d^{2}\log^{2}T\\
	\big|D_{t}\big| & \leq\frac{\|g_{t}(x)\|_{2}^{2}}{(1-\overline{\alpha}_{t})^{2}}\int p_{X_{0}\mymid X_{t}}(x_{0}\mymid x)\big\| x-\sqrt{\overline{\alpha}_{t}}x_{0}\big\|_{2}^{2}\mathrm{d}x_{0}\lesssim d^{2}\log^{2}T\\
	\big|E_{t}\big| & \leq\frac{\|g_{t}(x)\|_{2}^{2}}{(1-\overline{\alpha}_{t})^{2}}\int p_{X_{0}\mymid X_{t}}(x_{0}\mymid x)\big\| x-\sqrt{\overline{\alpha}_{t}}x_{0}\big\|_{2}^{3}\mathrm{d}x_{0}\lesssim d^{2}\log^{2}T
	\end{align}
	\end{subequations}
for all $t<\tau(y_T)$. 	
As a consequence, the properties~\eqref{eq:xt} and~\eqref{eq:yt} in Lemma~\ref{lem:main-ODE} tell us that
	\begin{align*}
	 & \frac{p_{\sqrt{\alpha_{t}}Y_{t-1}}\big(\psi_{t}(y_{t})\big)}{p_{Y_{t}}(y_{t})}\bigg(\frac{p_{\sqrt{\alpha_{t}}X_{t-1}}\big(\psi_{t}(y_{t})\big)}{p_{X_{t}}(y_{t})}\bigg)^{-1}=\frac{p_{\psi_{t}(Y_{t})}\big(\psi_{t}(y_{t})\big)}{p_{Y_{t}}(y_{t})}\bigg(\frac{p_{\sqrt{\alpha_{t}}X_{t-1}}\big(\psi_{t}(y_{t})\big)}{p_{X_{t}}(y_{t})}\bigg)^{-1}\\
	 & \qquad=1+O\Bigg(\frac{d^{6}\log^{6}T}{T^{3}}+\frac{\big(\varepsilon_{\score,t}(y_{t})+\varepsilon_{\score,t}(\Phi_t(y_t)) \big)\sqrt{d\log^{3}T}}{T}+\frac{d\log T\big(\varepsilon_{\Jacobi,t}(y_{t})+\varepsilon_{\Jacobi,t}(\Phi_t(y_{t}))\big)}{T}\Bigg)
	\end{align*}
	for all $t<\tau(y_T)$. 
	Give that $y_{t-1}=\frac{1}{\sqrt{\alpha_t}} \psi_t(y_t)$, one can make use of the relation \eqref{eq:recursion} and derive
	\begin{align}
	&\frac{p_{t-1}(y_{t-1})}{q_{t-1}(y_{t-1})}  =\frac{p_{t}(y_{t})}{q_{t}(y_{t})} \cdot \label{eq:relation-ratioy-one-step}\\
	&~~ \left\{ 1+O\Bigg(\frac{d^{6}\log^{6}T}{T^{3}}+\frac{\big(\varepsilon_{\score,t}(y_{t})+\varepsilon_{\score,t}(\Phi_t(y_t)) \big)\sqrt{d\log^{3}T}}{T}+\frac{d\log T\big(\varepsilon_{\Jacobi,t}(y_{t})+\varepsilon_{\Jacobi,t}(\Phi_t(y_{t}))\big)}{T}\Bigg) \right\} \notag
	\end{align}
	for any $t<\tau(y_T)$. 
	If we employ the shorthand notation $\tau=\tau(y_{T})$, then it can be seen that
	\begin{subequations}
		\label{eq:pt-qt-equiv-ODE-St-temp}
	\begin{align}
	\frac{q_{1}(y_{1})}{p_{1}(y_{1})} & =\left\{ 1+O\Bigg(\frac{d^{6}\log^{6}T}{T^{2}}+S_{\tau-1}(y_{\tau-1})\Bigg)\right\} \frac{q_{\tau-1}(y_{\tau-1})}{p_{\tau-1}(y_{\tau-1})}\notag\\
	 & \in\left[\frac{p_{\tau-1}(y_{\tau-1})}{2q_{\tau-1}(y_{\tau-1})},\frac{2p_{\tau-1}(y_{\tau-1})}{q_{\tau-1}(y_{\tau-1})}\right].	
		\label{eq:pt-qt-equiv-ODE-St-taui-temp}
	\end{align}
	Repeating this argument also yields 
	\begin{align}
		\frac{q_{t}(y_{t})}{2p_{t}(y_{t})} \leq \frac{q_{1}(y_{1})}{p_{1}(y_{1})} \leq  \frac{2q_{t}(y_{t})}{p_{t}(y_{t})}, \qquad \forall t < \tau. 
		\label{eq:pt-qt-equiv-ODE-St-k-temp}
	\end{align}
	\end{subequations}

\section{Analysis for the accelerated DDPM sampler (proof of Theorem~\ref{thm:main-SDE})}

In this section, we turn to the accelerated stochastic sampler and present the proof of Theorem~\ref{thm:main-SDE}. 


\subsection{Main steps of the proof}

\paragraph{Preparation.} First, we find it convenient to introduce the following mapping
\begin{align}\label{eq-mu-hat}
	\widehat{\mu}_t^{\star}(x_t) = \frac{1}{\sqrt{\alpha_t}}\big(x_t + (1-\alpha_t)s_{t}^{\star}(x_t)\big).
\end{align}
For any $t$,   introduce the following auxiliary sequences: $H_T \sim \mathcal{N}(0, I_d)$, and
\begin{align}
	H_{t-1} &= \frac{1}{\sqrt{\alpha_{t}}}\bigg\{H_t + \sqrt{\frac{1 - \alpha_t}{2}} Z_t + (1-\alpha_{t})s_t^{\star}(H_t) - \frac{(1-\alpha_{t})^{3/2}}{\sqrt{2}(1-\overline{\alpha}_{t})}J_t(H_t)Z_t + \sqrt{\frac{1 - \alpha_t}{2}} Z_t^+\bigg\}\\
&= \widehat{\mu}_t^{\star}(H_{t})+ \sqrt{\frac{1 - \alpha_t}{2\alpha_{t}}}\Big( Z_t - \frac{1-\alpha_{t}}{1-\overline{\alpha}_{t}}J_t(H_t)Z_t +  Z_t^+\Big)
\end{align}
 for $t= T,\cdots,1$. 
We shall also adopt the following notation throughout for notational convenience:  
 \begin{equation}
	 \widehat{x}_t \coloneqq \frac{1}{\sqrt{\alpha_t}} x_t. 
	 \label{eq:defn-xt-proof-thm3}
 \end{equation}

\paragraph{Step 1: decomposing the KL divergence of interest.} Applying Pinsker's inequality and repeating the arguments as in~\citet[Section~5.3]{li2023towards} 
lead to the following elementary decompositions:
\begin{align}
	\mathsf{TV}(p_{X_{1}},p_{Y_{1}}) &\leq\sqrt{\frac{1}{2}\mathsf{KL}(p_{X_{1}}\parallel p_{Y_{1}})},\label{eq:Pinsker-thm3} \\
\mathsf{KL}(p_{X_{1}}\parallel p_{Y_{1}}) 
 & \leq \mathsf{KL}(p_{X_{T}}\parallel p_{Y_{T}})+
	\sum_{t=2}^{T}\mathop{\mathbb{E}}_{x\sim q_{t}}\Big[\mathsf{KL}\Big(p_{X_{t-1}\mymid X_{t}}(\cdot\mid x) \,\|\, p_{Y_{t-1}\mymid Y_{t}}(\cdot\mid x)\Big)\Big].  
	\label{eqn:kl-decomp}
\end{align} 
In particular, the term $\mathsf{KL}(p_{X_{T}}\parallel p_{Y_{T}})$ can be readily bounded by \Cref{lem:KL-T} as follows:  
 \[
	 \mathsf{KL}(p_{X_{T}}\parallel p_{Y_{T}}) \lesssim \frac{1}{T^{200}}. 
\]
As a result, it suffices to bound $\mathsf{KL}\big(p_{X_{t-1}\mymid X_{t}}(\cdot\mid x) \,\|\, p_{Y_{t-1}\mymid Y_{t}}(\cdot\mid x)\big)$ 
for each $1<t\leq T$ separately, which we shall accomplish next.

\paragraph{Step 2: bounding the conditional distributions $p_{X_{t-1}\mymid X_{t}}$ and $p_{H_{t-1}\mymid H_{t}}$.} 
We now compare the two conditional distributions $p_{X_{t-1}\mymid X_{t}}$ and $p_{H_{t-1}\mymid H_{t}}$. 

Towards this end, let us first introduce the  set below: 
\begin{align}
\label{eqn:eset}
	\mathcal{E} \defn \bigg\{(x_t, x_{t-1}) \mymid -\log p_{X_t}(x_t) \leq \frac{1}{2}c_6 d\log T, ~\|x_{t-1} - \widehat{x}_t\|_2 \leq c_3 \sqrt{d(1 - \alpha_t)\log T} \bigg\}
\end{align}
with $\widehat{x}_t$ defined in \eqref{eq:defn-xt-proof-thm3}, 
and we would like to evaluate both $p_{H_{t-1}\mymid H_{t}}$ and $p_{X_{t-1}\mymid X_{t}}$ over the set $\mathcal{E}$. 
Regarding $p_{H_{t-1}\mymid H_{t}}$, we have the following lemma. 
\begin{lemma}\label{lem:Ht}
	For every $(x_t, x_{t-1}) \in \mathcal{E}$ as defined in \eqref{eqn:eset}, we have  
	\begin{align}\label{eq:Ht}
		p_{H_{t-1} \mymid H_t}(x_{t-1}\mymid x_t) &\propto \exp\bigg\{-\frac{\alpha_t}{2(1-\alpha_t)}\Big\|\Big(I - \frac{1-\alpha_t}{2(1-\overline{\alpha}_{t})}J_{t}(x_t)\Big)^{-1}\big(x_{t-1} - \widehat{\mu}_t^{\star}(x_t)\big)\Big\|_2^2 + O\Big(\frac{d^3\log^{5} T}{T^{2}}\Big)\bigg\}. 
		\end{align}
\end{lemma}
%
\noindent 
Turning to $p_{X_{t-1}\mymid X_{t}}$ over the set $\mathcal{E}$, we can invoke \citet[Lemma 12]{li2023towards} to derive the following result. 
\begin{lemma}
	\label{lem:sde-R} 
	There exists some large enough numerical constant $c_{\zeta}>0$ such that: 
	for every $(x_t, x_{t-1}) \in \mathcal{E}$,  
		\begin{align}
			&p_{X_{t-1}\mymid X_{t}}(x_{t-1}\mymid x_{t})  =\frac{1}{\big(2\pi\frac{1-\alpha_{t}}{\alpha_{t}}\big)^{d/2}\big|\det\big(I-\frac{1-\alpha_{t}}{2(1-\overline{\alpha}_{t})}J_{t}(x_{t})\big)\big|} \notag\\
	 & \qquad\qquad\qquad\cdot\exp\bigg(-\frac{\alpha_{t}}{2(1-\alpha_{t})}\bigg\|\bigg(I-\frac{1-\alpha_{t}}{2(1-\overline{\alpha}_{t})}J_{t}(x_{t})\bigg)^{-1}\big(x_{t-1}-\widehat{\mu}_t^{\star}(x_{t})\big)\bigg\|_{2}^{2}+\zeta_{t}(x_{t-1},x_{t})\bigg)		 
			\label{eq:cond-dist-crude-fast}
		\end{align}
	holds for some residual term $\zeta_{t}(x_{t-1},x_t)$ obeying 
	\begin{equation}
		\big|\zeta_{t}(x_{t-1},x_t) \big|\leq c_{\zeta} \frac{d^{3}\log^{4.5}T}{T^{3/2}}. 
		\label{eq:eq:cond-dist-crude-xit-fast}
	\end{equation}
	%
	%
	\end{lemma}

Moving beyond the set $\mathcal{E}$, 
it suffices to bound the log density ratio $\log \frac{p_{X_{t-1} \mymid X_t}}{p_{H_{t-1} \mymid H_t}}$ for all pairs $(x_t,x_{t-1})$,  
which can be accomplished in a way similar to \citet[Lemma 13]{li2023towards}.
\begin{lemma}
	\label{lem:sde-R-full}
	For all $(x_t, x_{t-1}) \in \real^d \times \real^d$, we have
	\begin{align}
	\log \frac{p_{X_{t-1} \mymid X_t}(x_{t-1}\mymid x_t)}{p_{H_{t-1} \mymid H_t}(x_{t-1}\mymid x_t)} 
	& \leq T^{c_{0}+2c_{R}+2}\left\{ \big\| x_{t-1}-\widehat{x}_{t}\big\|_{2}^{2}+\|x_{t}\|_{2}^{2}+1\right\},
		\label{eq:SDE-ratio-crude-2}
	\end{align}
	where $c_0$ is defined in \eqref{eqn:alpha-t}. 
	\end{lemma}

%
	Equipped with \Cref{lem:Ht,lem:sde-R,lem:sde-R-full}, 
	one can readily repeat similar arguments as in \citet[Step 3, Theorem 3]{li2023towards} to derive the following result: 
	%
\begin{lemma}\label{lemma:KL-X-H}
	For any $1<t\leq T$, one has
	\begin{align}
		\mathop{\mathbb{E}}\limits_{x_{t}\sim q_{t}}\Big[\mathsf{KL}\Big(p_{X_{t-1}\mymid X_{t}}(\cdot\mymid x_{t})\parallel p_{H_{t-1}\mymid H_{t}}(\cdot\mymid x_{t})\Big)\Big]\lesssim\left(\frac{d^{3}\log^{4.5}T}{T^{3/2}}\right)^{2}. \label{eqn:each-kl}
	\end{align}
\end{lemma}

\paragraph{Step 3: quantifying the KL divergence between $p_{H_{t-1} \mid H_t}$ and $p_{Y_{t-1} \mid Y_t}$.}  
In the previous step, we have quantified the KL divergence between $p_{X_{t-1} \mid X_t}$ and $p_{H_{t-1} \mid H_t}$. 
Recognizing that $H_{t-1}$ is a first-order approximation of $Y_{t-1}$ using the true score function, 
we still need to look at the influence of the score estimation errors, for which we resort to the lemma below.
\begin{lemma}\label{lem:influence-error-KL}
	For any $1<t\leq T$, one has
	\begin{align}
		&\mathop{\mathbb{E}}_{x_{t}\sim q_{t}}\Big[\mathsf{KL}\Big(p_{X_{t-1}\mymid X_{t}}(\cdot\mymid x_{t})\parallel p_{Y_{t-1}\mymid Y_{t}}(\cdot\mymid x_{t})\Big)\Big] - 
		\mathop{\mathbb{E}}_{x_{t}\sim q_{t}}\Big[\mathsf{KL}\Big(p_{X_{t-1}\mymid X_{t}}(\cdot\mymid x_{t})\parallel p_{H_{t-1}\mymid H_{t}}(\cdot\mymid x_{t})\Big)\Big] \notag\\
		&\qquad \lesssim \exp\big(- c_{20} d\log T \big) + \frac{d\log^3 T}{T} \mathop{\mathbb{E}}_{X_t\sim q_t}\big[\varepsilon_{\score, t}(X_t)^2\big]+\frac{d^5\log^{7} T}{T^{3}} .
		\label{eq:H-upper-bound-DDPM}
	\end{align}
\end{lemma}

\paragraph{Step 4: putting all this together.} 
We are now ready to complete the proof. 
Substituting \eqref{eqn:each-kl} and \eqref{eq:H-upper-bound-DDPM} into the decomposition~\eqref{eqn:kl-decomp} yields
\begin{align*}
	\mathsf{KL}(p_{X_{1}}\parallel p_{Y_{1}}) & \leq\mathsf{KL}(p_{X_{T}}\parallel p_{Y_{T}})+\sum_{t=1}^{T-1}\mathop{\mathbb{E}}_{x_{t}\sim q_{t}}\Big[\mathsf{KL}\Big(p_{X_{t-1}\mymid X_{t}}(\cdot\mymid x_{t})\parallel p_{H_{t-1}\mymid H_{t}}(\cdot\mymid x_{t})\Big)\Big]\\
	 & +\sum_{t=1}^{T-1}\left\{ \mathop{\mathbb{E}}_{x_{t}\sim q_{t}}\Big[\mathsf{KL}\Big(p_{X_{t-1}\mymid X_{t}}(\cdot\mymid x_{t})\parallel p_{Y_{t-1}\mymid Y_{t}}(\cdot\mymid x_{t})\Big)\Big]-\mathop{\mathbb{E}}_{x_{t}\sim q_{t}}\Big[\mathsf{KL}\Big(p_{X_{t-1}\mymid X_{t}}(\cdot\mymid x_{t})\parallel p_{H_{t-1}\mymid H_{t}}(\cdot\mymid x_{t})\Big)\Big]\right\} \\
		& \lesssim\mathsf{KL}(p_{X_{T}}\parallel p_{Y_{T}})+\sum_{2\leq t\leq T}\frac{d^{6}\log^{9}T}{T^{3}}+\frac{d\log^{3}T}{T}\sum_{t=2}^{T}\mathop{\mathbb{E}}_{X_{t}\sim q_{t}}\left[\varepsilon_{\score,t}(X_{t})^{2}\right]\\
	 & \asymp\frac{d^{6}\log^{9}T}{T^2}+d\varepsilon_{\score}^{2}\log^{3}T,
	\end{align*}
thereby concluding the proof of Theorem~\ref{thm:main-SDE}.

\subsection{Proof of Lemma~\ref{lem:Ht}}
To begin with, we observe that
\begin{align*}
    &p_{H_{t-1} \mymid H_t}(x_{t-1}\mymid x_t)\propto \\
 &\qquad \exp\Big(-\frac{\alpha_t}{1-\alpha_t}\big(x_{t-1} - \widehat{\mu}_t^{\star}(x_t)\big)^{\top}\mathsf{Var}\Big(Z_t - \frac{1-\alpha_{t}}{1-\overline{\alpha}_{t}}J_t(H_t)Z_t + Z_t^+ \mid H_t=x_t\Big)^{-1}\big(x_{t-1} - \widehat{\mu}_t^{\star}(x_t)\big)\Big). 
    \end{align*}
It is easy to verify that 
\begin{align*}
&\mathsf{Var}\Big(Z_t - \frac{1-\alpha_{t}}{1-\overline{\alpha}_{t}}J_t(H_t)Z_t + Z_t^+ \mid H_t=x_t\Big) \\
&~~~~= 2\Big(I - \frac{1-\alpha_t}{2(1-\overline{\alpha}_{t})}J_{t}(x_t)\Big)\Big(I - \frac{1-\alpha_t}{2(1-\overline{\alpha}_{t})}J_{t}(x_t)\Big)^{\top} + \frac{(1-\alpha_t)^2}{2(1-\overline{\alpha}_{t})^2}J_{t}(x_t)J_{t}(x_t)^{\top}.
\end{align*}
For any $(x_{t-1},x_{t})\in\mathcal{E}$, we can deduce that
\begin{subequations}
	\begin{align}\label{eq:Jt-upper}
		\left\| J_t\left(x_t\right)\right\| &\lesssim  d \log T , \\
\left\|x_{t-1}-\widehat{\mu}^{\star}_t\left(x_t\right)\right\|_2 & \leq\left\|x_{t-1}-\widehat{x}_t\right\|_2+\frac{1-\alpha_t}{\sqrt{\alpha_t}\left(1-\overline{\alpha}_t\right)}
		\mathbb{E}\left[\left\|x_t-\sqrt{\overline{\alpha}_t} X_0\right\|_2 \mid X_t=x_t\right]  \notag\\
		&\stackrel{\text{(i)}}{ \lesssim} \sqrt{d\left(1-\alpha_t\right) \log T}+\sqrt{\frac{d \log T}{1-\overline{\alpha}_t}}\left(1-\alpha_t\right) \asymp \sqrt{d\left(1-\alpha_t\right) \log T} ,
\end{align}
\end{subequations}
where \eqref{eq:Jt-upper} follows \eqref{eq:prop-norm-J}
and (i) both arises from \Cref{lem:x0}. 
Taking the above relations together and using the relation~\eqref{eqn:alpha-t}, we arrive at
\begin{align*}
\frac{1-\alpha_t}{(1-\overline{\alpha}_{t})^2}\|J_{t}(x_t)\|^2\|x_{t-1} - \widehat{\mu}_t^{\star}(x_t)\|_2^2 \lesssim \frac{d^3\log^{5} T}{T^{2}}, 
\end{align*}
which completes the proof.
\subsection{Proof of Lemma~\ref{lem:sde-R-full}}
\label{sec:proof-lem:sde-R-full}


According to the expression \eqref{eq:Ht}, one has
\[
H_{t-1}\mid H_{t}=x_{t}\ \sim\ \mathcal{N}\Bigg(      \widehat{\mu}^{*}_{t}(x_{t}),\,\underset{\eqqcolon\,\Sigma(\widehat{x}_{t})}{\underbrace{\frac{1-\alpha_{t}}{\alpha_{t}}\bigg(I-\frac{1-\alpha_{t}}{2(1-\overline{\alpha}_{t})}J_{t}(x_{t})\bigg)^{2}+\frac{(1-\alpha_t)^3}{4\alpha_t(1-\overline{\alpha}_{t})^2}J_{t}(x_t)^2}}\,\Bigg).
\]
In order to quantify the above density of interest, we first bound the Jacobian matrix $J_{t}(x)$ defined in \eqref{eq:Jacobian-Thm4}. 
On the one hand, the expression \eqref{eq:Jt-x-expression-ij-23} tells us that $J_{t}({x}) \preceq I_d$ for any $x$,  
given that the term within the curly bracket in \eqref{eq:Jt-x-expression-ij-23} is a negative covariance matrix. 
On the other hand, $J_{t}(x)$ can be lower bounded by
\begin{align*}
J_{t}(x) & \succeq-\frac{1}{1-\overline{\alpha}_{t}}\mathbb{E}\Big[\big(X_{t}-\sqrt{\overline{\alpha}_{t}}X_{0}\big)\big(X_{t}-\sqrt{\overline{\alpha}_{t}}X_{0}\big)^{\top}\mid X_{t}=x\Big]\notag\\
 & \succeq-\frac{\mathbb{E}\Big[\big\| X_{t}-\sqrt{\overline{\alpha}_{t}}X_{0}\big\|_{2}^{2}\mid X_{t}=x\Big]}{1-\overline{\alpha}_{t}}I_{d}\succeq-\frac{2\|x\|_{2}^{2}+2T^{2c_{R}}}{1-\overline{\alpha}_{t}}I_{d}\\
 & \succeq-T^{c_{0}+1}\big(\|x\|_{2}^{2}+T^{2c_{R}}\big)I_{d},
\end{align*}
where the second line applies the assumption that $\|X_{0}\|_{2}\leq T^{c_{R}}$,
and the last line invokes the choice \eqref{eqn:alpha-t}. 
As a consequence, we obtain
\begin{subequations}
\label{eq:Sigma-UB-LB}
\begin{align}
\Sigma(\widehat{x}_{t}) & \succeq\frac{1-\alpha_{t}}{\alpha_{t}}\bigg(1-\frac{1-\alpha_{t}}{2(1-\overline{\alpha}_{t})}\bigg)^{2}I_{d}=\frac{1-\alpha_{t}}{4\alpha_{t}}\bigg(\frac{1-\overline{\alpha}_{t}+\alpha_{t}-\overline{\alpha}_{t}}{1-\overline{\alpha}_{t}}\bigg)^{2}I_{d}\succeq\frac{1-\alpha_{t}}{4\alpha_{t}}I_{d}\succeq\frac{1-\alpha_{t}}{4}I_{d};\\
\Sigma(\widehat{x}_{t}) & \preceq\frac{1-\alpha_{t}}{\alpha_{t}}T^{2c_{0}+2}\big(2\|\widehat{x}_{t}\|_{2}^{4}+2T^{4c_{R}}\big)I_{d}+\frac{(1-\alpha_t)^3}{4\alpha_t(1-\overline{\alpha}_{t})^2}I_d\preceq4T^{2c_{0}+2}\big(\|\widehat{x}_{t}\|_{2}^{4}+T^{4c_{R}}\big)I_{d}.
\end{align}
\end{subequations}
%

%
%

%
%

With the above relations in mind, we are ready to bound the density function $p_{H_{t-1} \mymid H_t}(x_{t-1}\mymid x_t)$ for any $x_t,x_{t-1}\in \mathbb{R}^d$. 
It is seen from \eqref{eq:Ht} that
\begin{align*}
\log\frac{1}{p_{H_{t-1}\mid H_{t}}(x_{t-1}\mymid x_{t})} & =\frac{\big(x_{t-1}-\widehat{\mu}^{*}_{t}(x_{t})\big){}^{\top}\big(\Sigma(\widehat{x}_{t})\big)^{-1}\big(x_{t-1}-\widehat{\mu}^{*}_{t}(x_{t})\big)}{2}+\frac{1}{2}\log\mathsf{det}\big(\Sigma(\widehat{x}_{t})\big)+\frac{d}{2}\log(2\pi)\\
 & \leq\frac{2\big\| x_{t-1}-\widehat{\mu}^{*}_{t}(x_{t})\big\|_{2}^{2}}{1-\alpha_{t}}+\frac{d}{2}\log\Big(8\pi T^{2c_{0}+2}\big(\|\widehat{x}_{t}\|_{2}^{4}+T^{4c_{R}}\big)\Big)\\
 & \leq2T^{c_{0}+1}\left\{ 2\big\| x_{t-1}-\widehat{x}_{t}\big\|_{2}^{2}+\|x_{t}\|_{2}^{2}+T^{2c_{R}}\right\} +\frac{d}{2}\log\Big(8\pi T^{2c_{0}+2}\big(\|\widehat{x}_{t}\|_{2}^{4}+T^{4c_{R}}\big)\Big)\\
 & \leq T^{c_{0}+2c_{R}+2}\left\{ \big\| x_{t-1}-\widehat{x}_{t}\big\|_{2}^{2}+\|x_{t}\|_{2}^{2}+1\right\} ,
\end{align*}
where the second inequality results from \eqref{eq:Sigma-UB-LB}, 
and the third inequality makes use of  \eqref{eqn:alpha-t} and  the fact that 
\begin{align}
   \|x_{t-1}-\widehat{\mu}_{t}^{\star}(x_t)\|_{2}^{2} & \leq2\|x_{t-1}-\widehat{x}_{t}\|_{2}^{2}+2\|\widehat{x}_{t}-\widehat{\mu}_{t}^{\star}(x_t)\|_{2}^{2} \notag\\
    & =2\|x_{t-1}-\widehat{x}_{t}\|_{2}^{2}+2\bigg(\frac{1-\alpha_{t}}{\sqrt{\alpha_{t}}(1-\overline{\alpha}_{t})}\bigg)^{2}\Big\|\int_{x_{0}}p_{X_{0}\mymid X_{t}}(x_{0}\mymid x_{t})\big(x_{t}-\sqrt{\overline{\alpha}_{t}}x_{0}\big)\mathrm{d}x_{0}\Big\|_{2}^{2}\notag\\
    & \leq2\|x_{t-1}-\widehat{x}_{t}\|_{2}^{2}+\frac{2(1-\alpha_{t})^{2}}{\alpha_{t}(1-\overline{\alpha}_{t-1})^{2}}\sup_{x_{0}:\|x_{0}\|_{2}\leq T^{c_{R}}}\|x_{t}-\sqrt{\overline{\alpha}_{t}}x_{0}\|_{2}^{2}\notag\\
    & \leq2\|x_{t-1}-\widehat{x}_{t}\|_{2}^{2}+\frac{64c_{1}^{2}\log^{2}T}{T^{2}}\bigg(2\|x_{t}\|_{2}^{2}+2\overline{\alpha}_{t}T^{2c_{R}}\bigg)\notag\\
    & \leq2\|x_{t-1}-\widehat{x}_{t}\|_{2}^{2}+\|x_{t}\|_{2}^{2}+T^{2c_{R}}. 
      \label{eq:dist-xt-mu-xt-UB-crude}
   \end{align}
Given that $\log\frac{p_{X_{t-1}\mid X_{t}}(x_{t-1}\mymid x_{t})}{p_{H_{t-1}\mid H_{t}}(x_{t-1}\mymid x_{t})}\leq \log\frac{1}{p_{H_{t-1}\mid H_{t}}(x_{t-1}\mymid x_{t})}$, 
we have concluded the proof.

\subsection{Proof of Lemma~\ref{lemma:KL-X-H}}
Firstly, it follows from \Cref{lem:Ht} and \Cref{lem:sde-R} that: for any $(x_t,x_{t-1})\in \mathcal{E}$, 
\begin{align}
	\frac{p_{X_{t-1}\mymid X_{t}}(x_{t-1}\mymid x_{t})}{p_{H_{t-1}\mymid H_{t}}(x_{t-1}\mymid x_{t})} & =\exp\Big(O\Big(\frac{d^{3}\log^{4.5}T}{T^{3/2}}\Big)\Big)\\
   &=1+O\bigg(\frac{d^{3}\log^{4.5}T}{T^{3/2}}\bigg)\in\Big[\frac{1}{2},2\Big],
	\label{eq:ratio-p-X-H-range}
\end{align}
%
which further allows one to derive 
\begin{align}
 & \mathop{\mathbb{E}}\limits_{x_{t}\sim q_{t}}\Big[\mathsf{KL}\Big(p_{X_{t-1}\mymid X_{t}}(\cdot\mymid x_{t})\parallel p_{H_{t-1}\mymid H_{t}}(\cdot\mymid x_{t})\Big)\Big]\notag\\
 & =\Big(\int_{\mathcal{E}}+\int_{\mathcal{E}^{\mathrm{c}}}\Big)p_{X_{t}}(x_{t})p_{X_{t-1}\mymid X_{t}}(x_{t-1}\mymid x_{t})\log\frac{p_{X_{t-1}\mymid X_{t}}(x_{t-1}\mymid x_{t})}{p_{H_{t-1}\mymid H_{t}}(x_{t-1}\mymid x_{t})}\mathrm{d}x_{t-1}\mathrm{d}x_{t},\notag\\
 & \stackrel{(\text{i})}{=}\int_{\mathcal{E}}p_{X_{t}}(x_{t})\Bigg\{ p_{X_{t-1}\mymid X_{t}}(x_{t-1}\mymid x_{t})-p_{H_{t-1}\mymid H_{t}}(x_{t-1}\mymid x_{t})\notag\\
 & \qquad\qquad\qquad+p_{X_{t-1}\mymid X_{t}}(x_{t-1}\mymid x_{t})\cdot O\Bigg(\bigg(\frac{p_{H_{t-1}\mymid H_{t}}(x_{t-1}\mymid x_{t})}{p_{X_{t-1}\mymid X_{t}}(x_{t-1}\mymid x_{t})}-1\bigg)^{2}\Bigg)\Bigg\}\mathrm{d}x_{t-1}\mathrm{d}x_{t} \notag\\
 & \qquad+\int_{\mathcal{E}^{\mathrm{c}}}p_{X_{t}}(x_{t})p_{X_{t-1}\mymid X_{t}}(x_{t-1}\mymid x_{t})\log\frac{p_{X_{t-1}\mymid X_{t}}(x_{t-1}\mymid x_{t})}{p_{H_{t-1}\mymid H_{t}}(x_{t-1}\mymid x_{t})}\mathrm{d}x_{t-1}\mathrm{d}x_{t}\notag%
\\
 & \stackrel{(\text{ii})}{=}\int_{\mathcal{E}}p_{X_{t}}(x_{t})\Bigg\{ p_{X_{t-1}\mymid X_{t}}(x_{t-1}\mymid x_{t})-p_{H_{t-1}\mymid H_{t}}(x_{t-1}\mymid x_{t})+p_{X_{t-1}\mymid X_{t}}(x_{t-1}\mymid x_{t}) O\bigg(\frac{d^{6}\log^{9}T}{T^{3}}\bigg)\Bigg\}\mathrm{d}x_{t-1}\mathrm{d}x_{t}\notag\\
 & \qquad+\int_{\mathcal{E}^{\mathrm{c}}}p_{X_{t}}(x_{t})p_{X_{t-1}\mymid X_{t}}(x_{t-1}\mymid x_{t})\Big\{2T\big(\|x_{t}\|_{2}^{2}+\|x_{t-1}-\widehat{x}_{t}\|_{2}^{2}+T^{2c_{R}}\big)\Big\}\mathrm{d}x_{t-1}\mathrm{d}x_{t}.%
\label{eqn:serenade}
\end{align}
Here,  (i) invokes the basic fact that: if $\big|\frac{p_Y(x)}{p_X(x)} - 1\big| < \frac{1}{2}$, then the Taylor expansion gives
\begin{align*}
p_X(x)\log \frac{p_X(x)}{p_Y(x)} &= -p_X(x)\log \Big(1 + \frac{p_Y(x) - p_X(x)}{p_X(x)}\Big) \notag \\
&= p_X(x) - p_Y(x) + p_X(x)O\bigg(\Big(\frac{p_Y(x)}{p_X(x)} - 1\Big)^2\bigg);
\end{align*}
and in (ii) we apply \eqref{eq:ratio-p-X-H-range} and Lemma~\ref{lem:sde-R-full}.

Next, we would like to bound each term on the right-hand side of \eqref{eqn:serenade} separately.  
In view of the definition of the set $\mathcal{E}$ (cf.~\eqref{eqn:eset}), one has 
\begin{align}
\mathbb{P}\big((X_{t},X_{t-1})\notin\mathcal{E}\big) & =\int_{(x_{t},x_{t-1})\notin\mathcal{E}}p_{X_{t}-1}(x_{t-1})p_{X_{t}\mymid X_{t-1}}(x_{t}\mymid x_{t-1})\mathrm{d}x_{t-1}\mathrm{d}x_{t} \notag\\
 & =\int_{(x_{t},x_{t-1})\notin\mathcal{E}}p_{X_{t}-1}(x_{t-1})\frac{1}{\big(2\pi(1-\alpha_{t})\big)^{d/2}}\exp\bigg(-\frac{\|x_{t}-\sqrt{\alpha_{t}}x_{t-1}\|_{2}^{2}}{2(1-\alpha_{t})}\bigg)\mathrm{d}x_{t-1}\mathrm{d}x_{t} \notag\\
 & \le\exp\big(-c_3d\log T\big),
	\label{eq:P-Xt-X-t-1-not-E}
\end{align}
and similarly,
\begin{align}
\int_{(x_{t-1},x_{t})\notin\mathcal{E}}p_{X_{t}}(x_{t})p_{X_{t-1}\mymid X_{t}}(x_{t-1}\mymid x_{t})\Big(2T\big(\|x_{t}\|_{2}^{2}+\|x_{t-1}-\widehat{x}_{t}\|_{2}^{2}\big)+T^{2c_{R}}\big)\mathrm{d}x_{t-1}\mathrm{d}x_{t} & \le\exp\big(-c_3d\log T\big).	
\label{eq:P-Xt-X-t-1-not-E-246}
\end{align}
In addition, for every $(x_t,x_{t-1})$ obeying $\|x_{t-1} - x_t/\sqrt{\alpha_t}\|_2 > c_3 \sqrt{d(1 - \alpha_t)\log T}$ and $-\log p_{X_t}(x_t) \leq \frac{1}{2}c_6 d\log T$, 
it follows from the definition \eqref{eq-mu-hat} of $\widehat{\mu}_t^{\star}(\cdot)$ that
\begin{align}
	\|x_{t-1}-\widehat{\mu}_{t}^{\star}(x_{t})\|_{2} & =\bigg\| x_{t-1}-\frac{1}{\sqrt{\alpha_{t}}}x_{t}-\frac{1-\alpha_{t}}{\sqrt{\alpha_{t}}(1-\overline{\alpha}_{t})}\mathbb{E}\Big[x_{t}-\sqrt{\overline{\alpha}_{t}}X_{0}\mid X_{t}=x_{t}\Big]\bigg\|_{2} \label{eq:x-tminus1-mu-2norm}\\
 & \geq\bigg\| x_{t-1}-\frac{1}{\sqrt{\alpha_{t}}}x_{t}\bigg\|_{2}-\frac{1-\alpha_{t}}{\sqrt{\alpha_{t}}(1-\overline{\alpha}_{t})}\mathbb{E}\Big[\big\| x_{t}-\sqrt{\overline{\alpha}_{t}}X_{0}\big\|_{2}\mid X_{t}=x_{t}\Big] \notag\\
 & \geq c_{3}\sqrt{d(1-\alpha_{t})\log T}-6\overline{c}_{5}\frac{1-\alpha_{t}}{\sqrt{\alpha_{t}(1-\overline{\alpha}_{t})}}\sqrt{d\log T} \notag\\
 & =\Bigg(c_{3}-6\overline{c}_{5}\frac{\sqrt{1-\alpha_{t}}}{\sqrt{\alpha_{t}(1-\overline{\alpha}_{t})}}\Bigg)\sqrt{d(1-\alpha_{t})\log T}\geq\frac{c_{3}}{2}\sqrt{d(1-\alpha_{t})\log T},
	\label{eq:x-tminus1-mu-2norm-479}
\end{align}
where the third line results from \eqref{eq:E-xt-X0} in Lemma~\ref{lem:x0},
and the last line applies \eqref{eqn:properties-alpha-proof}
and holds true as long as $c_{3}$ is large enough. 
Taking this result together with \Cref{lem:Ht} reveals that: for any $x_t$ obeying $-\log p_{X_t}(x_t) \leq \frac{1}{2}c_6 d\log T$, one has
\begin{align}
	\int_{x_{t-1}:\|x_{t-1} - x_t/\sqrt{\alpha_t}\|_2 > c_3 \sqrt{d(1 - \alpha_t)\log T}} 
	p_{H_{t-1} \mymid H_t}(x_{t-1}\mymid x_t)\mathrm{d}x_{t-1} &\le \exp\Big(- \frac{c_3}{2} d\log T \Big).
	\label{eq:P-Yt-X-t-1-not-E}
\end{align}
Combine \eqref{eq:P-Xt-X-t-1-not-E} and \eqref{eq:P-Yt-X-t-1-not-E} to arrive at
\begin{align}
 & \left|\int_{\mathcal{E}}p_{X_{t}}(x_{t})\Big\{ p_{X_{t-1}\mymid X_{t}}(x_{t-1}\mymid x_{t})-p_{H_{t-1}\mymid H_{t}}(x_{t-1}\mymid x_{t})\Big\}\mathrm{d}x_{t-1}\mathrm{d}x_{t}\right|\leq \mathbb{P}\big((X_{t},X_{t-1})\notin\mathcal{E}\big) \notag\\
 & \quad+\int_{\log p_{X_t}(x_t) \leq \frac{1}{2}c_6 d\log T, \|x_{t-1} - x_t/\sqrt{\alpha_t}\|_2 > c_3 \sqrt{d(1 - \alpha_t)\log T}}p_{X_{t}}(x_{t})p_{H_{t-1}\mymid H_{t}}(x_{t-1}\mymid x_{t})\mathrm{d}x_{t-1}\mathrm{d}x_{t} \notag\\
 & \quad\leq2\exp\Big(- \frac{c_3}{2} d\log T \Big).
	\label{eq:X-H-leak}
\end{align}

To finish up, plugging \eqref{eq:P-Xt-X-t-1-not-E-246} and \eqref{eq:X-H-leak} into \eqref{eqn:serenade} yields: for each $t\geq 2$,  
\begin{align}
	\mathop{\mathbb{E}}_{x_{t}\sim q_{t}}\Big[\mathsf{KL}\Big(p_{X_{t-1}\mymid X_{t}}(\cdot\mymid x_{t})\parallel p_{Y_{t-1}^{\star}\mymid Y_{t}}(\cdot\mymid x_{t})\Big)\Big] 
	 \lesssim
	\frac{d^{6}\log^{9}T}{T^{3}} + 3\exp\Big(- \frac{c_3}{2} d\log T \Big)
		\lesssim \frac{d^{6}\log^{9}T}{T^{3}}.
\end{align}

\subsection{Proof of Lemma~\ref{lem:influence-error-KL}}
We first introduce the following notation: 
\begin{align*}
    \mu_{t}(x_t,z_t)& \coloneqq \frac{1}{\sqrt{\alpha_{t}}}\Big(x_t + \sqrt{\frac{1 - \alpha_t}{2}} z_t + (1-\alpha_{t})s_t\Big(x_t + \sqrt{\frac{1 - \alpha_t}{2}} z_t\Big)\Big); \\
    \mu_{t}^{\star}(x_t,z_t) &\coloneqq \frac{1}{\sqrt{\alpha_{t}}}\Big(x_t + \sqrt{\frac{1 - \alpha_t}{2}} z_t + (1-\alpha_{t})s_t^{\star}(x_t) - \frac{(1-\alpha_{t})^{3/2}}{\sqrt{2}(1-\overline{\alpha}_{t})}J_t(x_t)z_t\Big). 
    \end{align*}
In the sequel, we shall use $ \mu_{t}$ and  $\mu_{t}^{\star}$ to denote $\mu_{t}(x_t,z_t)$ and $\mu_{t}^{\star}(x_t,z_t)$, respectively, for simplicity, as long as it is clear from the context.  
It is observed that 
   \begin{align}
      & \mathop{\mathbb{E}}\limits_{x_{t}\sim q_{t}}\Big[\mathsf{KL}\Big(p_{X_{t-1}\mymid X_{t}}(\cdot\mymid x_{t})\parallel p_{Y_{t-1}\mymid Y_{t}}(\cdot\mymid x_{t})\Big)\Big] 
	   - \mathop{\mathbb{E}}\limits_{x_{t}\sim q_{t}}\Big[\mathsf{KL}\Big(p_{X_{t-1}\mymid X_{t}}(\cdot\mymid x_{t})\parallel p_{H_{t-1}\mymid H_{t}}(\cdot\mymid x_{t})\Big)\Big] \notag\\
      & = \int p_{X_{t}}(x_{t})p_{X_{t-1}\mymid X_{t}}(x_{t-1}\mymid x_{t})\log\frac{p_{H_{t-1}\mymid H_{t}}(x_{t-1}\mymid x_{t})}{p_{Y_{t-1}\mymid Y_{t}}(x_{t-1}\mymid x_{t})}\mathrm{d}x_{t-1}\mathrm{d}x_{t} \notag\\
      & = \int p_{X_{t}}(x_{t})p_{X_{t-1}\mymid X_{t}}(x_{t-1}\mymid x_{t})p_{Z_{t}\mymid H_{t-1},H_{t}}(z_{t}\mymid x_{t-1},x_{t})\log\frac{p_{H_{t-1}\mymid H_{t}}(x_{t-1}\mymid x_{t})}{p_{Y_{t-1}\mymid Y_{t}}(x_{t-1}\mymid x_{t})}\mathrm{d}z_{t}\mathrm{d}x_{t-1}\mathrm{d}x_{t} \notag\\
	   & \stackrel{\text{(i)}}{\le} \int p_{X_{t}}(x_{t})p_{X_{t-1}\mymid X_{t}}(x_{t-1}\mymid x_{t})p_{Z_{t}\mymid H_{t-1},H_{t}}(z_{t}\mymid x_{t-1},x_{t})\log\frac{p_{H_{t-1}\mymid H_{t}, Z_{t}}(x_{t-1}\mymid x_{t}, z_{t})}{p_{Y_{t-1}\mymid Y_{t}, Z_{t}}(x_{t-1}\mymid x_{t}, z_{t})}\mathrm{d}z_{t}\mathrm{d}x_{t-1}\mathrm{d}x_{t} \notag\\
	   & \stackrel{\text{(ii)}}{=} \int p_{X_{t}}(x_{t})p_{X_{t-1}\mymid X_{t}}(x_{t-1}\mymid x_{t})p_{Z_{t}\mymid H_{t-1},H_{t}}(z_{t}\mymid x_{t-1},x_{t})\frac{\alpha_{t}}{(1-\alpha_{t})}\notag\Big(\big\| x_{t-1}-\mu_{t}\big\|_{2}^{2} - \big\| x_{t-1}-\mu_{t}^{\star}\big\|_{2}^{2} \Big)\mathrm{d}z_{t}\mathrm{d}x_{t-1}\mathrm{d}x_{t}\\
      &=\underbrace{\int p_{X_{t}}(x_{t})p_{X_{t-1}\mymid X_{t}}(x_{t-1}\mymid x_{t})p_{Z_{t}\mymid H_{t-1},H_{t}}(z_{t}\mymid x_{t-1},x_{t})\frac{\alpha_{t}}{(1-\alpha_{t})}\notag\big\| \mu_{t}^{\star}-\mu_{t}\big\|_{2}^{2}  \mathrm{d}z_{t}\mathrm{d}x_{t-1}\mathrm{d}x_{t}}_{\mathcal{H}_1}\\
      &\qquad+\underbrace{\int p_{X_{t}}(x_{t})p_{X_{t-1}\mymid X_{t}}(x_{t-1}\mymid x_{t})p_{Z_{t}\mymid H_{t-1},H_{t}}(z_{t}\mymid x_{t-1},x_{t})\frac{2\alpha_{t}}{(1-\alpha_{t})}\notag (\mu_{t}^{\star}-\mu_{t})^{\top}\big( x_{t-1}-\mu_{t}^{\star}\big) \mathrm{d}z_{t}\mathrm{d}x_{t-1}\mathrm{d}x_{t}}_{\mathcal{H}_2}. 
     \end{align}
Here, (i) follows the property of KL divergence that 
   \begin{align*}
   \int p_{Z_{t}\mymid H_{t-1},H_{t}}(z_{t}\mymid x_{t-1},x_{t})\log\frac{p_{Z_{t}\mymid H_{t-1},H_{t}}(z_{t}\mymid x_{t-1},x_{t})}{p_{Z_{t}\mymid Y_{t-1},Y_{t}}(z_{t}\mymid x_{t-1},x_{t})}\mathrm{d}z_{t} \ge 0,
   \end{align*}
   whereas (ii) results from the following expressions:
   \begin{align*}
    p_{H_{t-1} \mymid H_t, Z_t}(x_{t-1}\mymid x_t, z_t) &\propto \exp\Big(-\frac{\alpha_t}{(1-\alpha_t)}\Big\|(x_{t-1} - {\mu}_t^{\star}(x_t))\Big\|^2 \Big)\\
    p_{Y_{t-1} \mymid Y_t, Z_t}(x_{t-1}\mymid x_t, z_t) &\propto \exp\Big(-\frac{\alpha_t}{(1-\alpha_t)}\Big\|(x_{t-1} - {\mu}_t(x_t))\Big\|^2 \Big).
    \end{align*}
   To bound $\mathcal{H}_1$,  we first note that 
   \begin{align*}
     & \frac{1}{1-\alpha_t}\big\| \mu_{t}-\mu_{t}^{\star}\big\|_{2}^2 = (1-\alpha_t)\cdot \\
      & \qquad \left\|s_t\Big(x_t + \sqrt{\frac{1 - \alpha_t}{2}} z_t\Big)-s_t^{\star}\Big(x_t + \sqrt{\frac{1 - \alpha_t}{2}} z_t\Big) +s_t^{\star}\Big(x_t + \sqrt{\frac{1 - \alpha_t}{2}} z_t\Big) - s_t^{\star}(x_t) + \frac{(1-\alpha_{t})^{1/2}}{\sqrt{2}(1-\overline{\alpha}_{t})}J_t(x_t)z_t\right\|_{2}^2 \\
      &\leq (1-\alpha_t) \left\|s_t\Big(x_t + \sqrt{\frac{1 - \alpha_t}{2}} z_t\Big)-s_t^{\star}\Big(x_t + \sqrt{\frac{1 - \alpha_t}{2}} z_t\Big)\right\|_{2}^2+\\
      &\qquad  (1-\alpha_t)\bigg\|s_t^{\star}\Big(x_t + \sqrt{\frac{1 - \alpha_t}{2}} z_t\Big) - s_t^{\star}(x_t) + \frac{(1-\alpha_{t})^{1/2}}{\sqrt{2}(1-\overline{\alpha}_{t})}J_t(x_t)z_t\bigg\|_2^2\\
      &\lesssim \frac{\log T}{T}\varepsilon_{\score, t}\Big(x_t + \sqrt{\frac{1 - \alpha_t}{2}} z_t\Big)^2 + \frac{d^5\log^{7} T}{T^{3}}
      \end{align*}
   where the last inequality follows from the definition~\eqref{eq:pointwise-epsilon-score-J}, the relation~\eqref{eqn:alpha-t}, and the fact that
   \begin{align*}
   &(1-\alpha_t)\bigg\|s_t^{\star}\Big(x_t + \sqrt{\frac{1 - \alpha_t}{2}} z_t\Big) - s_t^{\star}(x_t) + \frac{(1-\alpha_{t})^{1/2}}{\sqrt{2}(1-\overline{\alpha}_{t})}J_t(x_t)z_t\bigg\|_2^2 \\
   &= \frac{(1-\alpha_{t})^{2}}{2(1-\overline{\alpha}_{t})^2}\bigg\|\int_0^1 \Big(J_t(x_t) - J_t\Big(x_t + \gamma\sqrt{\frac{1 - \alpha_t}{2}} z_t\Big)\Big) z_t\mathrm{d}\gamma\bigg\|_2^2 
   \lesssim \frac{d^5\log^{7} T}{T^{3}}.
   \end{align*}
   Here, the last inequality holds by invoking the property \eqref{eq:prop-norm-partJ} that  
   for $(x, x_{t-1}) \in \mathcal{E}$,
\begin{align} \label{eq:J-Lip}
\|J_t(x) - J_t(x_t)\| &\le \sup_{u \in \mathbb{S}^{d-1}} |u^{\top}(J_t(x) - J_t(x_t))u| \lesssim d^{3/2}\|x - x_t\|_2\log^{3/2} T. 
\end{align}
For the case with $(x, x_{t-1}) \notin \mathcal{E}$, this term will decay exponentially fast and can be bounded analogously. Furthermore, we observe that 
   \begin{align*}
   p_{\Phi_t(X_t, Z_t)}(x) = \big(\pi(2(1 - \overline{\alpha}_t) + 1-\alpha_t)\big)^{-d/2}\int p_{X_0}(x_0)\exp\bigg(-\frac{\|x - \sqrt{\overline{\alpha}_t}x_0\|_2^2}{2(1 - \overline{\alpha}_t) + 1-\alpha_t}\bigg) \mathrm{d} x_0 \asymp p_{X_t}(x),
   \end{align*}
   which in turn implies that
   \begin{align*}
      \mathcal{H}_{1}&\leq  
	   \Big(1+O\bigg(\frac{d^{3}\log^{4.5}T}{T^{3/2}}\bigg)\Big) \int p_{X_{t}}(x_{t})p_{H_{t-1}\mymid H_{t}, Z_{t}}(x_{t-1}\mymid x_{t}, z_{t})p_{Z_{t}}(z_{t}) \frac{\alpha_t}{1-\alpha_t}\big\| \mu_{t}-\mu_{t}^{\star}\big\|_{2}^2 \mathrm{d}x_t\mathrm{d}x_{t-1}\mathrm{d}z_t\\
	   &\lesssim \mathop{\mathbb{E}}\limits_{x^{+}\sim \Phi_t(X_t, Z_t) } \left[ \frac{\log T}{T}\varepsilon_{\score, t}(x^{+})^2\right] + \frac{d^5\log^{7} T}{T^{3}}\\
      &\asymp\frac{d\log^3 T}{T} \mathop{\mathbb{E}}_{X_t\sim q_t}\big[\varepsilon_{\score, t}(X_t)^2\big]+\frac{d^5\log^{7} T}{T^{3}} 
   \end{align*}
   We then decompose  $\mathcal{H}_2$ as follows
   \begin{align*}
   \mathcal{H}_2 &=\int p_{X_{t}}(x_{t})\big(p_{X_{t-1}\mymid X_{t}}(x_{t-1}\mymid x_{t})-p_{H_{t-1}\mymid H_{t}}(x_{t-1}\mymid x_{t})\big)p_{Z_{t}\mymid H_{t-1},H_{t}}(z_{t}\mymid x_{t-1},x_{t})\\
   &\qquad \cdot \frac{2\alpha_{t}}{(1-\alpha_{t})}\notag (\mu_{t}^{\star}-\mu_{t})^{\top}\big( x_{t-1}-\mu_{t}^{\star}\big) \mathrm{d}z_{t}\mathrm{d}x_{t-1}\mathrm{d}x_{t}\\
   & + \int p_{X_{t}}(x_{t})p_{H_{t-1}\mymid H_{t}, Z_t}(x_{t-1}\mymid x_{t},z_t)p_{Z_{t}}(z_{t}) \cdot \frac{2\alpha_{t}}{(1-\alpha_{t})}\notag (\mu_{t}^{\star}-\mu_{t})^{\top}\big( x_{t-1}-\mu_{t}^{\star}\big) \mathrm{d}z_{t}\mathrm{d}x_{t-1}\mathrm{d}x_{t}\\
	   &\stackrel{\text{(i)}}{=}\int p_{X_{t}}(x_{t})\big(p_{X_{t-1}\mymid X_{t}}(x_{t-1}\mymid x_{t})-p_{H_{t-1}\mymid H_{t}}(x_{t-1}\mymid x_{t})\big)p_{Z_{t}\mymid H_{t-1},H_{t}}(z_{t}\mymid x_{t-1},x_{t})\\
   &\qquad \cdot \frac{2\alpha_{t}}{(1-\alpha_{t})}\notag (\mu_{t}^{\star}-\mu_{t})^{\top}\big( x_{t-1}-\mu_{t}^{\star}\big) \mathrm{d}z_{t}\mathrm{d}x_{t-1}\mathrm{d}x_{t}\\
   &= \left(\int_{\mathcal{E}}+\int_{\mathcal{E}^c}\right) p_{X_{t}}(x_{t})\big(p_{X_{t-1}\mymid X_{t}}(x_{t-1}\mymid x_{t})-p_{H_{t-1}\mymid H_{t}}(x_{t-1}\mymid x_{t})\big)p_{Z_{t}\mymid H_{t-1},H_{t}}(z_{t}\mymid x_{t-1},x_{t})\\
   &\qquad \cdot \frac{2\alpha_{t}}{(1-\alpha_{t})}\notag (\mu_{t}^{\star}-\mu_{t})^{\top}\big( x_{t-1}-\mu_{t}^{\star}\big) \mathrm{d}z_{t}\mathrm{d}x_{t-1}\mathrm{d}x_{t},
   \end{align*}
   where (i) follows the fact that $\mathbb{E}[H_{t-1}-\mu^{*}_{t}|H_t,Z_t]=0$. In the following, we mainly focus on the term $\int_{\mathcal{E}}$ denoted as $\mathcal{K}_1$, since the other term can be bounded similarly as \cite[Lemma 10]{li2023towards} and is exponentially small.

   \begin{align}
      \mathcal{K}_1
      &\overset{\mathrm{(i)}}{\lesssim} \frac{d^{3}\log^{4.5}T}{T^{3/2}}\int_{\mathcal{E}} p_{X_{t}}(x_{t})p_{H_{t-1}\mymid H_{t},Z_t}(x_{t-1}\mymid x_{t},z_t) P_{Z_t}(z_t)\big\|x_{t-1} - \mu_{t}^{\star}\big\|_2 \frac{1}{1-\alpha_t}\big\| \mu_{t}-\mu_{t}^{\star}\big\|_{2}\mathrm{d}x_{t-1}\mathrm{d}x_{t} 
      \notag\\
   %
   &\overset{\mathrm{(ii)}}{\lesssim}\frac{d^{3}\log^{4.5}T}{T^{3/2}}\sqrt{\mathcal{K}_2\mathcal{K}_3}. 
      \label{eq:final-bound-K1}
   \end{align}
Here, we have
   \begin{align*}
	   \mathcal{K}_2 &= \int_{\mathcal{E}} p_{X_{t}}(x_{t})p_{H_{t-1}\mymid H_{t},Z_t}(x_{t-1}\mymid x_{t},z_t) P_{Z_t}(z_t)
   \big\|x_{t-1} - \mu_{t}^{\star}\big\|_2^2\mathrm{d}x_{t-1}\mathrm{d}x_{t} \mathrm{d}z_{t}\\
	   &\leq  \frac{d(1-\alpha_t)}{\alpha_t}\lesssim \frac{d\log T}{T};\\
	   \mathcal{K}_3 &= \int_{\mathcal{E}} p_{X_{t}}(x_{t})p_{H_{t-1}\mymid H_{t},Z_t}(x_{t-1}\mymid x_{t},z_t) P_{Z_t}(z_t)
   \frac{1}{(1-\alpha_t)^2}\big\| \mu_{t}-\mu_{t}^{\star}\big\|_{2}^2\mathrm{d}x_{t-1}\mathrm{d}x_{t}\mathrm{d}z_{t}\\
	   &\lesssim  \mathop{\mathbb{E}}\limits_{X_t \sim q_t}\left[\varepsilon_{\text {score }, t}\left(X_t\right)^2\right]+ \frac{d^5\log^{6} T}{T^2}.
   \end{align*}
   Therefore, we arrive at
   \begin{align*}
      \mathcal{K}_1\lesssim \frac{d^{3.5}\log^5 T}{T^2}\mathop{\mathbb{E}}\limits_{X_t \sim q_t}\left[\varepsilon_{\text {score }, t}\left(X_t\right)^2\right]+ \frac{d^{6}\log^8 T}{T^3}.
   \end{align*}
   Taking the above bounds on $\mathcal{H}_1$ and $\mathcal{K}_1$ together completes the proof.



\end{document}